\documentclass[%
 reprint,
nofootinbib,
pra,
aps
]{revtex4-2}
\usepackage{amsmath}
\usepackage{amssymb}
\usepackage{amsfonts}
\usepackage{graphicx} % Required for inserting images
\usepackage[normalem]{ulem}
\usepackage{tcolorbox}
\usepackage{amsfonts}
\usepackage{xcolor}
\usepackage{mathtools}
\usepackage{graphicx}% Include figure files
\usepackage{dcolumn}% Align table columns on decimal point
\usepackage{braket}
\usepackage{amsthm}
\usepackage{caption}
\usepackage{subcaption}
\usepackage{hyperref}% add hypertext capabilities
\newcommand{\SupI}{\emph{Appendix}}

\newcommand{\kc}{\mathcal{K}}
\newcommand{\scc}{\mathcal{S}}

\newcommand{\lc}{\mathcal{L}}

\newcommand{\ajcomm}[1]{ \hfill \break  \{ \textcolor{blue}{AJ:#1} \}  }

\newcommand{\utheta}{\underline{\theta}}
\newcommand{\usigma}{\underline{\sigma}}
\newcommand{\hutheta}{\underline{\hat{\theta}}}

\newcommand{\argmin}[1]{\underset{#1}{\text{argmin}} }

\newcommand{\Enu}{\underset{\usigma \sim \nu}{ \mathbb{E}} }

\newcommand{\EXp}[1]{\underset{#1}{ \mathbb{E}} }
\newcommand{\Var}[1]{\underset{#1}{\operatorname{Var}} }
\newcommand{\su}{\setminus u}

\usepackage{xcolor}
\usepackage{comment}   % needed to hide deleted blocks in final

\definecolor{REVDEL}{RGB}{200,0,0}
\definecolor{REVADD}{RGB}{0,70,200}

\newif\ifSHOWREVISIONS
\SHOWREVISIONStrue   % <-- change to \SHOWREVISIONSfalse for clean

\newtheorem{theorem}{Theorem}
\newtheorem{corollary}{Corollary}
\newtheorem{lemma}{Lemma}
\newtheorem{proposition}{Proposition}
\newtheorem{definition}{Definition}
\newtheorem{condition}{Condition}

\long\def\ca#1\cb{} %Use for commenting out: \ca...\cb

\begin{document}
%\title{Discrete distributions are learnable from metastable samples}
\title{Discrete distributions are learnable from metastable samples}

\author{Abhijith Jayakumar}
\email{abhijithj@lanl.gov}
\author{Andrey Y. Lokhov}
%\email{lokhov@lanl.gov}
\author{Sidhant Misra}
%\email{sidhant@lanl.gov}
\author{Marc Vuffray}
%\email{vuffray@lanl.gov}
\affiliation{Theoretical Division, Los Alamos National Laboratory, Los Alamos, NM 87545, USA}

%\affiliation{$^{2}$Center for Nonlinear Studies, Los Alamos National Laboratory, Los Alamos, NM 87545, USA}
%\maketitle

\begin{abstract}
Physically motivated stochastic dynamics are widely used to sample from high-dimensional distributions. However, such samplers often get trapped in metastable states, approximately sampling from a distribution that differs significantly from the desired stationary state. We rigorously show that for multivariable discrete distributions, the true stationary model can nevertheless be recovered from these metastable samples. This relies on a fundamental observation: for distributions satisfying a strong metastability condition, their single-variable conditional probabilities are on average extremely close to those of the true stationary distribution. This remains true even when the two distributions are far apart under global metrics such as Kullback-Leibler divergence. Consequently, we can effectively learn the true model using a conditional-likelihood estimator even when the samples are drawn from a restricted state space. Extending these general results to Ising models, we prove rigorous parameter and structure learning guarantees. Finally, we demonstrate this phenomenon numerically on higher-alphabet spin glass models.
\end{abstract}
\maketitle
%\section*{\bf Significance statement\\}
%Slow-mixing Markov chains often produce samples from metastable regions rather than the true equilibrium distribution, posing a fundamental barrier to reliable learning in high-dimensional models. Here, we show that despite large global discrepancies (e.g., in total variation distance) between a metastable distribution and the true stationary state, single-variable conditionals remain close enough to allow accurate learning in practically interesting settings. This finding unifies ideas from statistical physics, Markov chain theory, and machine learning, demonstrating that even “wrong” data, collected from only a portion of the state space, can yield correct model estimates. Our results not only deepen the theoretical understanding of metastability but also provide a practical route to learning complex discrete distributions in settings where standard sampling methods suffer from prohibitively slow convergence.

\twocolumngrid
\section{Introduction}

Markov chains are by far the most popular tool used to study systems described by many-variable probability distributions. It is well known that Markov chains can mix very slowly to the stationary distribution and this is often associated with the existence of so-called~\emph{metastable} states, where the chain can get stuck for long periods of time~\cite{griffiths1966relaxation,kirkpatrick1987stable, cassandro1984metastable, levin2010glauber}. Slow mixing in such systems is not merely an algorithmic inconvenience but is a pervasive property observed in natural systems from molecular biology to quantum field theory~\cite{hodgman2009metastable, dinner1998metastable, del2004critical,schaefer2011critical}.
Many recent works that give rigorous guarantees on learning such distributions work under the assumption that independent and identically distributed (i.i.d.) samples are available from the ground truth distribution~\cite{diakonikolas2017learning, cule2010maximum, verleysen2003learning, wang2019nonparametric}.  This assumption is well justified in contexts where data is collected from independent agents in a real-world setting. However, it is less likely to hold in cases where the source of the data is a natural dynamical system or Markov chain sampling algorithm, due to slow mixing often exhibited by such systems~\cite{barahona1982computational, krauth2006statistical, martinelli2004relaxation}. These observations raise the following question: is it possible to learn something useful about the stationary distribution of a Markov chain given samples drawn from such a metastable state of the chain?

\textit{Prima facie}, this task looks hopeless. A chain exhibiting metastability often samples from a restricted part of a state space. This in turn implies that metastable distributions will be quite different from the stationary distribution when their difference is measured using global metrics like Kullback–Leibler (KL) divergence or total variation (TV) distance. This can have severe consequences for learning using certain algorithms that attempt to minimize global metrics between the data distribution and a hypothesis. For instance, the  Maximum Likelihood Estimator (MLE) implicitly attempts to minimize the KL divergence. Due to the large KL divergence between the metastable and stationary distributions, the MLE estimate will never approximate the true distribution well even if MLE uses a large amount of data from the metastable distribution.

However for the purposes of learning, it is not always necessary or useful to minimize such global metrics. For example, consider the problem of learning undirected graphical models. Efficient algorithms for this problem, like pseudo-likelihood (PL)~\cite{ravikumar2010high,lokhov2018optimal,wu2018sparse}, Interaction Screening~\cite{vuffray2016interaction, vuffray2019efficient} and Sparsitron~\cite{Klivans2017}, work by learning the \emph{single-variable conditionals} of this distribution (i.e., the conditional distributions of one variable where other variables are fixed). These works exploit the fact that there is a bijective mapping between positive distributions over discrete variables and their single-variable conditionals \cite{besag1974spatial}. It might then seem that samples from the true distribution are necessary to even approximately learn the conditionals. In this work, we will establish that this is not the case. That is, there exist other distributions that can be globally different from the true distribution but have conditionals that are on average very close to the ground-truth single-variable conditionals. We will explicitly show that metastable distributions of reversible Markov chains that have the true model as its equilibrium distribution satisfy such a property.

\subsection{Prior work and motivation}
The main theoretical motivation for studying learning from metastable distributions comes from the observation that many instances of multivariable discrete distributions are believed to be hard to sample from~\cite{sly2010computational, sly2012computational,barahona1982computational}. This necessarily leads to poor mixing of Markov chain samplers, which is usually observed in the low-temperature region, when the variables in the model interact strongly~\cite{mossel2013exact, levin2010glauber}. Now from the context of learning, it is natural to ask if something about the stationary distribution can be reconstructed from data produced by such a Markov chain.

For the purposes of learning, we will model the stationary distribution as a Gibbs distribution with a certain energy function, equivalently as an undirected graphical model \cite{koller2009probabilistic}. The theory of learning undirected graphical models with discrete variables in the setting where i.i.d. samples are given from the true distribution is well developed~\cite{bresler2015efficiently, vuffray2016interaction, vuffray2019efficient, lokhov2018optimal, Klivans2017, Ankur2017nips, wu2018sparse}. Recent works  have established that methods that learn the parameters of the model by minimizing metrics based on single-variable conditionals achieve efficient sample and computational complexity scaling. Some works have also demonstrated efficient learning in the setting where samples are given as a time series from a Markov chain sampler~\cite{dutt2021exponential, bresler2017learning, gaitonde2023unified, gaitonde2024efficiently}. Notice that this setting is markedly different from learning from metastable distributions. In our case, we do not observe any dynamical information from the Markov chain. Instead, we assume that the samples we have are i.i.d. from a metastable distribution of the Markov chain.

An alternative perspective on Markov chains and local learning methods has been explored in some recent works in the context of learning energy based models with continuous variables~\cite{koehler2022statistical, qin2023fit}. These works show that such models that do not bottleneck the probability flow are efficiently learnable by a local learning method known as score matching. In our work, in the discrete variable setting, we show that the presence of large bottlenecks does not impede learning.

Metastability in the context of multivariable discrete systems has been well explored in the statistical physics literature~\cite{griffiths1966relaxation, hanggi1986escape, sewell1980stability} due to the connection between such distributions and the emergence of different thermodynamic phases in models of interacting systems. A rigorous examination of this phenomenon has also been undertaken by some authors~\cite{levin2010glauber, sly2012computational} who show specific examples of regions within which stochastic dynamics can mix well but the chain struggles to escape the said region effectively. More recently, Liu et al.~\cite{liu2024locally} have also explored metastability in Markov chains, and have demonstrated rigorous algorithmic utility for data produced from these states.

\subsection{Extended summary of results}
First, we informally state the primary result of our work:\\ \emph{Let $\mu(\usigma)$ be a distribution over a set of discrete variables  and let  $P$ be a Markov chain  that has $\mu$ as its equilibrium distribution; we define a family of metastable distributions of $P$ such that given i.i.d. samples from these distributions, the stationary distribution $\mu$ can be learned to near optimal accuracy using the pseudo-likelihood method. This family of metastable distributions is defined by an approximate detailed balance condition. Moreover, metastable distributions that represent a chain stuck in a region of state space can be shown to lie in this family.} 
\newline

The rest of the paper is dedicated to defining metastable distributions, establishing these definitions by examples, and then using the notion of metastability to give learning guarantees. This leads to the discovery of elegant connections between several ideas from statistical physics and learning theory.

We begin in Section~\ref{sec:notions} by defining two notions of metastability by relaxing, respectively, the stationarity and detailed balance conditions that are obeyed by the equilibrium distribution of a reversible Markov chain $(P)$. The general definition of metastability corresponds to distributions that only change by $\eta$ in the total variation distance (TV) when updated by the Markov chain.  It follows that when $\eta = 0$ this condition is satisfied only by the equilibrium distribution. This definition of  metastability captures the intuitive notion of a metastable distribution as one that does not change much with time.  We then define a second notion of metastability which we can call \emph{$\eta$-strong metastability}. These are defined as distributions that violate the detailed balance condition  of $P$ by $\eta.$ Here the violation is again measured in TV. Just like how the exact detailed balance condition of a  distribution with respect to $P$ implies that it is a stationary state of $P,$ the strong notion of metastability can be trivially shown  to imply the  notion of metastability. To foreshadow the learning results, we will later show that samples from strongly metastable states  allow us to recover the true energy function describing the stationary distribution of $P$.

%{After introducing these two definitions of metastability, we briefly explore the connections between these two definitions. The obvious question that we explore is whether the $\eta-$strong metastability implies $\eta'-$ metastability, with either $\eta'=\eta$ or with a small difference between them. We answer this question in the negative and demonstrate significant separations between these two notions of metastability. We show that, in the worst case, the ratio between these two quantities can even scale as the size of the state space. These general separations are constructed using the well-known mapping between reversible Markov chains and electrical networks. This also leads to an interpretation of metastability in terms of current flows which might be of independent interest.}

\begin{figure*}[!htb]
    \centering
    \includegraphics[width=0.61\linewidth]{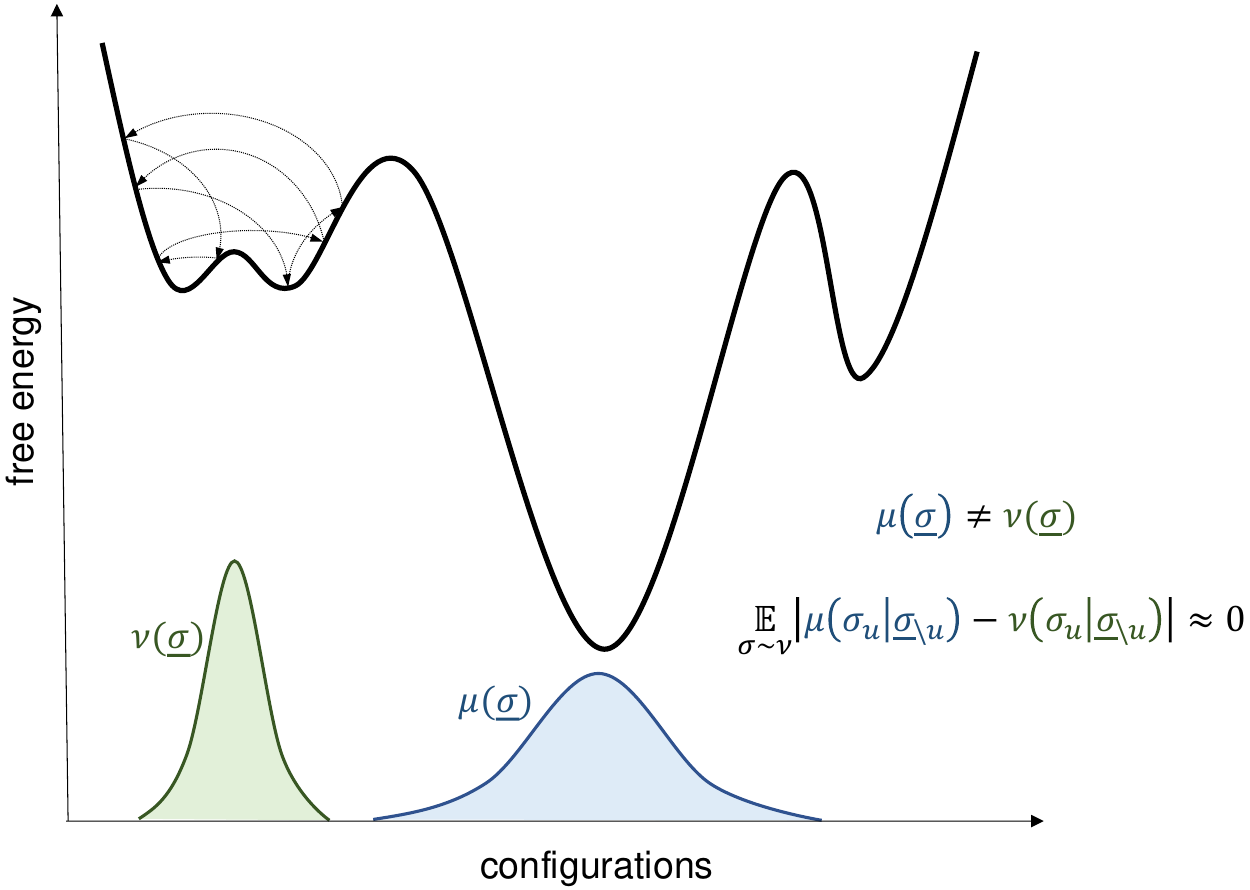}
    \caption{An informal representation of our result given by Theorem 1. 
    Samples coming from a metastable distribution of reversible Markov chain samplers are far from the full measure in global metrics. Surprisingly, at the same time we show that single-variable conditionals in metastable distributions are on average close to those of the true distribution. In the Curie-Weiss model, we will use an explicit construction to demonstrate that such metastable states correspond to the local minima of the free energy, which agrees with an intuitive statistical physics picture of metastability.}
    \label{fig:schematic}
\end{figure*}

The question remains whether there are interesting or relevant cases of strong metastable distributions. We show two explicit constructions of strongly metastable distributions that show these states do exist and that they align with the  intuitive notion of a metastable system as one that is stuck in a region of state space, or as a system trapped in local minima of free energy \cite{cahn2003material, PEREPEZKO2005247}.
     For general reversible Markov chains, we construct $\eta-$strongly metastable states by using the notion of \emph{conductance} from theory of Markov chains. We show that any subset in the state space supports a $\eta$-strong metastable state with $\eta$ proportional to the conductance of the subset. For the case of slow mixing chains this implies that there exists an $\eta$ strongly metastable state with $\eta$ being exponentially small in the number of variables in the system. Moreover, the explicit constructions we have correspond to distributions resulting from Markov chains that have mixed well inside a set but have no support outside of it, essentially being stuck in such a set.
     In the final section of the paper, we illustrate another class of strongly metastable states in the Curie-Weiss model. These distributions are peaked around the minima of the free energy of the system and can also be shown to have an $\eta$ that decays with the number of variables in the model.
    
After establishing the notion of strong metastability we move on to showing the connection between strongly metastable distributions and learning discrete distributions. The key connection we show is that strongly metastable states have on average almost the same single-variable conditionals as the equilibrium distribution $\mu$, which is the true distribution we are ultimately interested in learning.

For our learning setup, we  model our distributions as Gibbs distributions with some energy function ($\mu(\usigma) \propto e^{E^*(\usigma)})$ and aim to learn this energy function by optimizing over a parametric family of energy functions using the pseudo-likelihood method.  %Additionally, we assume two natural conditions on the Markov chain ($P$) and the ground truth distribution ($\mu$)
Now given some natural assumptions about the learning problem (see Condition \ref{cond:bounded_flip} and \ref{cond:temperature_gen}) we rigorously establish that strongly metastable states are almost as good as the equilibrium distribution for learning the energy function. Below we informally state the main result that facilitates this.

\ca
\emph{Condition \ref{cond:bounded_flip}: Bounded spin-flip} We assume that there is a non-zero ratio between the single-variable conditionals of the true distribution $\mu$ and the probability of the chain P changing only one variable. This is a mild condition as reversible Markov chain samplers like Glauber dynamics or Metropolis-Hastings samplers satisfy this condition with the ratio going as $\Theta(1/n)$, which is sufficiently large to show non-trivial learning guarantees. We call this ratio $\omega_P.$

\emph{Condition \ref{cond:temperature_gen}: Bounded interactions/ temperature bound} We assume that the energy functions that we optimize over only change by at most $2 \gamma $ when a single-variable is changed in their input. For sparse models, this corresponds to a bounded on inverse temperature of the model which is a necessary quantity that is seen to control the sample complexity of this learning problem even in the setting where the samples come from the true distribution \cite{lokhov2018optimal,santhanam2012information}. 

From these natural assumptions, we rigorously establish that strongly metastable states are almost as good as the equilibrium distribution for learning the energy function. Below we informally state the main result that facilitates this.
\cb

\paragraph*{ Theorem \ref{thm:close_condtionals}: Closeness of conditionals } The single-variable conditionals of an $\eta-$ strongly metastable state of $P$ differ from those of the equilibrium distribution $\mu$ in average TV distance by only $O(\eta).$ The average in this case is taken over the conditioning variables using the metastable distribution.

This is schematically represented in Figure~\ref{fig:schematic}.
This result is crucial for learning as the PL method works by minimizing the average KL divergence between single-variable conditionals. Specifically, given samples from the data distribution and a parametric family of energy functions, the PL method estimates the energy function from the data by minimizing the average KL divergence between the single-variable conditionals of the data and the model, averaged over the data distribution \cite{besag1974spatial, ravikumar2010high}. The data distribution in our case is given by the i.i.d. samples drawn from a strongly metastable distribution.
Now using the well known bounds connecting TV to the KL divergence we can argue that the PL loss  computed from samples drawn from  a strongly metastable distribution is close to its optimal value when the parameters of the model hypothesis match that of the stationary distribution of $P$.

Now notice from our constructions that $\eta$ for strong metastable states is often a decaying function of the number of variables for slow-mixing systems.  This implies that for large systems that exhibit slow mixing, the $O(\eta)$ bias incurred in the PL loss function is dominated by the statistical errors that scale as $M^{-1/2}$.  These bounds on the PL loss computed from metastable samples are given formally in Theorem \ref{thm:PL_close}.

These ``nearly optimal'' results can be further developed into parameter learning guarantees. We demonstrate this for the special case of binary undirected graphical models with pairwise interactions. In other words, we solve the inverse Ising problem given data from metastable states of the model \cite{aurell2012inverse, nguyen2017inverse}.

We use tools from statistical learning theory to show that the parameters of the energy function of pairwise binary (Ising) exponential family models can be reconstructed from samples from metastable states. We show that the coupling and magnetic field parameters of such models can be learned using $e^{O(\beta d)} \log(n)$ samples for an $n$ spin system with only a bias that scales with $\eta$. Hence in the most interesting cases (i.e. like in Section \ref{sec:cheeger_metastable} for a slow mixing chain) this bias is exponentially small in the number of spins and is completely dominated by the statistical error in the learning process.

In the final section of the paper we give a numerical illustrations of our results using the  Curie-Weiss ferromagnet and  a 
spin glass model with higher-order interactions. By exploiting the permutation symmetry of the Curie-Weiss model, we numerically show that there are strongly metastable states supported around the minima of the free energy of this model. We also collect metastable data from this model by running Glauber dynamics around one of the free energy minima and show that the full model parameters can be learned faithfully from this data using the PL method. For the spin glass model, we demonstrate learning from metastable distributions on a spin glass model with higher-order interactions and a three level alphabet.

\ca
The following are the main contributions of this work. In Section ..., we rigorously define two related notions of metastability.  While metastable distributions of the Markov chain have been understood at a non-rigorous level in statistical physics, we give a concrete definition of this. ...
 In Section ... we connect strong metastability to the theory of Markov chains and explicitly show constructions of such metastable distributions.

 In Section ... we connect strong metastability to learning graphical models. We do this by showing that strongly metastable distributions have on average single-variable conditionals very close to that of the ground truth distributions. This implies that the ground truth can be recovered even if we have samples from such metastable distributions using a method that works by reconstructing the single-variable conditionals. We show this rigorously for the case of the pseudo-likelihood method.

We also demonstrate this effect numerically on the Currie-Weiss model and a Spin Glass model. For the Curie-Weiss model, we also analyze the model using its free energy to give a new characterization of the minima of the free energy in terms of the single-variable conditionals. 

\textcolor{gray}{Explain contributions}
\cb

%\subsection{Motivating experiment: Learning the Curie-Weiss Ferromagnet}

\section{Definitions of metastability}
\label{sec:notions}

 A metastable distribution must be a distribution that behaves similarly to the stationary distribution under a time update. From this, a natural definition of a metastable distribution can be given  by bounding the change in the distribution after one step of a Markov chain. 

\begin{definition}[Metastability]
A distribution $\nu$ is $\eta$- metastable with respect to a Markov chain P if and only if,
\begin{equation}
    |\nu - \nu P|_{TV} \leq \eta
\end{equation}
\end{definition}

This definition most naturally captures the idea of a metastable distribution. A similar notion of metastability has been recently defined in \cite{liu2024locally} using the KL divergence instead of TV. Later we will show explicit constructions where $\eta$ is smaller than some inverse polynomial function of the size of the state space of the chain. In the language of high-dimensional distributions, this implies an exponentially small quantity in the number of variables. It can be shown that if the Markov chain is started from such a distribution then that chain will mix poorly to the stationary distribution.

\begin{proposition}\label{prop:slow_mix}
    Starting from an $\eta-$ metastable distribution it takes at least $\frac{|\mu - \nu|_{ TV} - \epsilon}{\eta}$ steps to get $\epsilon$ close to the equilibrium distribution $\mu$ in TV.
\end{proposition}

This simple fact is a consequence of the data-processing inequality. The proof is given in \SupI \ref{app:slow_mix}.

However, almost all Markov chain samplers used in practice have a stationary state that satisfies time reversibility, which is captured by the detailed balance condition. For reversible Markov chains, the detailed balance condition can be shown to imply the existence of a stationary distribution. For a distribution $\mu$ and a chain $P$ with a state space $\scc$, the following implication holds true,
\begin{equation}
    P(i|j) \mu(j) = P(j|i) \mu(i), ~~ \forall i,j \in \scc \implies \mu P = \mu
\end{equation}

Now if we take the view that this time reversibility condition must only be mildly violated for a metastable distribution then a stronger notion of metastability emerges.

\begin{definition}[Strong Metastability]
Let $P$ be a Markov chain with a state space $\scc$. Then a distribution $\nu$ is $\eta$-strongly metastable with respect to $P$ iff,

\begin{equation}\label{eq:strong_def}
    \frac12\sum_{i,j \in \scc } |P(i|j)\nu(j) - P(j|i)\nu(i)|~ \leq ~\eta
\end{equation}
    
\end{definition}

If we set $\eta$ to zero in the definition of strong metastability, then we get the detailed balance condition, $P(i|j) \nu(j) = P(j|i) \nu(i).$ As foreshadowed by their names, it can be easily deduced from their definitions that strong metastability implies metastability with the same $\eta$.

In the following sections we will show explicit constructions of strongly metastable distributions that apply to general reversible chains. We will then use this definition of strong metastability to show how the equilibrium distribution can be learned given samples drawn from such a state. We also discuss more closely the relationship between these two notions of metastability in \SupI~ \ref{appendix:meta_separation}. Specifically in Proposition \ref{prop:strong_weak},  we show that it is possible to have large separations between these two measures of metastability, implying that the strong metastability condition is indeed the relevant condition to look for in the context of learning.

\ca
\ajcomm{Condense the separation results and constructions }

-----------------

In the following sections we will show explicit constrictions of strongly metastable distributions that applies to general reversible chains. We will then use this definition of strong metastability to show how the equilibrium distribution can be learned given samples drawn from such a state.

Notice that either of these definitions of metastability does not directly imply that the metastable distribution is close to the equilibrium distribution. In fact, in practice, it is often seen that metastable distributions are supported in subregions of the state space from which the Markov chain struggles to escape. Hence it is surprising that the parameters of the true model may be reconstructed from such states.

%\subsection{Relation between strong and weak metastability.}

% We will give explicit examples of such states in the following sections and 

These two definitions of metastability raises an interesting question: \emph{For a reversible Markov chain does $\eta-$  metastability imply $\eta-$ strong metastability?}
We can construct a simple counter-example showing that is not true. For an even integer $L$, let $\scc = [L]$ and  $P$ be the Markov chain corresponding to the random walk on a cycle graph on $L$ elements. Take $\nu$ to be the following `tent' distribution,
\begin{equation}
\nu(i)  = \frac{1}{Z} 
\begin{cases}
i  L ,~~ &i \leq L/2, \\
-i  L + L^2 ,~~ &i > L/2, \\
\end{cases}
\end{equation}

The $Z$ here ensures normalization. Now by direct computation, we can verify that this state is $\frac{2L}{Z}$ metastable while being $\frac{L^2}{2Z}$ strongly metastable. This shows that there can be $\Theta(|\scc|)$ factor difference between the two measures of metastability.  This construction can be generalized to general reversible Markov chains using by mapping them to an electrical network and analyzing the current flows.

\begin{proposition}\label{prop:strong_weak}
Let $G = (\scc,E)$ be a graph defined on the state space with the edge set $E = \{(i,j) | P(i|j) \neq 0 \}.$  Let $diam(G)$ be the diameter of this graph. Then the chain $P$ has a $\eta$- metastable distribution that is $\Omega(diam(G) \eta)-$strongly metastable.
\end{proposition}
We conjecture that this construction does not give the optimal separation between metastable and strongly metastable distributions. The reason is that the diameter is a topological property of the chain and it can be changed drastically by adding a few well-chosen, very small but non-zero transitions to the chain.

While the relationship between these two notions of metastability is an interesting question, it is not the main focus of this paper.  We leave further exploration of this for future work.

\cb

\ca
{
\begin{proposition}
\label{prop:sm2}
Let $P$ be a reversible Markov chain and $\nu$ be an $\eta$ metastable distribution with  $\nu(I - P)  =  \frac{\eta}{2}(e_s - e_t).$ Then,
\begin{equation}
    \frac12\sum_{i,j \in \scc } |P(i|j)\nu(j) - P(j|i)\nu(i)|~ \geq \frac{\eta}{2}\sqrt\frac{\min_{ij:P(i|j) \neq 0}P(i|j)\mu(j)}{ \mu(s) p^{s,t}_{\text escape}}.
\end{equation}
Here $ p^{s,t}_{\text escape}$ is the probability that the Markov chain starting from $s$ reaches $t$ before returning to $s.$
\end{proposition}
}
\cb

\subsection{Existence of strongly metastable distributions for reversible Markov chains}
\label{sec:cheeger_metastable}

From the definition of strong metastability, it is not clear whether metastable distributions seen while sampling from challenging many-variable problems satisfy this condition.  We will show that strong metastability is closely tied to regions in the state space that bottleneck the probability flow. This will further imply that any slow-mixing reversible Markov chain must necessarily have strongly metastable distributions.

To this end let us first define \emph{conductance}, which captures the ease with which probability can flow out of a certain region in phase space.

\begin{definition}[Conductance]
Given a Markov chain $P$ on a state space $\scc$ with stationary state $\mu$. The conductance of any set $A \subseteq \scc$,
\begin{equation}
    \Gamma(A) \coloneqq \frac{\sum_{j \in A, i \in A^c} P(i|j)\mu(j)}{\sum_{j \in A} \mu(j)}
\end{equation}
From this, the conductance of the chain is defined as,
\begin{equation}
    \Gamma \coloneqq \min_{A: \mu(A) < 1/2} \Gamma(A)
\end{equation}
\end{definition}

Conductance is a well-studied property of Markov chains and forms the basis of many classic results in randomized algorithms \cite{jerrum1988conductance, dyer1991random, diaconis1991geometric, aldous2002reversible}. The usefulness of this property comes from the Cheeger bound, which connects it to the spectral gap of the chain and in turn to the mixing time.

\begin{lemma}[Cheeger bound \cite{jerrum1988conductance}]
Let $\lambda_2$ be the second largest eigenvalue of a reversible, irreducible Markov chain $P$ with conductance $\Gamma$. Then,
\begin{equation}
    2\Gamma \geq 1 -\lambda_2 \geq\frac{\Gamma^2}{2}
\end{equation}
\end{lemma}

Now let $1 = \lambda_1 > \lambda_2 \geq \ldots \lambda_{|\scc|} > -1 ,$ be the eigenvalues of a reversible, ergodic, irreducible chain. It is well known that the mixing time of such a chain is determined by the inverse spectral gap, $\tau_{mix} = \Theta(1/(1 -\lambda_2))$ \cite{levin2017markov, aldous2002reversible}. %This holds under the assumption that $|\lambda_{|\scc|}| < \lambda_2,$ but this can be easily satisfied by considering the lazy-version of the chain \footnote{For Glauber dynamics, this assumption can be shown to hold for large enough interaction strengths \textcolor{red}{AJ:Check!}}.

Now let us look at the implications of these statements for a reversible Markov chain that is attempting to sample from a family of distributions on $n$  discrete variables, each taking a value from a set $Q$. In this case $\scc = Q^n$. Furthermore, assume that the chain is only allowed to make moves that change at most one variable at a time. This condition is satisfied by both Glauber dynamics and the Metropolis-Hastings algorithm.  Slow mixing in such a chain implies that the mixing time of the chain scales exponentially with $n$. From the Cheeger bound this implies that there exists a set with an exponentially small conductance. We will now show that such a set will necessarily support a strongly metastable distribution.

For $A \subset \scc$ define the following distribution,
\begin{equation} \label{eq:meta_explicit}
        \mu_A(\usigma) \coloneqq
    \begin{cases}
   \frac{\mu(\usigma)}{\mu(A)},~~&\usigma \in A, \\
   0,~~&\usigma \in A^c. \\
    \end{cases}
\end{equation}

Observe that the detailed balance condition is only violated at the boundary of this set. From this observation we can directly compute the strong metastability of such a state,
\begin{multline}
    \frac12\sum_{\usigma,\usigma'\in \scc } |P(\usigma|\usigma')\mu_A(\usigma')- P(\usigma'|\usigma)\mu_A(\usigma)| =  \\
        \sum_{\usigma \in A, \usigma' \in A^c } |P(\usigma'|\usigma)\mu_A(\usigma)| = \Gamma (A).
\end{multline}
Thus by applying Cheeger bound to a slow-mixing reversible Markov chain $P$ with stationary distribution $\mu$, we conclude that there exists an $\eta$-strongly metastable distribution with $\eta$ exponentially small in the number of spins.

In statistical physics, the idea of a metastable state is connected to the minima of a free energy function \cite{griffiths1966relaxation}. We also explore this connection in Section \ref{sec:numerics} and demonstrate the existence of strongly metastable distributions supported around the minima of the free energy of the Curie-Weiss model.

    \subsection{Robustness of strong metastability}
To gauge the utility of the notion of strong metastability, we need to understand how it behaves when a mixture of such distributions is taken and how it behaves under dynamics. We show that the property of strong metastability is robust to both these actions.

\begin{proposition}\label{prop:robust_mix}
Let $\{\nu_k\}$ be a collection of $\eta_k$-strongly metastable distributions. For any convex combination $\nu = \sum_k p_k\nu_k$, with $p_k \ge 0$ and $\sum_k p_k = 1$, the distribution $\nu$ is $\left(\sum_k p_k\eta_k\right)$-strongly metastable.
\end{proposition}

This property is important because the dynamics can transition between metastable states, so the data used for learning may be drawn from mixtures of such distributions. It is also necessary to make sure that a strongly metastable distribution retains this property as it evolves according to the Markov chain. For the weaker notion of metastability, data processing inequality implies that if $\nu$ is $\eta$-metastable, then so is $\nu P$. We can show that an similar statement can be made for strong metastability.

\begin{proposition}\label{prop:robust_dynamics}
Let $\nu$ be an $\eta$-strongly metastable distribution for a reversible Markov chain $P$. Then, for every integer $t \ge 1$, the distribution $\nu P^{t}$ is $(1 +  2 t)\eta$-strongly metastable.
\end{proposition}

%This is important, as in practice the samples are collected from a dynamic that is continoually evolving. The last two propostions show that such samples can be seen to be coming from a mixture of strongly metastable distribtions, which inherits the strong metastability 

Proofs of Propositions \ref{prop:robust_mix} and \ref{prop:robust_dynamics} are given in \SupI~\ref{appendix:robustness_proofs}.

\section{Learning from metastable distributions}

Now that we have established that interesting cases of strongly metastable distributions exist, we will show that the energy function of the true distribution can be learned given samples from a strongly metastable distribution.
The key result, that we will exploit to learn from metastable distributions, is the fact that strong metastability when combined with the time reversibility of the chain, implies that the single-variable conditionals of the metastable distribution are close to those of the stationary state of the Markov chain. Building on this, we then use prior information on the true model and stochastic convex optimization techniques to show that learning from metastable distributions can be successfully performed using the Pseudo-likelihood method.% or Interaction screening.

To this end, we assume a technical condition on the transition probability of a reversible Markov chain.
\vspace{0.05in}
\begin{condition}[Bounded spin-flip probability]
\label{cond:bounded_flip}
For a reversible chain $P$, let $P(\usigma_{u \rightarrow q}|\usigma)$  be the probability of variable at position  $u$ being changed to $q \in Q$  when the system is in the configuration $\usigma$. Then we assume that the ratio of this to the single-variable conditionals is lower bounded by a quantity that depends on $P$,
\begin{equation}
   \frac{ P(\usigma_{u \rightarrow q}|\usigma)}{ \mu(q| \usigma_{\su})}  \geq \omega_P > 0, ~~\forall ~\usigma_{\setminus u } \in Q^{n-1}, ~~q \in Q, ~u \in [n].
\end{equation}
\end{condition}

The bound in the above condition is satisfied by both Glauber dynamics and the Metropolis-Hastings sampler \cite{mackay2003information}.
For Glauber dynamics, by definition of the algorithm $\omega_P = \frac{1}{n}$. For the Metropolis-Hastings Markov chain, $\omega_P $ will also be proportional to $1/n$, with the ratio between them only depending on the inverse temperature of the chain.

We will later see that $\omega_P$ will control the error in the learning procedure and metastable states with small $\eta/\omega_P$ give us good learning guarantees.

Now we state one of the main results of the paper, the proof of this result is given in \SupI ~\ref{appendix:proof_thm1}

\begin{theorem}[Conditionals of strong metastable distributions]
\label{thm:close_condtionals}
Consider a system of $n$  discrete random  variables where each of them can take values  from  the set $Q$.  Let $\nu$ be an $\eta-$strongly metastable distribution of a reversible Markov chain $P$ with a stationary state $\mu$. If this chain satisfies Condition \ref{cond:bounded_flip}, then the single-variable conditionals of $\nu$ are close to those of $\mu$ in the following sense,
\begin{equation}
  \sum_{u = 1}^n  \sum_{\sigma \in Q^n} \nu(\sigma)  \left| \nu(.|\usigma_{\su}) - \mu(.|\usigma_{\su})\right|_{TV} \leq \frac{\eta}{\omega_P}
\end{equation}
\end{theorem}

This is a central observation in this paper which says that a strong metastable distribution has on average single-variable conditionals that are very close to that of the equilibrium distribution. This result will be the basis for showing that the parameters of the true distribution ($\mu$) can be learned given samples from ($\nu$). This is due to the fact that methods like PL and Interaction Screening that efficiently learn discrete Gibbs distributions rely on minimizing average metrics over the single-variable conditionals. 

Theorem \ref{thm:close_condtionals} should be contrasted with the explicit constructions in Equation \eqref{eq:meta_explicit} that gave strongly metastable distributions that are considerably far from the equilibrium in terms of global TV.  According to that construction, the TV distance between the metastable distribution from the equilibrium distribution is, $|\mu_A - \mu|_{TV} = 1 - \mu(A),$ which is generally not an exponentially small quantity even if the Markov chain mixes slowly.  On the other hand, the distance in Theorem \ref{thm:close_condtionals}  for $\nu = \mu_A$ is $\Gamma(A)/\omega_P$  which can be exponentially small if $A$ is a set that causes slow mixing in the Markov chain.  We demonstrate this fact explicitly for the Curie-Weiss model in Section~\ref{sec:numerics}. 

Theorem \ref{thm:close_condtionals} implies that the metastable distribution supported on one of the minima of the free energy function has conditionals that are very close to the true distribution. %For instance, the average distance between the quartic state $\nu^{(4)}$ and $\mu$ scales as $O(1/n)$ for an $n$ spin CW model. 
This can be true even if the true model has very low mass within the support of the metastable state, as evidenced by the numerical experiments in Section \ref{sec:numerics}.
Consequently, learning algorithms that minimize global distance measures between distributions will not fare well for learning $\mu$ given data from a strongly metastable state. For instance, Maximum likelihood estimation minimizes the KL divergence between the data distribution and a proposed parametric hypothesis. The KL divergence between $\mu_A$ and $\mu$ is given by $-\log(\mu(A))$. Thus the true model $\mu$ is very far from being optimal for MLE when data comes from the metastable state $\mu_A$. Also, the MLE can be seen as trying to match the sufficient statistics of the data to the hypothesized model \cite{mackay2003information}, and it is usually the case that sufficient statistics of metastable distributions do not match that of the stationary distribution. So, for the purpose of recovering the true model from a metastable state, we have to rely on minimizing the distance between conditionals and not global metrics. 

Theorem \ref{thm:close_condtionals} can be modified  to give closeness in terms of other metrics like the average KL divergence between conditionals. This is a simple consequence of the reverse Pinsker inequality (Lemma 4.1 in \cite{10.1214/19-EJP338}), which upper bounds the KL divergence between two distributions in terms of the TV. Given distributions $p$ and $q$, we have, $D_{KL}(p || q) \leq \frac{2}{\min_{i} q(i)} | p -q|_{TV} $  . Using this relation directly on Theorem \ref{thm:close_condtionals}  we get,

\begin{corollary} \label{cor:KLclose}
For distributions $\mu$ and $\nu$ as defined in Theorem \ref{thm:close_condtionals}, for every $u \in [n]$  we have the following bound,
\begin{equation}
   \sum_{u=1}^n\Enu D_{KL} \left( \nu(.|\usigma_{\su}) || \mu(.|\usigma_{\su})\right) \leq \frac{2\eta}{\omega_P \min_{u,\usigma} \mu(\sigma_u| \sigma_{\su})}.
\end{equation}
\end{corollary}

This metric is specifically important to bound, as the PL estimator precisely minimizes the average KL divergence between conditionals . The lower bound on the conditional is a natural quantity that controls the error in these types of learning results \cite{Klivans2017, lokhov2018optimal, ravikumar2010high, santhanam2012information}. For instance, in the Sparsitron algorithm \cite{Klivans2017}, the authors use the term \emph{$\delta-$unbiased } conditionals to refer to the existence of such a lower bound. 
\subsection{Test error bound for multi-alphabet pseudo-likelihood method / logistic regression}

We will now look at the consequences of the results above for learning the distribution $\mu$ given independent samples from an $\eta-$strongly metastable distribution $\nu$. For this purpose we will parametrize the ground truth distribution as a Gibbs distribution associated with an energy function, 
    $\mu(\usigma) \propto \exp(E^*(\usigma)).$

Here is $E^*$ is a general real valued function on the state space $Q^n$ and for strictly positive distributions such an energy-based representation does not restrict the distribution in any way. We will be using PL to learn the true energy function associated with $\mu$.

Now for the PL algorithm it is necessary to have a parametric hypothesis for the energy. The PL method is then essentially a regression over this parametric family to learn the true energy. Denote these parametric energy functions as $E(\usigma, \utheta)$ where $\utheta$ is a set of parameters. For example, one particular choice of a parametric family could be polynomials of a certain degree over the variables $\usigma$, in which case $\utheta$ can be chosen as the coefficients in the polynomial. Another parametrization, that has gained popularity in ML literature, is a neural-network based parametrization in which case $\utheta$ can be seen as the set of weights and biases of the network  \cite{song2021train,abhijith2020learning}. We assume that possible choices for $\utheta$ lie in a set $\Theta$, which is informed by the prior information we have about the data. The most common example of such a $\Theta$ is the $\ell_1$ ball, which is often chosen to enforce sparsity in learning tasks. Furthermore we demand that the true model lies inside this hypothesis set, i.e. that there is a $\utheta^* \in \Theta$ such that $E(\usigma, \utheta^*) = E^*(\usigma)$.

Given a parametric energy function, this defines a parametric distribution  $p(\usigma, \utheta) \propto \exp(E(\usigma, \utheta)),$ on the $\usigma$ variables. This in turn defines the following paramteric single-variable conditionals,

\begin{equation}
    p(\sigma_u = q|\usigma_{\setminus u}; \utheta) = \frac{1}{1 + \sum_{p \in Q, p\neq q } e^{ E(\usigma_{u \rightarrow p}, \utheta ) - E(\usigma_{u \rightarrow q}, \utheta)}}
\end{equation}

Now given $M$ independent samples $\usigma^{(1)}, \usigma^{(2)}, \ldots \usigma^{(M)}$  each drawn from a metastable distribution $\nu$, the PL estimate for the model parameters is given by minimizing the average negative log-likelihood of the single-variable conditionals.

\begin{equation} \label{eq:PL_loss_gen}
   \hutheta  = \argmin{\utheta \in \Theta}  \frac1{M}\sum_{u=1}^n \sum_{t = 1}^M ~ \lc_u(\utheta, \usigma^{(t)}).
\end{equation}
\begin{equation}
\lc_u(\utheta, \usigma) := \log\left( 1 +  \sum_{p \in [Q], p\neq q } e^{E(\usigma_{u \rightarrow p}, \utheta ) - E(\usigma_{u \rightarrow q}, \utheta)}\right).
\end{equation}

Now if the samples came from $\mu,$ it is well established that this is a consistent estimator of the model parameter \cite{lokhov2018optimal, ravikumar2010high, besag1974spatial}. Our results in the preceding section suggest that this is still a good estimator if the samples come from a strongly metastable state. We can rigorously show that this is indeed the case rather easily if we assume the following condition on the parametric energy function,

\begin{condition} [Finite interaction strength]
\label{cond:temperature_gen}
For every $\utheta \in \Theta$ (including $\utheta^*$), the change in energy function caused by perturbing  the value of $\usigma$ at a single site $u \in [n]$ is upper bounded as follows,
\begin{equation}
  \frac12 |E(\usigma,\utheta) - E(\usigma_{u \rightarrow q}, \utheta)| \leq \gamma, ~~\forall \usigma \in Q^n, q \in Q, u \in [n]. 
\end{equation}
\end{condition}

For the case of polynomial energy functions often considered in the rigorous learning theory literature, this bound essentially boils down to an upper bound on the $\ell_1$ norm of the parametric coefficients multiplying the model. Further for the specific case of  Ising model energy functions defined on bounded degree graphs, the $\gamma$ here will just be the product of an inverse temperature of the model and the degree of the graph. 

The finite interaction strength condition directly leads to the following bound on the single-variable conditionals,
\begin{equation}
\label{eq:condtional_bounds}
 \frac{1}{1 + (|Q|-1) e^{2  \gamma }}\leq   p(\sigma_u | \usigma_{\su }, \utheta) \leq \frac{1}{1 +  (|Q| - 1)e^{ -2  \gamma}}.
\end{equation}
    
Given this bound on the conditionals and in turn on the conditional likelihood, we can bound the test error of the PL method. For the parameters $\hutheta$ learned from samples drawn from a metastable state, we can bound the test deviation of the PL loss function at $\hutheta$ from true model at $\utheta^*.$

\begin{theorem}
 [Test error for logistic regression with metastable samples]
\label{thm:PL_close}
Given $M'$ independent samples $\usigma^{(1)}, \ldots, \usigma^{(M')}$, from an $\eta$-strongly metastable distribution $\nu$ of a reversible Markov chain, the true graphical model parameters are nearly optimal for PL in the following sense,
 \begin{align}
  \frac{1}{M'}\sum_{t,u}\lc_u(\utheta^*, \usigma^{(t)}) -  \lc_u(\hutheta, \usigma^{(t)}) \leq  \frac{2(1 + (|Q|-1) e^{2\gamma}) \eta}{\omega_P} + \nonumber  \\ \log(e^{2\gamma} (|Q|-1) )\sqrt{ \frac{ \log(\frac{1}{\delta})}{2M'}} .
\end{align}
 Here the quantities $\omega_P$ and $\gamma$ are related to the Markov chain and the prior in the PL estimation via conditions \ref{cond:bounded_flip} and \ref{cond:temperature_gen}.  
 \end{theorem}

 The $O(1/\sqrt{M'})$ term here is the statistical error in the estimation of the loss function. Hence, if $M'$  is small enough this statistical error will swamp the $O(\eta/\omega_P)$ bias introduced by having bad data. Looking at the connection between $\eta$ and mixing times, we can see that for slow mixing Markov chains this bias can be exponentially small in $n$. Hence we can conclude that there is a family of strongly metastable distributions for which the PL estimate of the energy function ($\hutheta$) asymptotically $(n,M \rightarrow \infty)$ approaches the true model parameters ($\utheta^*$). This observation is somewhat surprising, as this tells us that the true model can be nearly reconstructed from ``bad'' data using the PL estimator.

 Theorem \ref{thm:PL_close} is a very general result that applies to any discrete distribution as long as the two conditions are satisfied. However, this result by itself does not guarantee that the estimated parameters from \eqref{eq:PL_loss_gen} will closely match that of the true distribution. To give such learning guarantees we have to show that the loss function has sufficiently large curvature near the optimal point.  This would then imply that the small changes in the loss values translate to small changes in the estimated parameters. \footnote{See Figure 2 of Reference \cite{neghaban2012high} for a visual explanation of this point}

\ca

\subsection{Learning Ising models from metastable states}

The result in Theorem \ref{thm:close_condtionals} and the ensuing implications can be applied to general distribution over discrete variables independent of the graphical model framework.

We define the ground-truth Ising model on a set of $n$ spins, $\sigma_u \in \{1,-1\}, u \in V$, where $V = [n]$ is the vertex set of this model. In the most general form,\cite{koller2009probabilistic},
\begin{equation}
    \mu(\usigma) =\dfrac{\exp \left( \sum_{k \in \kc} \theta_k^*  \usigma_k\right)}{Z}.
\end{equation}

Here $\kc \subseteq 2^V$ is the set of interactions present in the model and $\usigma_k \coloneqq \prod_{u \in k} \sigma_u.$  For example, the CW model fits into this form with $\kc$ being the set of all tuples of size less than three. The function in the exponent, $E(\usigma;\utheta^*) \coloneqq  \sum_{k \in \kc} \theta_k^*  \usigma_k $  is known as the \emph{energy function} of this model. Additionally, it is useful to define a \emph{local energy} function $E_u(\usigma)$ for every $u \in V.$ Let $\kc_u = \{k | k \in \kc, u \in k\},$ then the local energy function is  defined as,
\begin{equation}
    E_u(\usigma;\utheta^*) = \sum_{k  \in \kc_u} \theta^*_k \usigma_k
\end{equation}

The importance of the local energy comes from the fact that it determines the single-variable conditionals of the model,
\begin{equation}
    \mu(\sigma_u | \usigma_{\su }) = \frac{1}{1 + \exp \left( -2 E_u(\usigma; \utheta^*) \right)}.
\end{equation}

In the learning task considered in most of the prior works \cite{vuffray2019efficient}, we are given i.i.d. samples from $\mu$ and we attempt to reconstruct the $\utheta^*$ parameters from them. An important feature of the sample complexity of this problem is its fundamental dependence on the strength of the local interactions i.e. $\sum_{k \in \kc} | \theta^*_u |.$ This feature is observed both in the sample complexity of algorithms that solve this problem and also in the known information-theoretic lower bound \cite{santhanam2012information}.  An unbounded interaction strength can cause the sample complexity of the learning task to blow up. This is obvious as the parameters in the energy cannot be reconstructed in general if we only have access to the minimum energy states. To mitigate this intrinsic pathology, we follow prior literature and assume a finiteness condition on the interaction strengths.

\begin{condition} [Finite interaction strength]
\label{cond:temperature}
Define $\utheta^*_u := \{\utheta^*_k| k \in \kc_u\}$. Then the local interaction strengths are bounded as follows,
\begin{equation}
||\utheta^*_u||_1 = \sum_{k \in \kc_u} |\theta^*_k| \leq \gamma < \infty,~~ \forall~u \in V
\end{equation}
\end{condition}

Crucially this condition gives size independent upper and lower bounds on the conditional probabilities which ultimately control the sample complexity of learning algorithms like PL,
\begin{equation}
\label{eq:condtional_bounds}
 \frac{1}{1 + \exp \left(2 \gamma \right)}\leq   \mu(\sigma_u | \usigma_{\su }) \leq \frac{1}{1 + \exp \left( -2 \gamma \right)}.
\end{equation}

In the learning task we are interested in this paper, we assume that the samples come from an $\eta-$ strongly metastable distribution, $\nu$, of reversible Markov chain $P$. Moreover, this Markov chain has $\mu$ as its unique stationary state. As the Markov chain is reversible this implies that the detailed balance equations are satisfied,
\begin{equation}
    P(\usigma'|\usigma) \mu(\usigma) = P(\usigma| \usigma) \mu(\usigma').
\end{equation}

\subsubsection{Pseudo-likelihood method}

Going forward we will focus on PL and explore the consequences of Theorem  \ref{thm:close_condtionals} on this method.

Given samples $\usigma^{(1)}, \ldots, \usigma^{(M)}$ drawn from the metastable distribution $\nu$. The PL estimate of the parameters in the energy function can be written as the solution to the following $\ell_1$ constrained optimization problem  of minimizing a specific loss function\cite{ravikumar2010high},

\begin{equation} \label{eq:PL_loss}
   \hutheta_u  = \argmin{|| \utheta_u||_1 \leq \gamma~}   \frac{1}{M}\sum_{t = 1}^M ~ \lc(\utheta_u, \usigma).
\end{equation}
\begin{equation}
\lc(\utheta_u, \usigma) := \log( 1+  \exp(-2 \sum_{k \in \kc_u} \theta_{k} \usigma_k) ).
\end{equation}

The $\ell_1$ constraint reflects the prior information we have about the strength of the optimal parameters. But we do not assume assume that the $\gamma$ here is the most optimal $\gamma$ we can choose in Condition \ref{cond:temperature}.

Now in the $M \rightarrow \infty$ limit, one can show that this minimizes the average KL divergence between the conditionals of $\nu$ and the single-variable conditionals of the parametrized distribution, $p(\usigma;\utheta) \propto \exp(\sum_{k \in \kc} \theta_k \usigma_k ).$ To see this connection, notice that the single-variable conditionals of the parametric distribution take the form $p(\sigma_u|\usigma_{\su};\utheta) = \frac{1}{1 + \exp(-2 \sum_{k \in \kc_u} \theta_{k} \usigma_k) }.$ Then the above mentioned average KL divergence between conditionals takes the form,

\begin{align}
 \Enu D_{KL}(&\nu(.|\usigma_{\su}) || p(.|\usigma_{\su};\utheta)) =\\ & \sum_{\usigma_{\su}} \nu(\usigma_{\su}) \EXp{\sigma_u \sim \nu(.|\usigma_{\su})}\log\left(\frac{\nu(\sigma_u|\usigma_{\su})}{  p(\sigma_u|\usigma_{\su};\utheta)}\right) \nonumber
\end{align}

We can see that up to an unimportant constant, this is precisely the loss function in \eqref{eq:PL_loss} as $M \rightarrow \infty.$

Now using the results developed so far, we can bound the value of this average divergence when the hypothesis matches the true model, i.e. $\utheta = \utheta^*$. In this case, from Corollary \ref{cor:KLclose} and \eqref{eq:condtional_bounds}, we can see that the the average conditional KL divergence between the metastable state and the true distribution upper bounded by  $\frac{4 (1 + e^{2\gamma}) \eta}{\omega_P}$. In practically interesting models, $\gamma$ is usually a constant independent of $n$. In the previous sections, we showed constructions of strongly metastable distributions where $\eta$ scales as inverse of the mixing time.For such distributions associated with slow mixing Markov chains this upper bound decays exponentially fast with the number of spins in the model. 

Remember that the KL divergence is always non-negative, hence this quantity also upper bounds the difference in the PL loss value between $\theta^*$ and the true optimum $\hat{\theta}.$ From this we can conclude that there is a family of strongly metastable distributions for which the the pseudo-likelihood estimator asymptotically approaches the true model parameter in the limit of infinite learning samples drawn from the metastable distribution. This is an extremely interesting observation, as this tells us that the true model can be nearly reconstructed from `` bad '' data using the PL estimator.

In practice, for finite $M$, there will be an $O(1/\sqrt{M})$ statistical error in the estimation of the loss function. Hence, if $M$  is small enough this statistical error will swamp the $O(\eta/\omega_P)$ bias introduced by having bad data.

However, this result by itself does not guarantee that the estimated parameters from \eqref{eq:PL_loss} will closely match that of the true distribution. To give such learning guarantees we have to show that the loss function has sufficiently large curvature near the optimal point.  This would then imply that the small changes in the loss values translate to small changes in the estimated parameters. \footnote{See Figure 2 of Reference \cite{neghaban2012high} for a visual explanation of this point}
\cb
\subsection{Parameter learning guarantees for Ising models}
We can exploit the properties of single-variable conditionals of strongly metastable states to prove guarantees on parameter and structure recovery for models with up to pairwise interactions, i.e. Ising models with magnetic field terms. The true energy function then has the form,

\begin{equation}
    E^*(\usigma) = \sum_{i <j} \theta_{ij}^* \sigma_i \sigma_j + \sum_{i=1}^n \theta^*_i \sigma_i,~~\sigma_i \in \{-1,1\}
\end{equation}

The learning task here is to estimate the parameters $\utheta^*$ from data.  For this task we will choose the parametric family to be the most general quadratic energy function on spin variables, 
    $E(\usigma, \utheta) = \sum_{i <j} \theta_{ij} \sigma_i \sigma_j + \sum_{i=1}^n \theta_i \sigma_i.$  For these models, a sparse prior is imposed on the learning by imposing an $\ell_1$ constraint for each variable in the model \cite{ravikumar2010high}. Using this setup, a PL estimate for all the parameters connected to the variable at position $u$ can be estimated as follows,

\begin{equation} \label{eq:ising_PL}
   \hutheta_u  = \argmin{|| \utheta_u||_1 \leq \gamma~}   \frac{1}{M}\sum_{t = 1}^M ~\log( 1+  \exp(-2\sigma^{(t)}_u (\sum_{j \neq u} \theta_{uj} \sigma^{(t)}_j + \theta_u )) ).
\end{equation}

Here we use the shorthand $\utheta_u$ to denote all the coefficients connected to the variable at $u$, i.e. $\utheta_u = \{\theta_u, \theta_{u1}, \theta_{u2}, \ldots, \theta_{un}\}.$ Notice that compared to \ref{eq:PL_loss_gen} we have split the single optimization problem to $n$ smaller optimizations. This is mainly done to make the ensuing theoretical analysis easier. 

This type of estimator was first proposed by Besag \cite{besag1974spatial} as an efficient method for performing parametric estimation on high-dimensional data. The sample complexity of this method for the inverse Ising problem have been thoroughly investigated much more recently \cite{gaitonde2023unified, lokhov2018optimal, wu2018sparse, ravikumar2010high} in the setting where i.i.d. samples are available from the true model. %Now we show that as a consequence of the results in the previous section, this estimator in \ref{eq:ising_PL} is also close to the true model parameters when the samples come from a strongly metastable state.

Below we state the learning guarantee on the error between the pairwise parameters learned from metastable samples and the true parameters of the energy function.

\begin{theorem}[Learning pair-wise couplings] 
\label{thm:learning_l1} 
For $M = \left \lceil 2^{10} \frac{e^{ 8\gamma } \gamma^4}{ \varepsilon^4} \log(\frac{8 n}{\delta}) \right  \rceil $, samples drawn from an $\eta$ strongly metastable distribution, with probability greater than $1 - \delta$ the following guarantee holds for all $u \in [n]$ ,
\begin{equation}
  \max_{v \in [n], v\neq u}  | \utheta_{uv}^* - \underline{\hat{\theta}}_{uv}|   \leq~~ \varepsilon + 4 e^{2 \gamma}\sqrt{\frac{(1+\gamma) \eta}{\omega_P}}.
\end{equation}
   
\end{theorem}

The proof of this result is provided in the \SupI~ \ref{appendix:proof_thm1}. This theorem can be seen as the metastable version of the learning guarantee for PL proved in \cite{lokhov2018optimal}.  Just like in the setting where we have samples from the true distribution, we see that the sample complexity is exponential in the $\gamma$ parameter.  But unlike these settings, there is an unavoidable bias in the error that does not go to zero as the number of samples increases. We can see from the constructions of strongly metastable states discussed before, this bias decays with the size of the system. 

For Glauber dynamics, remember that $\omega_P = 1/n$. So any metastable state with $\eta = o(\frac{1}{n})$ gives error guarantees with a bias that decays with system size.

An important problem associated with learning graphical models is \emph{structure learning} \cite{drton2017structure}, which refers to the reconstruction of the underlying graph structure of the model. To this end define the edge set $E = \{(u,v) \in [n] \times [n] \ | \ \theta_{u,v} \neq 0\}.$ Then this edge set can be recovered from the learned couplings with high probability using a simple thresholding method.

\ca
\begin{corollary}[Learning Ising models from metastable distributions]Given $M \geq\left \lceil 2^9 \frac{e^{8 \gamma } \gamma^4}{ \varepsilon^2} \log(\frac{8 n^2 }{\delta}) \right  \rceil$, from an $\eta-$ strongly metastable distribution; PL estimator recovers the true parameters of the model with the following guarantee holds for all $u \in [n]$ with probability at least $1-\delta$,
\begin{equation}
    \max_{j \neq u} ||\hat{J}_{uj} - J^*_{uj}|| \leq  \varepsilon ~ + ~8\eta \frac{ \gamma e^{4 \gamma }}{\omega_P} \left(1 + 2 \min\left(2, \tanh^{-2}(\gamma) \right)  \right) 
\end{equation}
\end{corollary}
\cb

    \begin{theorem}[Structure learning]\label{thm:learning_sparsity} Consider an Ising model  with  $|\theta_{uv}^*| > \alpha$ for every $(u,v) \in E.$     Let  $\hutheta$ be estimated using $M= \left \lceil 2^{26} \frac{e^{ 8\gamma } \gamma^4}{ \alpha^4} \log(\frac{8 n}{\delta}) \right \rceil $ samples from a $\eta-$strongly metastable state by solving the optimization problem in \eqref{eq:ising_PL}.  Then if  $\alpha > 16 e^{2 \gamma} \sqrt{\frac{(1 + \gamma)\eta}{\omega_P}}$,  the estimated edge set,
   \begin{equation}
    \hat{E} = \{(u,v) \in [n] \times [n]~~ |~~ \max( |\hat{\theta}_{uv}|, |\hat{\theta}_{vu}|) > \alpha/2  \}
\end{equation} 
will match $E$ with probability greater than $1-\delta.$
    \end{theorem}
The assumed  $\eta$ dependent lower bound on $\alpha$ is necessary to prevent the inherent bias in the learning from swamping very small non-zero couplings in the model. Even in cases where the bias is larger than the smallest coupling, this type of thresholding will correctly find the edges with larger weights.
        
Finally we will show how the magnetic field can be provably recovered for $d-$sparse models.  We achieve this by doing a second round of optimization after structure learning. This is done by estimating the second order terms first as in Theorem \ref{thm:learning_l1}, reconstructing the structure and then optimizing the PL loss for a second round with only the magnetic field terms as the variable. % As far as we know, this results give the first guarantees for recovering magnetic fields even in the $\eta = 0$ setting.

\begin{theorem}[Recovering magnetic fields]
Suppose that we are given the edge set $E$ and estimates for all the pairwise parameters in the model such that $|| \utheta_{\su}^* - \underline{\hat{\theta}}_{\su}||_{\infty} \leq \varepsilon$. Further assume that the underlying graph has max degree of $d$ and that $|\theta^*_u| \leq h_{max}$. Then given  samples from an $\eta-$strongly metastable state we can estimate the magnetic fields for each $u$ as follows;
\begin{align}
\hat{\theta}_u = \argmin{| \theta_u| \leq h_{max}~} ~&\frac{1}{M}\sum_{t = 1}^M ~\log(\\1 + &\exp(-2\sigma^{(t)}_u (\sum_{v: (u,v) \in E}\hutheta_{u,v}, \usigma^{(t)}_v \rangle + \theta_u) ).
\end{align}

For $M= \left \lceil \frac{2^7 h^4_{max} e^{4 h_{max}}}{\varepsilon^4_h} \log(\frac{4}{\delta})\right \rceil$ we can guarantee with probability at least $1-\delta$ that the error in this estimate will be bounded as,
$|\hat{\theta}_u - \theta^*_u |\leq \varepsilon_h + 4 \sqrt{d \varepsilon h_{max}} e^{h_{max}} +4 \sqrt{\frac{ \eta h_{max}}{\omega_P}}e^{h_{max}}.$

\end{theorem}

Once again we see that there is a bias term in the error guarantee that has its source in the metastability of the data distribution.
 The three theorems in this section together imply that every parameter in a $d$-sparse Ising model can be recovered with a small bias with sample complexity scaling as $O(\log(n))$, generalizing the high-dimensional results from prior works to the case of metastable samples.
%As far we know, giving parameter learning guarantees for every parameter in higher order models using $O(\log(n))$  samples is an open question even if the samples are assumed to come from the stationary distribution. 

\section{Numerical Experiments}
\label{sec:numerics}

\subsection{Strongly metastable distributions in the Curie-Weiss model}
\label{sec:cw_model_meta}

In this section, we illustrate some of the abstract results and constructions from the earlier sections to the specific case of the Curie-Weiss (CW) model. This model is defined by the following energy function on $n$ binary spin variables ($\sigma_i \in \{1, -1\}$) parameterized by a coupling constant $J$ and a magnetic field $h$,

\begin{equation}
    E(\usigma) = \frac{J}{n}\sum_{i=1}^n\sum_{j= i+1 }^n \sigma_i \sigma_j - h \sum_i \sigma_i.
\end{equation}

This model is widely studied in statistical physics where it serves as a canonical model  for understanding symmetry-breaking phase transitions. The all-to-all connectivity of this model makes it an ideal candidate for theoretical and numerical explorations \cite{levin2010glauber, kochmanski2013curie}.
We can  construct explicit examples of strongly metastable states with small $\eta$ in the CW model. 
The CW model has permutation invariance i.e. the probability of observing a certain configuration of spins in this model only depends on their total sum. Define the \emph{magnetization} of a spin configuration as $m(\usigma) = \sum_i \sigma_i/n$.
Then the probability of observing a configuration with magnetization $m$ is proportional to $e^{-n\Psi(m)}$, where the \emph{free energy} 
$\Psi(m)$ of the model  
\begin{equation}
\label{eq:fe_cw}
    \Psi(m) = \frac{-J}{2} m^2 + hm -  S(m),
\end{equation}
dictates the essential statistical properties of this system.
Here $S(m) =  \frac1n\log{\binom{n}{\frac{n(1 + m)}{2}}}$ is the entropy term. The minima of this free energy function also correspond to regions in the state space from which Glauber dynamics struggles to escape \cite{griffiths1966relaxation}. %\begin{DELENV}\sout{ We will use these known connections between metastability and the minima of the free energy to construct strongly metastable states around these minima.}
%\end{DELENV}
    
The approach taken here follows a classical paradigm from mean-field statistical mechanics and phase-transition theory: one studies the free energy
as a function of the relevant order parameter and analyzes its local expansion near its minima. For the Curie-Weiss model, the magnetization is the natural
order parameter, and the minima of $\Psi(m)$ correspond to the thermodynamic
phases of the model. A local expansion of $\Psi(m)$ around these minima
therefore describes the local concentration of the Gibbs measure, making it a natural starting point for the
construction of metastable states. This is exactly the perspective underlying
Landau-type descriptions of phase transitions and standard analyses of the
Curie-Weiss free energy landscape
\cite{LandauLifshitz1980, ellis1978statistics}.

Inspired by the statistical physics picture,  we look for candidate metastable states by analyzing  $\Psi(m)$. We motivate this method of construction of metastable states further in \SupI~\ref{app:cw_meta} by matching the conditionals of candidate metastable states with that of the true distribution. There we find that the metastable states that closely match the conditionals of the true distribution are centered around the minima of the free energy. Now, define $m_0$ to be the positive value where the gradient of the free energy vanishes,
\begin{equation}
    -Jm_0 + h - S'(m_0) = 0, ~~m_0 \geq 0.
\end{equation}

We can construct candidate metastable states by using a $K$-order Taylor approximation of the free energy around this minimum point. Notice that the first order terms vanish and the constant term can be ignored as we are working with energies. These candidate metastable states can be defined by a free energy function fixed by the order of the expansion as follows, 

\begin{equation}
    \Phi^{(K)}(m) \equiv \sum_{k = 2}^K  \Psi^{(k)}(m_0) \frac{(m-m_0)^k}{k!}
\end{equation}

Here $\Psi^{(k)}$ is the $k-$th derivative of the free energy. This metastable free energy defines a metastable distribution as follows,
\begin{equation}
    \nu^{(K)}(\usigma) \propto \frac{e^{-n \Phi^{(K)} (m(\usigma))}}{\binom{n}{\frac{n(1 + m(\usigma))}{2}}}.
\end{equation}

Notice that these metastable distributions are also permutation invariant. The violation of the detailed balance condition used in \eqref{eq:strong_def} is difficult to compute for general distributions. However, for permutation invariant distributions this quantity can be efficiently computed by first mapping the problem entirely onto the magnetization space. We give the details of this mapping in the  \SupI~\ref{app:cw_meta}. Now by numerically evaluating the detailed balance violation we conclude that the $\nu^{(4)}$ distribution is $O(\frac{1}{n^2})$ strongly metastable for an $n-$spin Curie-Weiss model for $J>1$ and $0<h<0.05$.  We also find that truncating the free energy exactly around $m_0$ with  a well-chosen width leads to $e^{-O(n)}-$ strongly metastable states, akin to the construction in \emph{Subsection} \
\ref{sec:cheeger_metastable}. Results of these experiments are given in \figurename~\ref{fig:meta_numerics}.

\begin{figure*}
    \centering
\begin{subfigure}[b]{0.31\textwidth}
    \includegraphics[width=\textwidth]{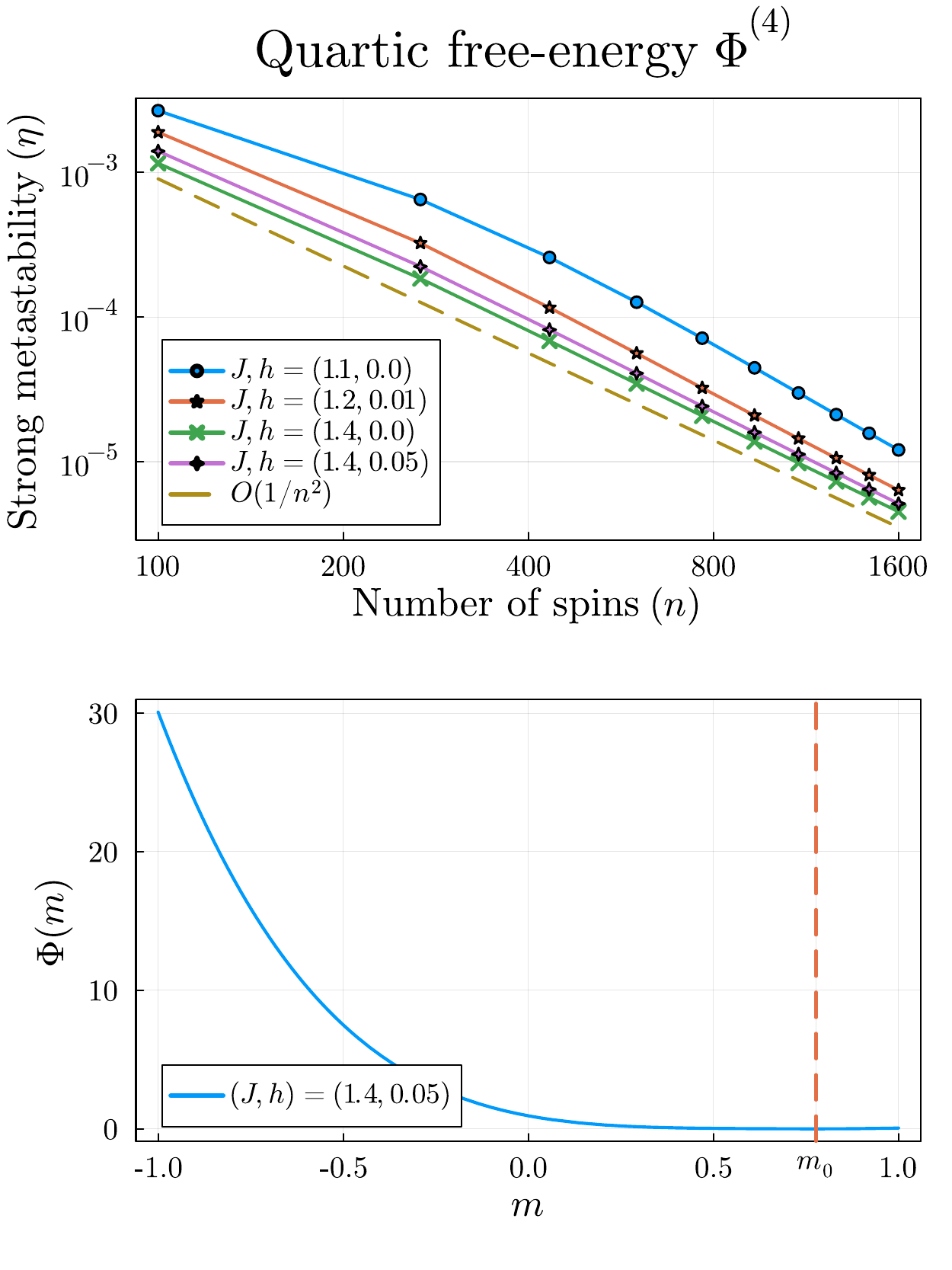}
    \caption{}
    \label{fig:cw_meta_1}
    \end{subfigure}
    \begin{subfigure}[b]{0.31\textwidth}
    \includegraphics[width=\textwidth]{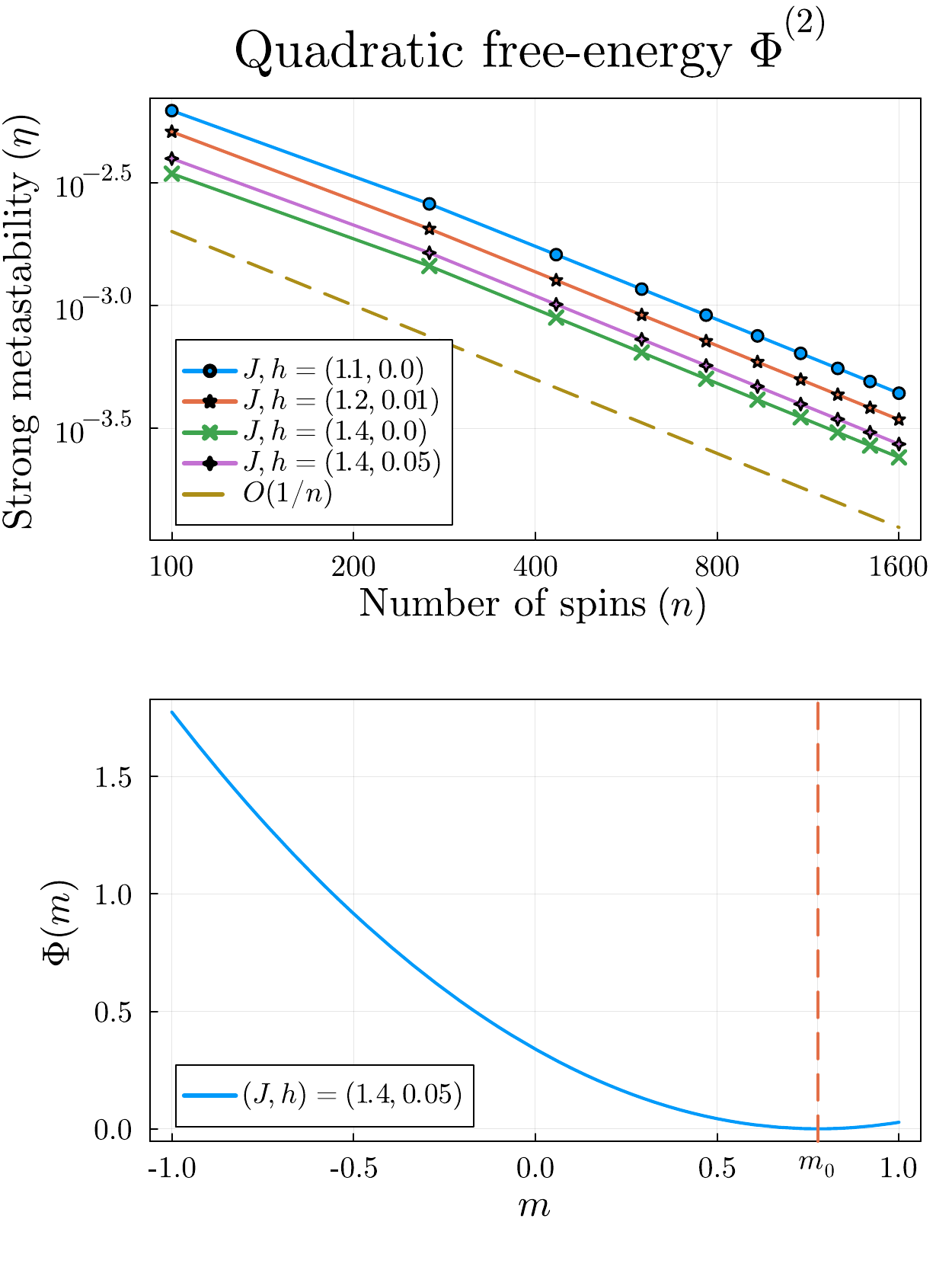}
    \caption{}
    \label{fig:cw_meta_2}
    \end{subfigure}
    \begin{subfigure}[b]{0.31\textwidth}
    \includegraphics[width=\textwidth]{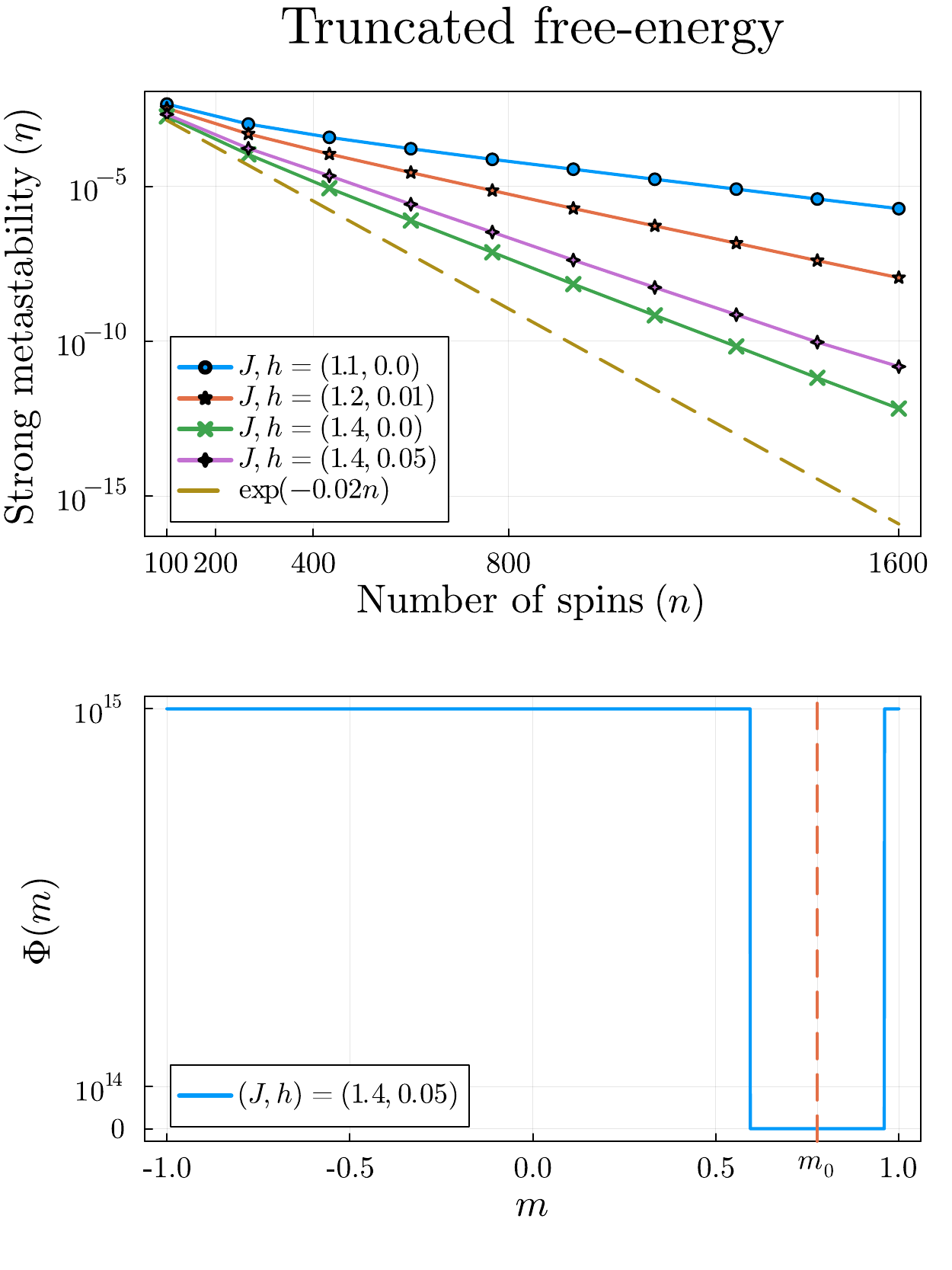}
    \caption{}
    \label{fig:cw_meta_3}
    \end{subfigure}
    \caption{ Strongly metastable states in the CW model. Here we plot the violation in detailed balance condition as defined in \eqref{eq:strong_def} computed by projecting to the magnetization space as in \eqref{eq:mag_reduction}.  (a) Fourth-order approximation to the free energy at $m_0$ (b) Second-order approximation (c) Truncated free energy  as defined in \eqref{eq:truncated_def}}
    \label{fig:meta_numerics}
\end{figure*}

%We can examine this question using a minimal model that can exhibit metastable behavior. We attempt to learn the parameters of the Curie-Weiss (CW) model,

\subsection{Numerical experiments on learning the Curie-Weiss model}

In this section, we numerically confirm that the $J$ and $h$ parameters of the CW model can be learned given samples  coming from a metastable state of a Markov chain.

We sample from this model using the Markov chain method known as Glauber dynamics \cite{glauber1963time} and then attempt to learn the parameters $J$ and $h$ using the pseudo-likelihood method. The CW  model is well suited for this exploration for two reasons. Firstly, in the right parameter range, this model develops two metastable distributions \cite{griffiths1966relaxation, levin2010glauber}. By tuning the magnetic field $h$ we can also change the relative size (in terms of probability) of these two states. Secondly, because of the all-to-all, uniform couplings, we can run Glauber dynamics on very large system sizes. 

% This is because energies of new configurations can be computed extremely fast by exploiting this structure. We wait a sufficient amount of steps between sample collection prevent time-correlation between samples and to accurately model the independence assumption in the learning results.

In this experiment, we choose the system parameters such that the distribution is bimodal but heavily lopsided with most of the probability in states that have negative magnetization. From \figurename~\ref{fig:cw_2}, we see that if i.i.d. samples were produced, the chance of observing a sample with positive magnetization $(\sum_i \sigma_i > 0)$ will essentially be zero. But if we use Glauber dynamics to sample from this model and start the Markov chain from the all-one state $(\sigma_i = 1 \forall i )$ then the dynamics will get stuck in the positive magnetization part of the distribution. This is a well-studied property of CW models that can be rigorously established \cite{levin2010glauber}. We corroborate this by plotting the empirical probabilities produced by an exact sampler and Glauber dynamics 
in \figurename~\ref{fig:cw_2}. This shows that Glauber dynamics is stuck around the minima with positive magnetization. At this point, we assume that the distribution of the samples  generated by Glauber dynamics is close to a metastable distribution with positive magnetization. Notice that by any global metric (TV, KL Divergence, etc.) this distribution is very far from the Gibbs distribution  of the true CW model. This implies that a technique like Maximum Likelihood Estimation (MLE) will not succeed in recovering the right model parameters ($J$ and $h$) as it empirically minimizes the KL-divergence between the samples and the parametric model we are trying to learn. 

However, we find that the original model parameters can be recovered from these samples using the PL estimator. The results of learning using PL are shown in Figure~\ref{fig:cw_1}, which shows that $J$ and $h$ can be learned from samples from the metastable distributions with remarkable accuracy.
Given these samples, the correct values of $J,h$ cannot be predicted by the maximum likelihood method. We demonstrate
 this in Figure \ref{fig:loss_1}  by plotting the negative log-likelihood in the $(J,h)$  space.  We see that the minimum of this function lies far away from the true model.  Figure \ref{fig:loss_2} shows the PL loss function. We see from these plots that the PL loss function has its minimum close to the true model, while MLE predicts the wrong sign for the magnetic field in the model.

%In this work, we will give a rigorous analysis of this phenomenon and also show other numerical examples of models being learned from samples from the metastable distributions of their Markov samplers.

\begin{figure*}
    \centering
\begin{subfigure}[b]{0.45\textwidth}
    \includegraphics[width=\textwidth]{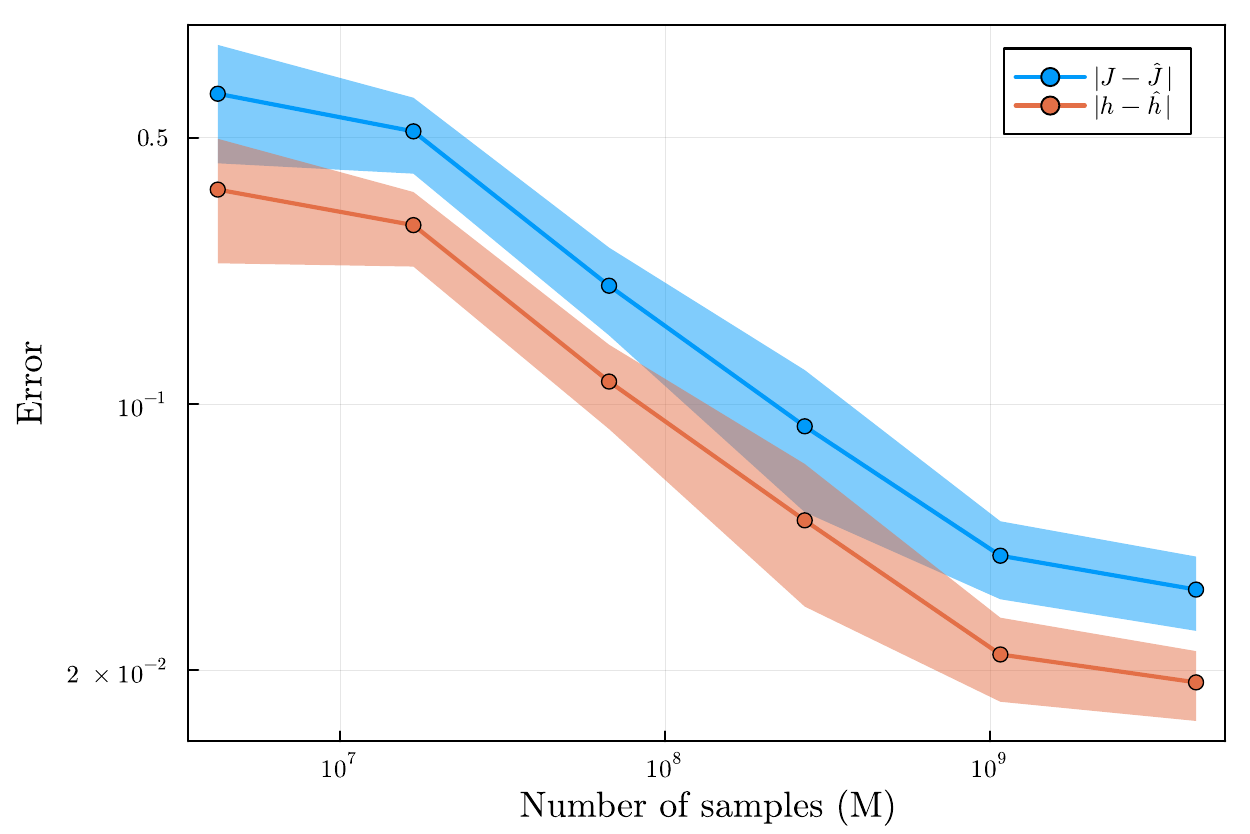}
    \caption{}
    \label{fig:cw_1}
    \end{subfigure}
    \begin{subfigure}[b]{0.45\textwidth}
    \includegraphics[width=\textwidth]{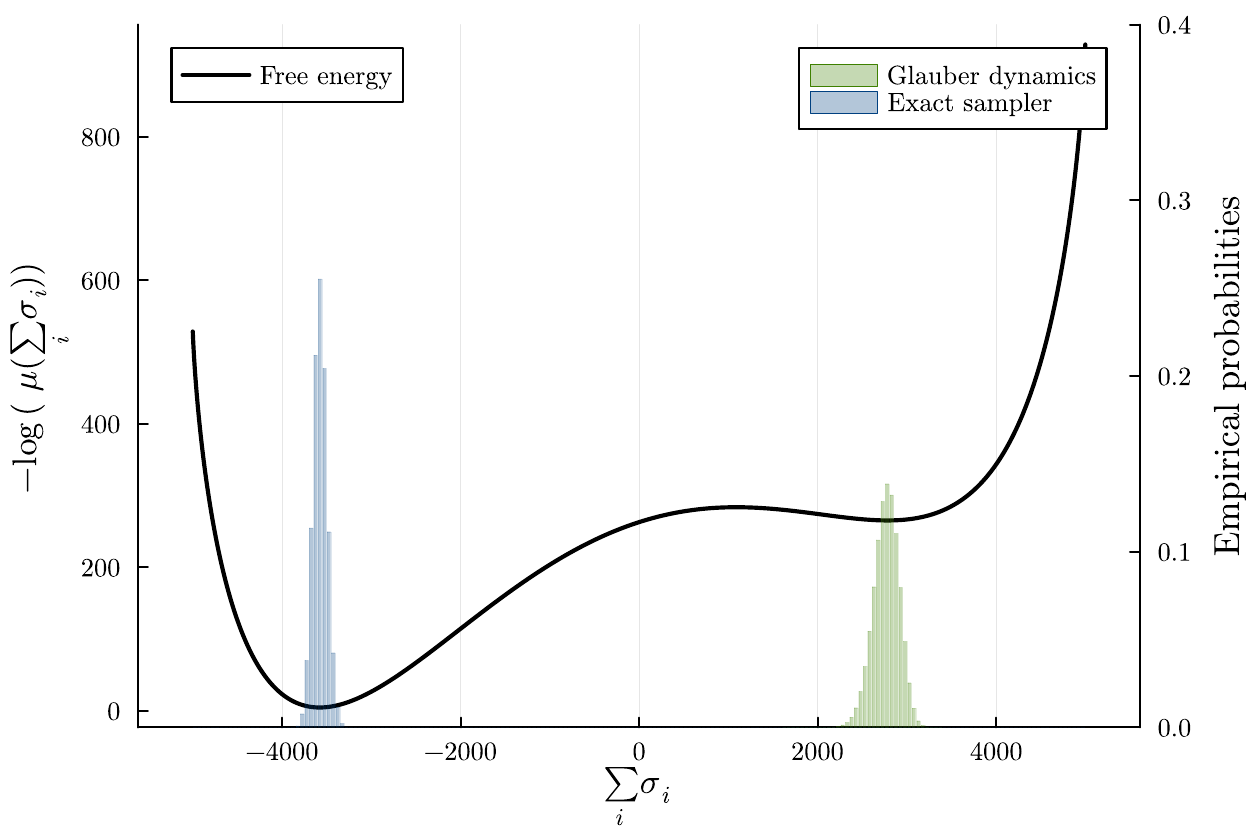}
    \caption{}
    \label{fig:cw_2}
    \end{subfigure}
    \caption{(a) Error in learning the Curie-Weiss model on 5000 spins. Samples here are produced by Glauber dynamics ``stuck'' at the positive minima of the free energy. True parameters here are $J = 1.2, h = 0.04.$ (b) The true distribution is highly biased towards  negative magnetization as seen by free energy curve. There is a metastable distribution with positive magnetization that is highly suppressed in terms of probability. The empirical distributions of samples ($M=4\times10^9$) drawn by an exact sampler and Glauber dynamics is overlaid on top of the free energy. This shows that the Markov chain is effectively stuck around the positive minima.}
\end{figure*}

\begin{figure*}
    \centering
\begin{subfigure}[b]{0.45\textwidth}
    \includegraphics[width=\textwidth]{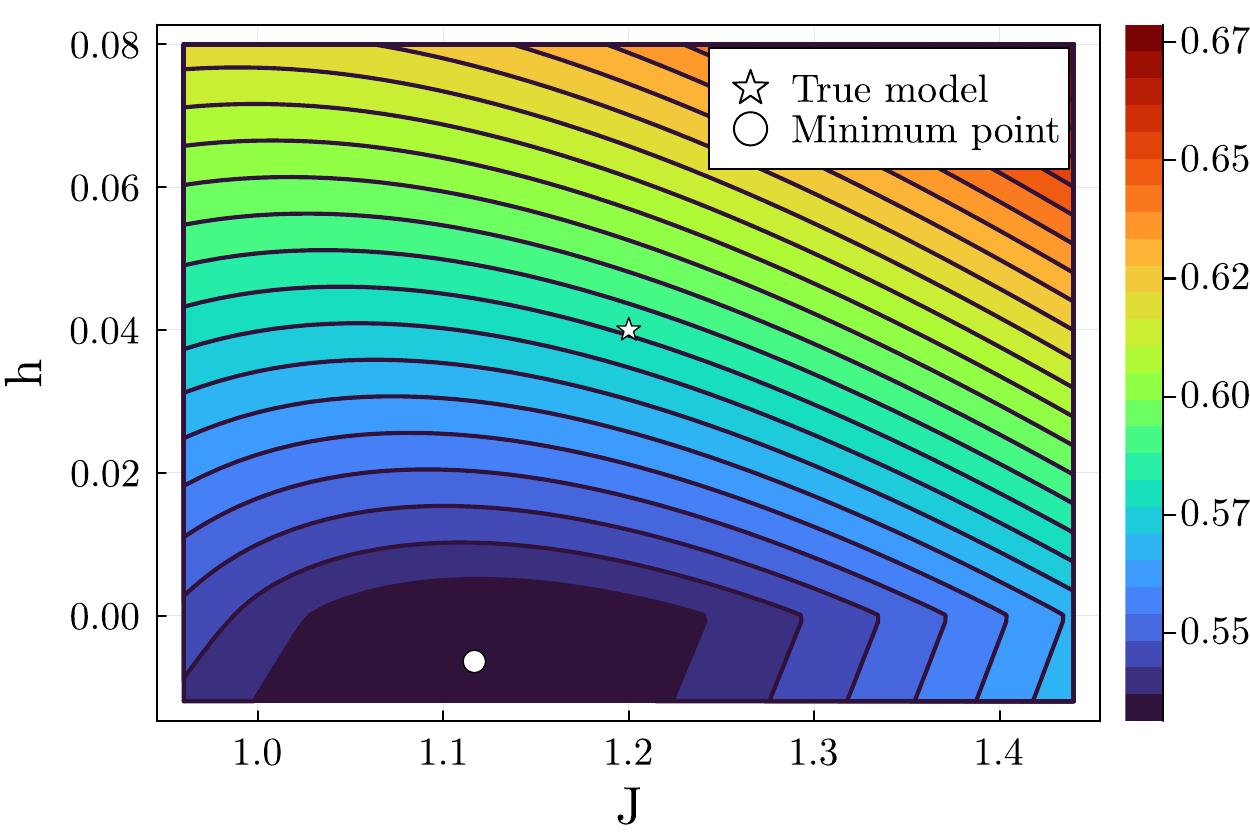}
    \caption{Maximum likelihood loss}
    \label{fig:loss_1}
    \end{subfigure}
    \begin{subfigure}[b]{0.45\textwidth}
    \includegraphics[width=\textwidth]{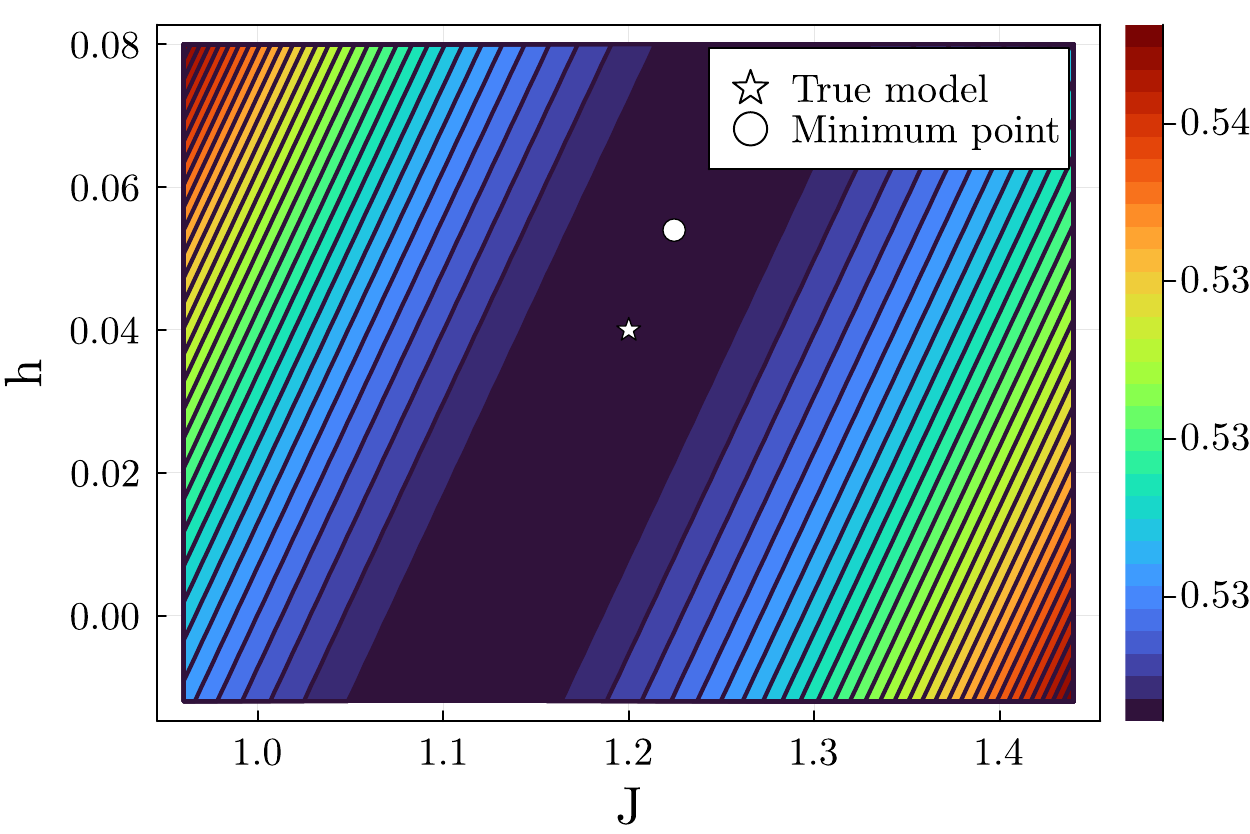}
    \caption{Psuedo-likelihood loss}
    \label{fig:loss_2}
    \end{subfigure}
    %\begin{subfigure}[b]{0.31\textwidth}
    %\includegraphics[width=\textwidth]{PLE_losszoom.pdf}
    %\caption{PL loss (zoomed)}
    %\label{fig:loss_3}
    %\end{subfigure}
    \caption{ Comparison of the loss function landscape for the CW model with true parameters $J = 1.2, h = 0.04.$ These are plotted with $M= 2^{32}$ samples produced by Glauber dynamics ``stuck'' at the positive minima of the free energy. (a) Negative log-likelihood computed from this data clearly has it's minimum far from the true model. The sign of the magnetization is opposite of the true model. This is expected as maximum likelihood tries to match the sufficient statistics of the data to the model. (b) PL loss function has the minima close to the true model and learns the magnetic field with the right sign.}
\end{figure*}

\subsection{Learning spin glass models}

To demonstrate learning from metastable samples in the most general setting, we study this learning problem on a 
model with $q>2$, higher-order interactions and glassy behavior.
To this end, we choose our energy function as the sum of third-order terms over a hypergraph specified by the edge set $\texttt{E}$,
\begin{equation}
\label{eq:three_body_ferro}
    E(\usigma) = -\beta \sum_{(i,j,k) \in \texttt{E}}  ~ W(\sigma_i, \sigma_j, \sigma_k).
\end{equation}

 Here each variable, $\sigma_i$, can take values from $\{1,\dots,q\}$.  The tensor $W$ is a $q\times q \times q$ tensor that determines how each hyperedge contributes to the total energy. 
 
 The tensor $W$ is defined in the following way: Let $a,b,c\in\{1,\dots,q\}$ denote $q$-state variables. We define 
$W(a,b,c)\in\{+1,-1\}$  by
\[
W(a,b,c)=
\begin{cases}
+1, & \text{for } (a,b,c)\in\{(1,r,r),(r,1,r),(r,r,1)\}_{r=1}^q,\\
-1, & \text{otherwise.}
\end{cases}
\]
Equivalently, the only favored local configurations are those with one index equal to
the distinguished state $1$ and the remaining two indices equal.

To understand the behavior of this model, first let us look at the special case of $q=2$. We can switch to the spin representation (\{1,-1\} alphabet) and equivalently write this energy function as, $E(\usigma) = -\beta \sum_{(i,j,k) \in \texttt{E}}  ~\sigma_i \sigma_j\sigma_k$. This shows that the model is not frustrated in the traditional sense, it has a ground state of all ones. However, this form of a three-body ferromagnetic interaction has been studied  \cite{franz2001ferromagnet, newman1999glassy}, as an example of a system showing glassy behavior despite not being frustrated in the traditional sense. When this model is sampled using Glauber dynamics from a random initial state for high enough $\beta$ on a hypergraph of sufficient edge density, this dynamics is seen to get stuck at a higher energy compared to the equilibrium distribution \cite{vuffray2014cavity, newman1999glassy}. As pointed out in \cite{franz2001ferromagnet}, this is caused by the presence of a large number of metastable states that prevents the mixing of the dynamics to the stationary distribution of the Markov chain.

To demonstrate the effectiveness of learning from metastable distributions in the most general setting, we study the parameter and structure learning problem on these models at $q=3$. For the underlying graph, we randomly choose a third order hypergraph with number of edges $|\texttt{E}| = 3 n/2$. 

Similar to the $q=2$ case studied in literature, we observe that for the $q=3$ model, as $\beta$ increases, Glauber dynamics starting from a random initial state is seen to get stuck at a higher average energy than that is given by the exact sampler. In \figureautorefname~\ref{fig:potts_1}, we study the average energy as a function of $\beta$ and observe this onset of metastability. Since the all one state is a ground state, starting Glauber dynamics from this state is seen to match the results of the exact sampler at $n=12$. Since exact sampling at $n=24$ is not viable, we use the fact that the randomly initialized Glauber being stuck at a higher energy than the minimum initialized Glauber as an indication of metastability.
Now to test our learning algorithm, we use the samples generated from the $n=24$ Glauber chain with random initialization as samples. The chain is restarted $16$ times for this experiment, and only every tenth sample is recorded from the chain to ensure that the dynamical effects of the sampler are reduced in the dataset fed to the learning algorithm. In the PLE loss we use a hypothesis energy function of the form,

\begin{equation}
\label{eq:three_body_ferro_hypothesis}
    E(\usigma, \utheta) = - \sum_{(i,j,k) \in \texttt{E}'}  ~ \theta_{i,j,k} \ W(\sigma_i, \sigma_j, \sigma_k),
\end{equation}
such that this hypergraph has twice the number of edges as the true model and is its strict superset, i.e., $|\texttt{E}'| = 3n$ and $\texttt{E} \subset \texttt{E}'$. We do not use the prior information that all the edges have uniform couplings, instead the learning is done with an $\ell_1$ constraint that will not enforce any uniformity. The results of the parameter and structure learning experiments are given in \figureautorefname~\ref{fig:pott_2}, averaged over $5$ different random instantiations of the hypergraph underlying the model. These show that both these tasks succeed given metastable samples. \figureautorefname~\ref{fig:pott_3} shows the histogram of the couplings learned from these metastable states. This shows that with enough samples the learning algorithm is clearly able to classify the hyperedges present in the energy function from the non-edges.

\begin{figure*}
    \centering
\begin{subfigure}[b]{0.31\textwidth}
    \includegraphics[width=\textwidth]{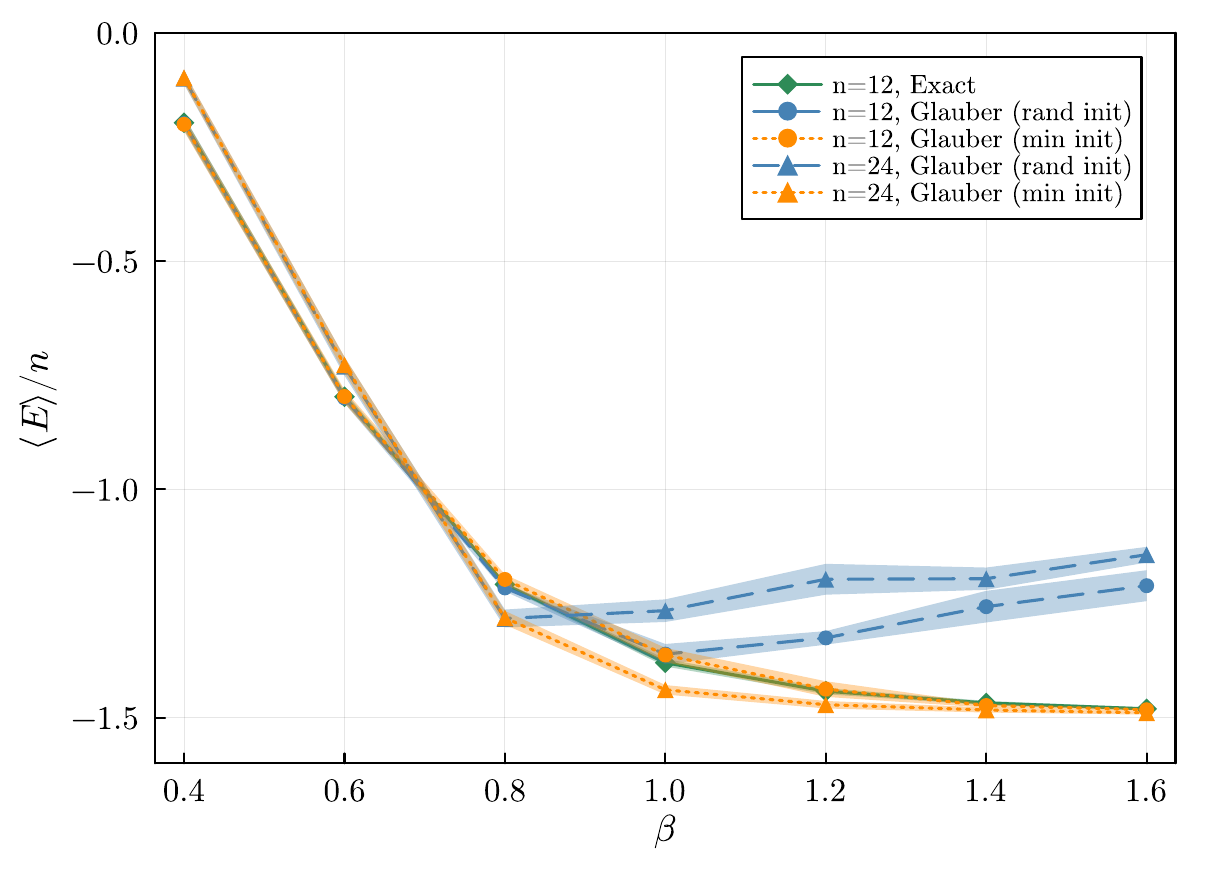}
    \caption{}
    \label{fig:potts_1}
    \end{subfigure}
    \begin{subfigure}[b]{0.31\textwidth}
    \includegraphics[width=\textwidth]{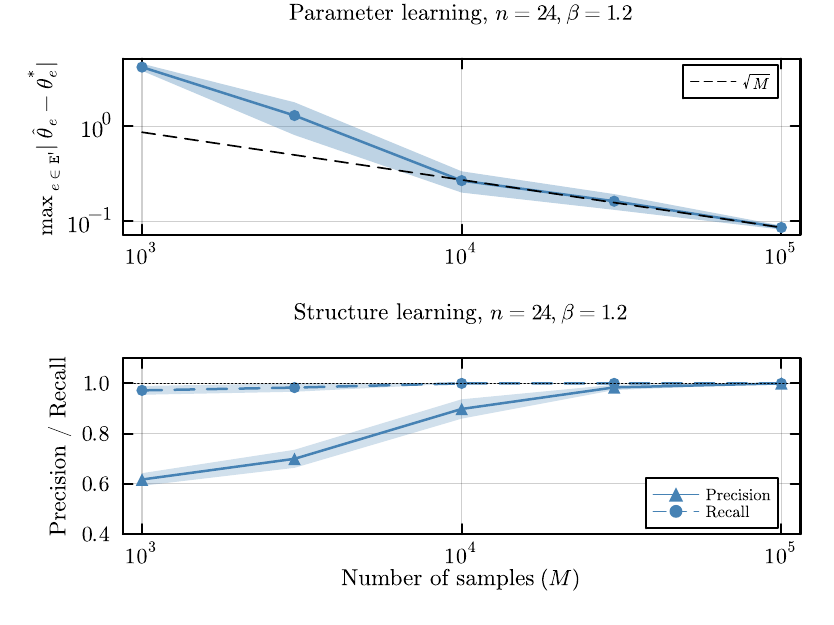}
    \caption{}
    \label{fig:pott_2}
    \end{subfigure}
    \begin{subfigure}[b]{0.31\textwidth}
    \includegraphics[width=\textwidth]{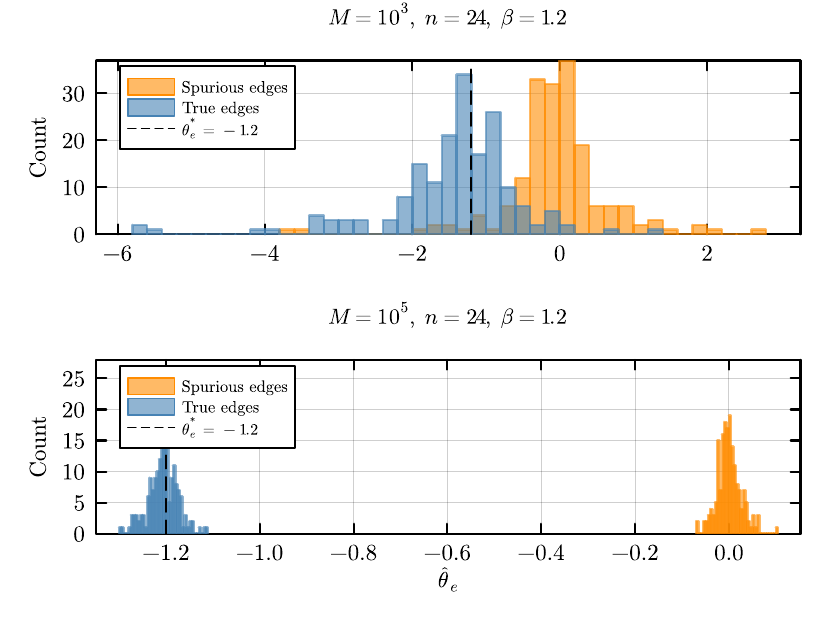}
    \caption{}
    \label{fig:pott_3}
    \end{subfigure}
    \caption{  Learning from metastable samples for a spin glass model with $q =3$. All datapoints are averaged over $5$ different instantiations of the random hypergraph as specified in \eqref{eq:three_body_ferro} (a) Metastability in this model can be observed by comparing the average energy of the samples produced by different samplers \cite{franz2001ferromagnet, vuffray2014cavity}. Glauber dynamics initialized from a random configuration is seen to get stuck at higher energy than the exact sampler and the same Glauber dynamics sampler started from the unfrustrated ground state of the model. For these experiments $M=10^5$. (b) Learning results from metastable samples collected from $n=24, \ \beta = 1.2$ model. For structure learning, via thresholding we choose to declare as edges all terms with a learned coupling greater than $0.12$ in absolute value as the edges. The precision and recall of predicting a hyperedge using this threshold is reported in the bottom panel  (c) Shows the histogram of learned couplings for different values of $M$.  All chains in these experiments are restarted $16$ times during the sampling with only the tenth sample kept till the desired number of samples are collected. Each restart is initialized with a burn-in of $10^4$ steps.}
\end{figure*}

%\ca
%\begin{figure*}
%    \centering
%\begin{subfigure}[b]{0.31\textwidth}
%    \includegraphics[width=\textwidth]{energy_sg.pdf}
%    \caption{}
%    \label{fig:sg_1}
%    \end{subfigure}
%    \begin{subfigure}[b]{0.31\textwidth}
%    \includegraphics[width=\textwidth]{error_sg.pdf}
%    \caption{}
%   \label{fig:sg_2}
%   \end{subfigure}
%   \caption{(a) Error in learning from the model in \eqref{eq:three_body_ferro} with $\beta=1.4$.(a) The average energy of the samples produced by Glauber dynamics is much higher that the true average energy of the model. To compute the average energy of the model for $n>20$ we linearly interpolate from the exact sampling results. $M=10^6$ for these experiments. This implies that the Markov chain is stuck in a metastable distribution (b) The maximum error in the learned energy function parameters. Learning is done from samples produced by the Markov chain. }
%\end{figure*}
%\cb

\section{Discussion}

Metastability of stochastic dynamics is an important idea that lies at the heart of many interesting phenomena studied in statistical physics. In this work, we have shown that metastability has important effects on learning from discrete data.
Our work opens up several interesting directions for exploration. For instance, it is natural to ask if there are large separations in the two measures of metastability defined in this work. In \SupI \ref{appendix:meta_separation}, we show that for a given $P$, the largest separation between these measures is at least as big as the diameter of the underlying graph of transitions. We conjecture that this is not the optimal separation between these two measures of metastability. If this separation is indeed tight then that implies only an $O(n)$ difference between the two measures of measures for Glauber dynamics Markov chain. This is because the underlying graph of transitions for Glauber dynamics is the $n$-bit hypercube, which has a diameter of $n$.  This would then imply non-trivial learning guarantees for all  metastable states.  The connections between these notions and electrical networks used in the proof of Proposition \ref{prop:strong_weak} might be useful for these explorations. 
The closeness of conditionals implied by strong metastability in Theorem \ref{thm:close_condtionals}, is akin to the Dobrushin–Lanford–Ruelle (DLR) condition used in the definition of infinite-volume Gibbs measures \cite{friedli2017statistical}. However, the family of distributions defined by a strongly metastability condition, with an appropriately vanishing $\eta$ with system size, is richer than what is defined by the DLR condition.

There is also a natural question of whether strong metastability, as defined in this paper, is the right model for describing samples collected from an unmixed Markov chain. A more refined question is whether one can guarantee, under suitable conditions, that the Markov chain samples from a strongly metastable state when run for a length of time much smaller that the mixing time. These are relevant questions not only for Gibbs distribution learning, but also for analyzing performance of a broad class of algorithms, and hence has been a subject of interest in recent literature. For instance, Liu et al.~\cite[Problem 1.23]{liu2024locally} posed this question as an open problem for a natural class of strongly metastable states. The recent work of Alman et al.~\cite[Theorem 5]{alman2025dnf} rigorously formalizes a conditional version of this phenomenon for random walks on graphs: the walk can mix locally, after conditioning on a non-negligible-probability event, in regimes where the unconditioned walk need not yet be globally mixed. Extending these explorations to general reversible Markov chains would be directly relevant to the results presented here.

Another point concerns how samples are collected in a typical setting. A discerning reader might argue that an i.i.d.\ sample from a metastable state is never produced while evolving the chain in a strict sense, since the distribution continues to evolve toward equilibrium. There are two responses to this concern. First, when the chain is trapped in a low-conductance set, this evolution is very slow and hence the samples collected from the evolving chain are well modeled by independent samples from a metastable distribution. Second, if the sampler can be reset, as is possible in many settings, one can reset the state, run the chain for the same amount of time, and obtain an independent sample from the same distribution. Moreover, as shown by the results in Section~II.B, strong metastability is fairly robust to errors in time of collection and mixing with samples from well-separated modes. This robustness is also reflected in our experiments, where we collect a large number of samples while restarting the chain only 16 times. This setting is more natural than the one considered in previous Markov-chain learning papers, since we do not require samples to be collected at every time step \cite{dutt2021exponential}.

 The learning guarantees for the Ising case presented here can be generalized to more general models as done in prior works on learning graphical models \cite{Klivans2017, vuffray2019efficient}. Notice that while we have focused on the pseudo-likelihood method, the same arguments presented here would be applicable to any method that exploits conditionals for learning, like Interaction screening \cite{vuffray2016interaction} or Sparsitron \cite{Klivans2017}. In fact, we believe that similar ideas will even carry over to the case of continuous alphabets.

 Improvements on our sample complexity results may be possible focusing on a specific class of models, like Sherrington-Kirkpatrick models \cite{chandrasekaran2025learning} or Low-rank Ising models \cite{koehler2024efficiently}. Following these works that deal with i.i.d samples from the stationary distribution, for these specific families, exponential dependence on the $\ell_1$ width for metastable learning may be improved. However in general, this dependence is unavoidable from the information theoretic lower bound \cite{santhanam2012information}. Extension of our results to continuous variable graphical models is also an interesting direction to consider. However, the state of sample complexity results in this setting is not as mature as that for the discrete setting even when the samples come from the stationary distribution\cite{pabbaraju2023provable}.

 Viewing the results of this work from a higher level, we have established that generalization is possible in the context of learning discrete distributions even if the data is constrained to come from some specific regions of state space. This generalization is made possible by the use of methods that learn the conditionals and by priors imposed during learning (in terms of the low-degree form of the energy assumed during learning and the constraints used in the optimization). This demonstrates an interplay between algorithms and prior which allow learning to succeed even in the presence of the ``wrong'' data. The conditional based methods used here are not limited to graphical models, they can also be used for instance to learn energy based models with a neural net parametrization of the energy \cite{abhijith2020learning}. The exploration of the ideas presented here in the context of more general energy-based models is expected to shed light on the observed generalization properties of such models.
~

\bibliography{main}
\bibliographystyle{acm}
\paragraph*{\bf Competing interests \\} The authors declare no competing interests
 \paragraph*{\bf Funding acknowledgments \\} The authors acknowledge support from the U.S. Department of Energy/Office of Science Advanced Scientific Computing Research Program and from the Laboratory Directed Research and Development program of Los Alamos National Laboratory under Project No.\ 20230338ER. Los Alamos National Laboratory is operated by Triad National Security, LLC, for the National Nuclear Security Administration of U.S.\ Department of Energy (Contract No.\ 89233218CNA000001)
 
\paragraph*{\bf Code availability \\} Codes  used for learning experiments can also be found here \url{https://github.com/abhijithjlanl/metastable-learning}. The documentation in this repository contains detailed information on how to recreate the experiments.

\onecolumngrid
\newpage
\appendix

\section{Proof of Proposition \ref{prop:slow_mix}} \label{app:slow_mix}

For any Markov chain, $P$ and any vector, $v$ the following data-processing inequality holds,

\begin{equation}
    ||v P||_1 = \sum_i |\sum_j P(i|j) v(j) | \leq  \sum_i \sum_j P(i|j) |v(j) | = || v||_1
\end{equation}

Now give an $\eta$ metastable distribution $\nu$, using the above data-processing inequality repeatedly on $v = \nu P - \nu $ gives us,
\begin{equation}
    |\nu P^{k+1}  - \nu P^{k}| \leq \eta
\end{equation}

Now suppose that after $T$ steps starting from $\nu$ the Markov chain gets $\epsilon$ close in TV to the equilibrium distribution $\mu$. From this,
\begin{align}
    |\nu - \mu|_{TV} &\leq  |\nu - \nu P|_{TV}  + |\nu P - \mu|_{TV}, \\
    &\leq |\nu - \nu P|_{TV} + |\nu P^2 - \nu P|_{TV}    \ldots |\nu P^{T} - \nu P^{T-1}|_{TV} + |\nu P^T  - \mu|_{TV},\\
    & \leq \eta T  +  |\nu P^T  - \mu|_{TV}.
\end{align}

This gives $|\nu P^T  - \mu|_{TV} \geq |\nu - \mu|_{TV} - \eta T.$

Now  $|\nu P^T  - \mu|_{TV} \leq \epsilon,$ gives us $T \geq \frac{|\nu - \mu|_{TV} - \epsilon}{\eta}.$

\section{Separations between metastability and strong metastabililty}\label{appendix:meta_separation}

%In the following sections we will show explicit constrictions of strongly metastable distributions that applies to general reversible chains. We will then use this definition of strong metastability to show how the equilibrium distribution can be learned given samples drawn from such a state.

%Notice that either of these definitions of metastability does not directly imply that the metastable distribution is close to the equilibrium distribution. In fact, in practice, it is often seen that metastable distributions are supported in subregions of the state space from which the Markov chain struggles to escape. Hence it is surprising that the parameters of the true model may be reconstructed from such states.

%\subsection{Relation between strong and weak metastability.}

% We will give explicit examples of such states in the following sections and 

The two definitions of metastability raises an interesting question: \emph{For a reversible Markov chain does $\eta-$ metastability imply $\eta-$ strong metastability?}
We can construct a simple counter-example showing that is not true. For an even integer $L$, let $\scc = [L]$ and  $P$ be the Markov chain corresponding to the random walk on a cycle graph on $L$ elements. Take $\nu$ to be the following `tent' distribution,
\begin{equation}
\nu(i)  = \frac{1}{Z} 
\begin{cases}
i  L ,~~ &i \leq L/2, \\
-i  L + L^2 ,~~ &i > L/2, \\
\end{cases}
\end{equation}

The $Z$ here ensures normalization. Now by direct computation, we can verify that this state is $\frac{2L}{Z}$ metastable while being $\frac{L^2}{2Z}$ strongly metastable. This shows that there can be $\Theta(|\scc|)$ factor difference between the two measures metastability.  This construction can be generalized to general reversible Markov chains using by mapping them to an electrical network and analyzing the current flows.

\begin{proposition}\label{prop:strong_weak}
Let $G = (\scc,E)$ be a graph defined on the state space with the edge set $E = \{(i,j) | P(i|j) \neq 0 \}.$  Let $diam(G)$ be the diameter of this graph. Then the chain $P$ has a $\eta$- metastable distribution that is $\Omega(diam(G) \eta)-$strongly metastable.
\end{proposition}
We conjecture that this construction does not give the optimal separation between metastable distributions and strong metastable distributions. The reason is that the diameter is a topological property of the chain and it can be changed drastically by adding a few well-chosen very small but non-zero transitions to the chain.

While the relationship between these two notions of metastability is an interesting question, it is not the main focus of this paper.  We leave further exploration of this for future work.
\subsubsection{Proof of Proposition \ref{prop:strong_weak}}
\paragraph*{Mapping the problem to an electric network:}
To prove this statement we first map the Markov chain to an electric network using the well-known mapping between these two problems\cite{nash1959random, doyle1984random}. The conductance \footnote{Not to be mistaken for the conductance of the Markov chain that is defined in context of probability flows. This conductance is the inverse of the electrical Resistance.} between two state $i,j$ will be given by,
\begin{equation}
    c_{ij} = \mu(i)P(j|i) = \mu(j)P(i|j).
\end{equation}

We also associate a voltage with each node, 
\begin{equation}
    V(i) = \frac{\nu(i)}{\mu(i)}.
\end{equation}

Let $E$ be the edges with non-zero conductance, $E = \{(i,j) \in \scc \times \scc|~ c_{ij} \neq 0\}.$

We also associate a current with every edge according to Ohm's law,
\begin{equation}
    I(i,j) = (V(i) - V(j)) c_{ij} ~~\forall (i,j) \in E
\end{equation}

Note that this current is antisymmetric, that is, $I(i,j) = -I(i,j).$

Now for two nodes $s$ and $t$ in the graph, pick an $\eta-$ metastable distribution such that $\nu(I-P) = \frac{\eta}{2}(\vec{e}_s - \vec{e}_t).$  In the last part of the proof we will show that such a state always exists for some $\eta > 0$. We leave the choice of $s$ and $t$ free for now, but later we will fix it to get optimal bounds.  We can show that this metastable distribution maps to an electrical network that has current $\eta/2$ injected at $s$ and extracted at $t.$ Consider the following relations,

\begin{equation} \label{eq:kirchoff}
    \nu(i) - \sum_{j} \nu(j) P(i|j) =  \sum_{j} P(j|i)\nu(i) - \nu(j) P(i|j) = \sum_{j} c_{ij}(V(i) - V(j)) = \sum_{j: (i,j) \in E} I(i,j)
\end{equation}

The above expression is just the total current flowing into $i$. From the metastability condition on $\nu$ we see that total current flow is conserved at all nodes except $s$ and $t$. At these nodes current of $\eta/2$ is injected/extracted.

Now the strong metastability of the state can be shown to be measured by the total absolute value current flowing in the system,

\begin{equation} \label{eq:tot_I}
    \sum_{i,j} |P(j|i)\nu(i) - \nu(j) P(i|j)| =  \sum_{i,j} c_{ij} |V(i) - V(j)| = \sum_{i,j \in E} |I(i,j)|
\end{equation}
To bound this quantity we consider spheres of increasing size centred around the source and consider the current flow through their boundaries. We define these spheres and their boundaries using the graph distance $d_E(i,j)$ which we take to be the shortest path length that connects node $i$ and $j$ in the undirected graph $G(\scc,E)$.
\begin{align}
 \partial B(s,r) &= \{(i,j) \in E~|~ d_E(s,i) = r,~ d_E(s,j) = r+1\}, ~~~~~{\textit{(Boundary sets)}}  \\   B(s,r) &= \{j \in \scc~|~ d_E(s,j) = r\}.~~~~{\textit{(Sphere sets)}}
\end{align}

Notice that the boundary sets are directed. They do not contain both $(i,j)$ and $(j,i)$ edges.

$s$ and $t$ nodes are unspecified as of now. We take these to be the maximally separated nodes in the graph $G(\scc, E)$, that is $d_E(s,t) = diam(G)$
Using this observation in \eqref{eq:tot_I},

\begin{equation} \label{eq:flow_bound}
    \sum_{i,j} |P(j|i)\nu(i) - \nu(j) P(i|j)| \geq  \sum_{r = 0}^{diam(G)-1}\sum_{(i,j) \in \partial B(s,r)} |I(i,j)| \geq \sum_{r = 0}^{diam(G)-1}\left|\sum_{(i,j) \in \partial B(s,r)} I(i,j)\right|
\end{equation}

Now intuitively we expect that the total current flow through every boundary set must be the same due to current conservation law in \eqref{eq:kirchoff}.   We show this rigorously below. From the conservation law, for any $0<r<diam(G)$ (avoiding the source and sink)

\begin{align}
   0 = \sum_{i \in B(s,r)} \sum_{j:(i,j) \in E} I(i,j) 
\end{align}

In the above equation the edges connecting nodes inside $B(s,r)$ will not contribute as both $I(i,j)$ and $I(j,i)$  will exist in the sum which cancel due to antisymmetry. The remaining edges either lie in $\partial B(s,r)$ or $\partial B(s,r-1)$. Hence,

\begin{align} \label{eq:gausslaw}
   0 = \sum_{i,j \in \partial B(s,r)} I(i,j)  + \sum_{(j,i) \in \partial B(s,r-1)   }I(i,j)
\end{align}

Now the total flow out of the source $\sum_{(i,j) \in \partial B(s,0)} I(i,j) = \sum_{(s,j) \in E} I(s,j) = \eta/2$. Using this relation iteratively in \eqref{eq:gausslaw}  we get that the current flow through every boundary is the same,

\begin{equation}
   \left| \sum_{(i,j) \in \partial B(s,0)} I(i,j) \right| =   \left| \sum_{(i,j) \in \partial B(s,1)} I(i,j) \right| = \ldots  =\left| \sum_{(i,j) \in \partial B(s,diam(G)-1)} I(i,j) \right| = \eta/2
\end{equation}

Using these relation in \eqref{eq:flow_bound}, we get,

\begin{equation}
    \sum_{i,j} |P(j|i)\nu(i) - \nu(j) P(i|j)| \geq  \sum_{r = 0}^{diam(G)-1}\sum_{(i,j) \in \partial B(s,r)} |I(i,j)| \geq diam(G) \eta/2
\end{equation}

\paragraph*{Existence of s-t  $\eta-$ metastable distribution}

The question remains whether there exists a distribution such that,  $\nu(I-P) = \frac{\eta}{2}(\vec{e}_s - \vec{e}_t).$ We will now show that such a distribution must always exist with some non-zero $\eta.$ \newline
First, observe that if $P$ is an irreducible chain then the equilibrium distribution $\mu$ is the unique zero left-eigenvector of $I-P$. Moreover, in this case the  by Perron–Frobenius theorem (refer Chapter 8, \cite{horn2012matrix}) $\mu$ is a strictly positive vector.

For some $\eta' > 0$, consider the system of equations    $\nu'(I-P) = \frac{\eta'}{2}(\vec{e}_s - \vec{e}_t).$ The matrix $I-P$ is not full rank, but its null space is only one dimensional as irreducible chain has a unique stationary state. The only zero right eigenvector of $I-P$ is the vector of all ones. $\frac{\eta'}{2}(\vec{e}_s - \vec{e}_t)$, is orthogonal to this vector. Hence this system will have a solution.

For any $\nu'$ satisfying this system and for any real $\alpha$, the vector $\nu' + \alpha*\mu$ any is also a solution to the above linear system. Now we can take $\alpha$ to be a large enough value such that we hit a positive solution. Thus for any $\eta'$, there always exists a positive vector $\nu' > 0$ such that $\nu'(I-P) = \frac{\eta'}{2}(\vec{e}_s - \vec{e}_t).$  Now given such an $\eta'$ and $\nu'$ we can define,
\begin{align}
    \nu(i) &= \nu'(i)/ \sum_i \nu'(i) \\
    \eta &= \eta'/ \sum_i \nu'(i)
\end{align}
We see that $\nu$ is a probability distribution which will be an $\eta-$metastable distribution, satisfying $\nu(I-P) = \frac{\eta}{2}(\vec{e}_s - \vec{e}_t).$  Since $\eta' > 0$ and $\nu'$ is positive, we have $\eta > 0.$

\section{Metastable distributions of Curie-Weiss model  }\label{app:cw_meta}

Consider the  order parameter magnetization per spin, $m(\usigma) = \frac{\sum_i \sigma_i}{n}$.

The Gibbs distribution associated with the Curie-Weiss model can be expressed as,
\begin{equation}
    \mu(\usigma) = \sum_{m\in\{-1, -1+2/n, \ldots, 1\}} \mu(\usigma|m) \frac{\exp(-n \Psi(m))}{Z_\Psi} 
\end{equation}

Where $\mu(\usigma|m)$ is the probability of observing the configuration $\usigma$ given a fixed magnetization. For CW-type models with permutation symmetry, the probability of observing a certain configuration only depends on the magnetization and hence this is purely an entropic quantity, $\mu(\usigma|m) = \delta_{m(\usigma),m}\frac{1}{\binom{n}{\frac{n(1 + m)}{2}}}.$ The free energy function, $\Psi(m)$ fixes the distribution over the order parameters values. By direct inspection of the CW model Hamiltonian we can see that,
\begin{equation}
    \Psi(m) = -\frac{J}{2} m^2 + hm -  S(m).
\end{equation}

Here $S(m) =  \frac1n\log{\binom{n}{\frac{n(1 + m)}{2}}}$ is the entropy term. This free energy is a well-studied quantity in statistical physics and defines the canonical model for the study of phase transitions \cite{kochmanski2013curie}. It is also well-established that the Glauber dynamics for this model gets stuck at metastable distributions which are defined by minima of this free energy \cite{levin2010glauber}. Now we will show that the minima of the free energy also corresponds to distributions that have the same single-variable conditionals as $\mu.$ 

To this end, let us explore the following question; \emph{Which other permutation invariant distributions will have the same single-variable conditional as $\mu$? \label{qu:1} }

The relevance of this question comes from the fact that learning algorithms like PL learn the single-variable conditionals, the existence of such a distribution $\nu$ that matches $\mu$ approximately in single-variable conditionals implies that $\mu$ can be learned even if the parameters come from $\nu.$

From the permutation invariance, we define $\nu$ using the following general form,
\begin{equation}
    \nu(\usigma) = \sum_m \mu(\usigma|m) \frac{\exp(-n \Phi(m))}{Z_\Phi}
\end{equation}

Now use the following relations, $\mu(\usigma|m) = \frac{\mu(\usigma,m)}{\mu(m)} = \mu(\usigma) \delta_{m(\usigma),m} \exp(n \Psi(m)) Z_\Psi,$

\begin{equation}
    \nu(\usigma) =  \mu(\usigma)\frac{Z_\Psi}{Z_\Phi}  e^{-n(\Phi(m) - \Psi(m))} ~~\delta_{m(\usigma),m}.
\end{equation}

Now remember that $\usigma_{\neg i}$ is just $\usigma$ with the $i-$th spin flipped. The single-variable conditionals of any distribution is simply $\nu(\sigma_i| \usigma_{\setminus i}) = \frac{1}{1 + \frac{\nu(\usigma_{\neg i})}{\nu(\usigma)}}.$ Hence these conditionals are fixed completely by the ratio of probabilities in the denominator. Let us now compute this ratio for $\nu$. Without loss of generality we assume that $\sigma_i = 1$, 
\begin{equation}
    \frac{\nu(\usigma)}{\nu(\usigma_{\neg i})} =    \frac{\mu(\usigma)}{\mu(\usigma_{\neg i})}e^{-n(\Phi(m) - \Phi(m - 2/n) - (\Psi(m) - \Psi(m -2/n))} ~~\delta_{m(\usigma),m}.
\end{equation}

Now for the conditionals of these two distributions to approximately match the term in the exponent has to be close to zero. Enforcing this condition will give us the form of $\Phi$ which is the only unknown in the RHS. 

To make the calculations easier, define the following finite difference operator of a function,
\begin{equation}
    \Delta_n f(m) = f(m) - f(m-2/n)
\end{equation}

Then,
\begin{align}
&\Delta_n m  = \frac{2}{n},\\
&\Delta_n m^2  = \frac{4m}{n} - \frac{4}{n^2}, \\
&\Delta_n m^k  =  O(\frac{m^{k-1}}{n}).
\end{align}

Now we make the assume the following quadratic ansatz for $\Phi(m)$, for some $a > 0$ and $m_0 \in [-1,1]$,

\[ \label{eq:meta_gauss} \Phi(m) = \frac{-a (m-m_0)^2}{2} \]

We have to determine $a$ and $m_0$ such that the single-variable conditionals match as much as possible.

\begin{equation}\label{eq:phi_diff}
 \Delta_n \Phi(m) = \frac{-2a(m-m_0)}{n} + \frac{2a}{n^2}.
\end{equation}

Now for $\Psi,$  we expand the free energy up to second order around the (as of yet unknown) point $m_0$,
\begin{equation}
    \Psi(m) = -\frac{J}{2} m^2 + h m - S(m_0) - (m-m_0)S'(m_0) - \frac{(m-m_0)^2}{2} S''(m_0) + O( (m-m_0)^3)
\end{equation}

This gives, 
\begin{equation}\label{eq:psi_diff} 
    \Delta_n\Psi(m) \approx -J( \frac{2m}{n} - \frac{2}{n^2})+ h \frac{2}{n}  - S'(m_0)\frac{2}{n} - \frac{2S''(m_0)(m-m_0)}{n} + \frac{2S''(m_0)}{n^2}.
\end{equation}

Now equating terms of order $\frac{1}{n^2}$ from \eqref{eq:phi_diff} and \eqref{eq:psi_diff} gives us the following relation for finding $a$,
\begin{align} \label{eq:fe_second}
   a  &= -J -  S''(m_0)
\end{align}

Substituting this back in the previous expression and matching terms of $1/n$ gives us the following expression for $m_0$
\begin{equation}\label{eq:fe_first}
    -Jm_0 + h - S'(m_0) = 0.
\end{equation}

Now this above expression is precisely the first-order stationarity condition for the free energy of CW model. And from $a>0$ we find that $\Phi''(m_0) = J - S''(m_0) > 0$ which tells us that $m_0$ lies at the minimum of the free energy.

This calculation establishes the following important observation: \emph{Minima of the CW free energy defines distributions whose single-variable conditionals match those of the CW distribution}

Notice also that $\Phi(m)$ defined as in \eqref{eq:meta_gauss} is just the second order Taylor expansion of the free energy $\Psi$ at $m_0$.
Now starting from this observation we can define a class of potentially interesting metastable states by expanding further around the positive minima of $\Psi$.

\begin{equation}
    \Phi^{(K)}(m) = \sum_{k = 2}^K  \Psi^{(k)}(m_0) \frac{(m-m_0)^k}{k!}
\end{equation}

\begin{equation}\label{eq:poly_metastable}
    \nu^{(K)}(\usigma) = \frac{e^{-n \Phi^{(K)} (m(\usigma))}}{\binom{n}{\frac{n(1 + m)}{2}}}.
\end{equation}

\subsection{Numerical evaluation of strong metastability violation in Curie-Weiss models }

For distributions which are permuation invariant, i.e. where $\nu(\usigma)$ is only a function of $m(\usigma)$, the quantity in \eqref{eq:strong_def} can be computed efficiently my mapping the problem on to the magnetization space first. 

First let us fix the Markov chain as Glauber dynamics defined by the Curie-Weiss model in \eqref{eq:fe_cw}. This gives us,
\begin{equation}
    Pr(\usigma_{\neg i} | \usigma) = \frac{1}{n} \mu(-\sigma_i|\usigma_{\setminus i})
\end{equation}

Now using this the violation in the detailed balance condition as defined in \eqref{eq:strong_def} can be rewritten as,

\begin{align}
\frac12\sum_{\usigma, \usigma' } |P(\usigma'|\usigma)\nu(\usigma) - P(\usigma|\usigma')\nu(\usigma')| &= \frac{1}{2n} \sum_{i = 1}^n \sum_{\usigma} |\mu(-\sigma_i|\usigma_{\setminus i}) \nu(\usigma) - \mu(\sigma_i|\usigma_{\setminus i}) \nu(\usigma_{\neg i})  | \\
&= \frac{1}{2n} \sum_{i = 1}^n \sum_{\usigma} |(1 - \mu(\sigma_i|\usigma_{\setminus i}) )\nu(\usigma) - \mu(\sigma_i|\usigma_{\setminus i}) \nu(\usigma_{\neg i})  | \\
&= \frac{1}{2n} \sum_{i = 1}^n \sum_{\usigma} |\nu(\usigma) - \mu(\sigma_i|\usigma_{\setminus i})( \nu(\usigma_{\neg i}) + \nu(\usigma) )  | \\
&=\frac{1}{2n} \sum_{i = 1}^n \sum_{\usigma} \nu(\usigma_{\setminus i}) | \nu(\sigma_i|\usigma_{\setminus i}) - \mu(\sigma_i|\usigma_{\setminus i})  |
\end{align}
Where $\nu(\usigma_{\setminus i}) = \nu(\usigma_{\neg i}) + \nu(\usigma)$ is the marginal distribution of all spins except $\sigma_i$

Now since  $| \nu(\sigma_i|\usigma_{\setminus i}) - \mu(\sigma_i|\usigma_{\setminus i})  | = | \nu(-\sigma_i|\usigma_{\setminus i}) - \mu(-\sigma_i|\usigma_{\setminus i})  |$. We can deduce that $\sum_{\usigma} \nu(\usigma_{\setminus i}) | \nu(\sigma_i|\usigma_{\setminus i}) - \mu(\sigma_i|\usigma_{\setminus i})  | = \sum_{\usigma} \nu(\usigma) | \nu(\sigma_i|\usigma_{\setminus i}) - \mu(\sigma_i|\usigma_{\setminus i})  |$.  This gives us the final expression,

\begin{equation}\label{eq:cond_diff}
     \frac12\sum_{\usigma, \usigma' } |P(\usigma'|\usigma)\nu(\usigma) - P(\usigma|\usigma')\nu(\usigma')|= \frac{1}{2n} \sum_{i = 1}^n \sum_{\usigma} \nu(\usigma) | \nu(\sigma_i|\usigma_{\setminus i}) - \mu(\sigma_i|\usigma_{\setminus i})  |
\end{equation}

For the Curie-Weiss model, we can easily see that the single-variable conditionals have the following expression that depends only on $\sigma_i$ and the magnetization of the state,

\begin{equation}\label{eq:cw_conditionals}
   \mu(\sigma_i|\usigma_{\setminus i}) = \frac12 (1 - \tanh(\sigma_i(J m(\usigma) - h) + \frac{J}{n}) \equiv  f(\sigma_i, m(\usigma)).
\end{equation}
Now we assume a general permutation invariant form for $\nu$,

\begin{equation}
    \nu(\usigma) = \frac{e^{-n\left(\Phi(m(\usigma)) + S(m(\usigma))\right)} }{Z_{\Phi}} \equiv   \frac{e^{-nH(m(\usigma))} }{Z_{\Phi}}
\end{equation}

From this we can get the conditional as,
\begin{equation}
    \nu(\sigma_i|\usigma_{\setminus i}) = \frac{1}{ 1 + \frac{\nu(\usigma_{\neg i})}{\nu(\usigma)}} = \frac{1}{1 + \exp(n(H(m(\usigma) - H(m(\usigma_{\neg i})))} \equiv g(\sigma_i, m(\usigma))
\end{equation}

We can use these expressions for conditionals in \eqref{eq:cond_diff}  and write this completely in the magnetization space. The number of states with a fixed per spin magnetization $m$  is simply given by $\exp(n S(m))$.  Now for a given $i$, a $\frac{1+m}{2}$ fraction of these will have $\sigma_i = 1.$ Using these we can write,
\begin{align}
    &\frac{1}{2n} \sum_{i = 1} \sum_{\usigma} \nu(\usigma) | \nu(\sigma_i|\usigma_{\setminus i}) - \mu(\sigma_i|\usigma_{\setminus i})  |\\
    &=  \frac{1}{2Z_{\Phi}n} \sum_{i = 1}^n \sum_{m \in \{-1,-1+\frac{2}{n}, \ldots, 1\}} e^{-n\Phi(m)} \left(\frac{1+m}{2}| g(1,m) - f(1,m) | + \frac{1-m}{2}|g(-1,m) -f(-1,m)|  \right),\\
    &=  \frac{1}{2Z_{\Phi}}\sum_{m \in \{-1,-1+\frac{2}{n}, \ldots, 1\}} e^{-n\Phi(m)}| g(1,m) - f(1,m) | \label{eq:mag_reduction}
\end{align}

Where in the last line we have use the normalization condition on the conditionals, $g(1,m) + g(-1,m) = f(1,m) +f(-1,m) = 1.$
Using this reformulation we can compute the strong metastability of the state in essentially $O(n)$ time. The normalization constant can also be computed efficiently using the formula, $Z_{\phi} = \sum_\sigma \exp(-n H(m(\usigma))) = \sum_{m} \exp(-n \Phi(m))$.%  For computing $S(m)$ we use the Stirling approximation to directly approximate log factorials to avoid numerical overflow errors.

The details of numerical evaluation of three different families metastable distributions for the CW model is shown in \figurename \ref{fig:meta_numerics}. First we look at the quartic $(K=4)$ and quadratic $(K=2)$ approximations to the CW free energy as defined in \eqref{eq:poly_metastable}. We perform the expansion around the positive minimum of the CW free energy ($m_0 > 0$). As this corresponds to the mode of the distribution with suppressed probability when $h > 0$ (refer \figurename \ref{fig:cw_2}). For the third case, we truncate the free energy around $m_0$ with a certain width,
\begin{equation} \label{eq:truncated_def}
    \Phi(m) = \begin{cases}
       &\Psi(m),~~\text{if}~~ |m-m_0| < \frac{m_0}{4\sqrt{a}}\\
       & \infty,~~\text{otherwise}
    \end{cases}
\end{equation}

This truncation corresponds to the type of metastable states defined in Section \ref{sec:cheeger_metastable}. We see numerically that these metastable states have an $\eta$ value that is exponentially small in the system size.

\ca

\subsection{Proof of Proposition \ref{prop:sm2}}
To prove this statement we first map the Markov chain to an electric network using the well-known mapping between these two problems\cite{nash1959random, doyle1984random}. The conductance between two state $i,j$ will be given by,
\begin{equation}
    c_{ij} = \mu(i)P(j|i) = \mu(j)P(i|j).
\end{equation}

We also associate a voltage with each node, 
\begin{equation}
    V(i) = \frac{\nu(i)}{\mu(i)}.
\end{equation}

We know that $\nu$ satisfies the relation $\nu(I-P) = \frac{\eta}{2}(\vec{e}_s - \vec{e}_t).$ We can show that this maps to an electrical network that has $\eta/2$ injected at $s$ and extracted at $t.$

\begin{equation}
    \nu(i) - \sum_{j} \nu(j) P(i|j) =  \sum_{j} P(j|i)\nu(i) - \nu(j) P(i|j) = \sum_{j} c_{ij}(V(i) - V(j)).
\end{equation}

The above expression is just the total current flowing into $i$. From the condition on $\nu$ we see that total current flow is conserved at all nodes except $s$ and $t$. At these nodes current of $\eta/2$ is injected/extracted.

Let $c^*$ be the smallest non-zero conductance in this network. Now using this mapping we see that,
\begin{align}
\sum_{i,j} |P(j|i)\nu(i) - \nu(j) P(i|j)| &=  \sum_{i,j} c_{ij} |V(i) - V(j)|,\\
&\geq \sqrt{\sum_{i,j} c^2_{ij} (V(i) - V(j))^2},\\
&\geq \sqrt{ c^*\sum_{i,j} c_{ij} (V(i) - V(j))^2},\\
&=  \sqrt{ 2c^*\sum_{i,j} c_{ij} (V^2(i) - V(i)V(j))},\\
&=  \sqrt{ 2c^*\sum_{i}V_i \sum_{j} c_{ij} (V(i) - V(j))},\\
&= \sqrt{ \eta c^* (V(s) - V(t)) }.
\end{align}

Now this potential difference $V(s) - V(t)$ is related to the injected current via the effective resistance of the network $V_s - V_t = \frac{\eta}{2}R^{s,t}_{eff}$

Now the effective resistance of a network and escape probabilities are related as follows (see Proposition 9.5 \cite{levin2017markov}),
\begin{equation}
\frac{1}{p^{s,t}_{escape} \sum_{j} c_{sj}} =  \frac{1}{p^{s,t}_{escape} \mu(s) } = R^{s,t}_{eff}
\end{equation}

Now choose $s$ such that $c_{si} = c^*$ for some $i$. 
\cb

\section{Robustness of strong metastability: proofs}\label{appendix:robustness_proofs}

\subsection{Proof of Proposition \ref{prop:robust_mix}}
\begin{proof}
\begin{align}
    \frac12\sum_{i,j \in \scc } | P(i|j)\nu(j) - P(j|i)\nu(i)| &=  \frac12\sum_{i,j \in \scc } |\sum_k p_k (P(i|j)\nu_k(j) - P(j|i)\nu_k(i))|, \\
    &\leq \frac12\sum_{i,j \in \scc } \sum_k p_k | (P(i|j)\nu_k(j) - P(j|i)\nu_k(i))| \leq \sum_k p_k \eta_k.
\end{align}
\end{proof}

\subsection{Proof of Proposition \ref{prop:robust_dynamics}}
\begin{proof}
Fix an integer $t \ge 1$, and define $\nu' = \nu P^t$. Also define $\Delta(i) \coloneqq \nu'(i) - \nu(i)$. We bound the detailed-balance violation of $\nu'$ as
\begin{align}
 \frac12\sum_{i,j \in \scc } | P(i|j)\nu'(j) - P(j|i)\nu'(i)|  &\leq \frac12\sum_{i,j \in \scc } | P(i|j)\nu(j) - P(j|i)\nu(i)|  + | P(i|j) \Delta(j) - P(j|i)\Delta(i)|  \\
 &\leq  \eta + \frac12 \sum_j \left(\sum_{i} P(i|j)\right) \ |\Delta(j)| + \frac12 \sum_i \left(\sum_{j} P(j|i)\right) \ |\Delta(i)| \\
  &\leq \eta + \|\Delta \|_1 = \eta + 2 | \nu - \nu P^t|_{TV}.
\end{align}

Now via triangle inequality,
\begin{equation}
|\nu - \nu P^t|_{TV} \le |\nu - \nu P|_{TV} + |\nu P - \nu P^2|_{TV} + \ldots +  |\nu P^{t-1} - \nu P^t|_{TV}  .
\end{equation}
Since strong metastability implies metastability, $\nu$ is $\eta$-metastable. By the data processing inequality, $\nu P^t$ is also $\eta$-metastable, as are the intermediate states. 
\begin{equation}
|\nu - \nu P^t|_{TV} \le t \eta,
\end{equation}
and substituting above yields
\begin{equation}
\frac12\sum_{i,j \in \scc } | P(i|j)\nu'(j) - P(j|i)\nu'(i)| \le (1 + 2t)\eta.
\end{equation}
\end{proof}

\section{Proof of Theorem \ref{thm:close_condtionals}} \label{appendix:proof_thm1}
\begin{proof}
The $\eta-$strong metastability condition gives us,
\begin{equation}
    \frac12 \sum_{\usigma, \usigma' \in \scc} \left| P(\usigma'|\usigma) \nu(\usigma) - P(\usigma|\usigma') \nu(\usigma')\right| \leq \eta
\end{equation}

Now let us introduce the notation $\usigma_{i \rightarrow \alpha}$ to represent $\usigma$ with the $i$-th variable replaced by the alphabet $\alpha$.

Now if we consider only one variable transitions of the chain we get,

\begin{equation}
\label{eq:strong_1}
    \frac12 \sum_{\usigma \in \scc}\sum_{u = 1}^n \sum_{\alpha \in Q, \alpha \neq \sigma_i}| P(\usigma_{u \rightarrow \alpha}|\usigma) \nu(\usigma) - P(\usigma|\usigma_{u \rightarrow \alpha}) \nu(\usigma_{u \rightarrow \alpha})| \leq \eta
\end{equation}

Since $P$ is a reversible Markov chain, the following relations can be derived from the detailed balance conditions,
\begin{equation}
    \frac{ P(\usigma_{u \rightarrow \alpha}|\usigma)}{ P(\usigma|\usigma_{u \rightarrow \alpha})} = \frac{ \mu(\usigma_{u \rightarrow \alpha})}{ \mu(\usigma)}  = \frac{  \mu(\alpha| \usigma_{\su})} {  \mu(\sigma_u| \usigma_{\su})}.
\end{equation}

After plugging these into \eqref{eq:strong_1} and a bit of algebra gives us,
\begin{equation}
   \frac12 \sum_{\usigma \in \scc}\sum_{u = 1}^n \sum_{\alpha \in Q, \alpha \neq \sigma_i}  \frac{ P(\usigma_{u \rightarrow \alpha}|\usigma)}{ \mu(\alpha| \usigma_{\su})} \Big| \mu(\alpha| \usigma_{\su})  \nu(\usigma)  -  \mu(\sigma_u| \usigma_{\su}) \nu(\usigma_{u \rightarrow \alpha}) \Big| \leq \eta.
\end{equation}

Now we can take the marginal probability $\nu(\usigma_{\setminus u})$ out of the expression within the summations in the LHS. Then we can use the definition of conditionals given by: $\nu(\sigma_i| \sigma_{\setminus i}) =\frac{\nu(\usigma)}{\nu(\usigma_{\setminus i})}$,    $\nu(\alpha| \sigma_{\setminus i}) =\frac{\nu(\usigma_{u \rightarrow \alpha})}{\nu(\usigma_{\setminus i})}$,  to rewrite it as follows,

\begin{equation}
   \frac12 \sum_{\usigma \in \scc}\sum_{u = 1}^n \sum_{\alpha \in Q, \alpha \neq \sigma_i}  \frac{ P(\usigma_{u \rightarrow \alpha}|\usigma)}{ \mu(\alpha| \usigma_{\su})} \nu(\usigma_{\setminus u}) \Big| \mu(\alpha| \usigma_{\su})  \nu(\sigma_u|\usigma_{\setminus u})  -  \mu(\sigma_u| \usigma_{\su}) \nu(\alpha| \usigma_{\setminus u}) \Big| \leq \eta.
\end{equation}

Now applying Condition \ref{cond:bounded_flip} we can write this as,

\begin{equation} \label{eq:disc_temp}
   \frac12 \sum_{\usigma \in \scc}\sum_{u = 1}^n   \nu(\usigma_{\setminus u})\sum_{\alpha \in Q, \alpha \neq \sigma_i}  \Big| \mu(\alpha| \usigma_{\su})  \nu(\sigma_u|\usigma_{\setminus u})  -  \mu(\sigma_u| \usigma_{\su}) \nu(\alpha| \usigma_{\setminus u}) \Big| \leq \frac{\eta}{\omega_P}.
\end{equation}

Now we assert that the expression  $\sum_{\alpha \in [Q], \alpha \neq \sigma_i}  \Big| \mu(\alpha| \usigma_{\su})  \nu(\sigma_u|\usigma_{\setminus u})  -  \mu(\sigma_u| \usigma_{\su}) \nu(\alpha| \usigma_{\setminus u}) \Big|$ is lower bounded simply by $|\nu(\sigma_u | \usigma_{\setminus u}) - \mu(\sigma_u | \usigma_{\setminus u})|$.

To show this consider two vectors $a, b \in \mathbb{R}^{|Q|}$, such that they sum to one, $\sum_i a_i = \sum_i b_i = 1$. It is simple to see that,

\begin{equation}
|a_k - b_k| = | a_k \sum_{j \in Q} b_j - b_k\sum_{j \in Q}a_j | =| \sum_{j \in Q, j\neq k} a_k b_j - a_j b_k| \leq \sum_{j \in Q, j\neq k} |a_k b_j - a_j b_k|
\end{equation}

Since the conditionals sum to one, as a consequence of the above inequality we can see that,
$\sum_{\alpha \in Q, \alpha \neq \sigma_i}  \Big| \mu(\alpha| \usigma_{\su})  \nu(\sigma_u|\usigma_{\setminus u})  -  \mu(\sigma_u| \usigma_{\su}) \nu(\alpha| \usigma_{\setminus u}) \Big| \geq |\nu(\sigma_u | \usigma_{\setminus u}) - \mu(\sigma_u | \usigma_{\setminus u})|$. Putting this into \eqref{eq:disc_temp}, we get,

\begin{equation} 
   \frac12 \sum_{u = 1}^n  \sum_{\usigma \in \scc}  \nu(\usigma_{\setminus u})\Big|\nu(\sigma_u | \usigma_{\setminus u}) - \mu(\sigma_u | \usigma_{\setminus u})\Big| \leq \frac{\eta}{\omega_P}.
\end{equation}

Now by splitting the inner sum over $\usigma$ as one going over $\sigma_u$ and another over $\usigma_{\setminus u}$, we can easily deduce that,

\begin{equation} 
    \sum_{u = 1}^n  \sum_{\usigma \in \scc}  \nu(\usigma)\Big|\nu(. | \usigma_{\setminus u}) - \mu(.| \usigma_{\setminus u}))\Big|_{TV} \leq \frac{\eta}{\omega_P}.
\end{equation}

\end{proof}

\section{Proof of corollary \ref{thm:PL_close}}
\begin{proof}
 For a fixed $u \in [n]$  and parametric conditionals are given by $
    p(q|\usigma_{\setminus u}; \utheta) = \frac{1}{1 +  \sum_{p \in Q, p\neq q }e^{ E(\usigma_{u \rightarrow p}, \utheta ) - E(\usigma_{u \rightarrow q}, \utheta)}}.$ And the corresponding log-likelihoods are $\lc_u(\utheta, \usigma) = -\log( p(\sigma_u|\usigma_{\setminus u}; \utheta) )$ The average KL divergence between the conditionals of $\nu$ and the parametric hypothesis can be written as,

\begin{align}
   \EXp{\usigma_{\su} \sim \nu_{\su}} D_{KL}\left(\nu(.|\usigma_{\su} || p(.|\usigma_{\su}; \utheta_u) \right)  = \EXp{\usigma \sim \nu} \log \left(\nu(\sigma_u|\usigma_{\su }) ( 1 +  \sum_{p \in Q, p\neq q }e^{ E(\usigma_{u \rightarrow p}, \utheta ) - E(\usigma_{u \rightarrow q}, \utheta)} ) \right) \\
   = \EXp{\usigma \sim \nu} \log \left(\nu(\sigma_u|\usigma_{\su })\right) + \EXp{\usigma \sim \nu} \lc_u(
    \utheta, \usigma) 
\end{align}

For convenience, let us define $\rho$ as the lower bound on the conditionals of  the parametric distribution, which is assumed to also include the true distribution, i.e.  $\rho \equiv \min_{\sigma, \utheta, u} p(\sigma_u| \usigma_{\setminus u }; \utheta).$

From the reverse Pinsker's inequality (Lemma 4.1 in \cite{10.1214/19-EJP338}) and the lower bound on the conditionals we see that,
$D_{KL}(\nu(.|\sigma_{\su})|| \mu(.|\sigma_{\su})) \leq \frac{2}{\rho}  \left| \nu(.|\usigma_{\su}) - \mu(.|\usigma_{\su})\right|_{TV}.$ Using this bound in Theorem \ref{thm:close_condtionals} gives us,
\begin{equation}
    \EXp{\usigma \sim \nu} \log( \nu(\sigma_u| \sigma_{\su}))  +     \EXp{\usigma \sim \nu} \lc(
    \utheta_u^*, \usigma) \leq \frac{2 \eta}{\rho~ \omega_P}.
\end{equation}

Now from the positivity of KL divergence,  $    \EXp{\usigma \sim \nu} \log( \nu(\sigma_u| \sigma_{\su})) \geq  \EXp{\usigma \sim \nu} \log( p(\sigma_u| \sigma_{\su}; \hutheta_u)) = -\EXp{\usigma \sim \nu}  \lc(\hutheta_u, \usigma).$
This gives us,
\begin{equation}
   \EXp{\usigma \sim \nu}  \lc(\utheta^*_u, \usigma) - \EXp{\usigma \sim \nu}  \lc(\hutheta_u, \usigma)  \leq \frac{2 \eta}{\rho~\omega_P}
\end{equation}
Now the bounds on the parametric conditionals can be used to bound the likelihood  $\lc(\theta,\usigma)$   as follows,
$ \log(\frac{1}{1-\rho}) \leq \lc(\theta, \usigma) \leq \log(\frac{1}{\rho}).$
This establishes $\lc(\theta,\sigma)$ as a bounded random variable. Now we can use Hoeffding's inequality (Lemma ef{lem:hoeff}) to show that with probability $1 -\delta,$
\begin{align}
  \frac{1}{M'}\left(\sum_{t = 1}^{M'} ~ \lc(\utheta_u^*, \usigma^{(t)}) -  \sum_{t = 1}^{M'} ~ \lc(\hutheta_u, \usigma^{(t)})\right)   &\leq    \EXp{\usigma \sim \nu}  \lc(\utheta^*_u, \usigma) - \EXp{\usigma \sim \nu}  \lc(\hutheta_u, \usigma)  + \log\left(\frac{1-\rho}{\rho}\right)\sqrt{ \frac{\log(1/\delta)}{2 M'}},\\
  &\leq\frac{2\eta}{\rho \omega_P} +  \log\left(\frac{1-\rho}{\rho}\right)\sqrt{ \frac{\log(1/\delta)}{2 M'}}.
\end{align}

Now from \eqref{eq:condtional_bounds} we see that $\rho \geq \frac{1}{1 +  (|Q|-1)e^{2  \gamma}}.$ Plugging this into the above equation gives us the desired result.

\end{proof}

\section{Proofs of Learning guarantees}

For Ising models, the exact PL estimator and the one computed from samples have the following form

\begin{equation}
    \lc(\utheta_{u}) = \EXp{\usigma \sim \nu} \lc(\utheta_{ u}, \usigma) =  \Enu ~\log( 1+ \exp(-2 (\sum_{k \neq u} \theta_{u,k} \sigma_u \sigma_k + \theta_u)) ) 
\end{equation}
\begin{equation}
    \lc_M(\utheta_{u}) = \frac{1}{M} \sum_{t = 1}^M \lc(\utheta_{u}, \usigma^{(t)}) = \frac{1}{M} \sum_{t = l}^M \log( 1+ \exp(-2 \sigma^{(t)}_u(\sum_{k \neq u} \theta_{u,k}  \sigma^{(t)}_k  + \theta_u) ) )
\end{equation}

\begin{equation}
    \hutheta_{u} \coloneqq \argmin{||\utheta_{u}||_1 \leq  \gamma } \lc_M(\utheta_{u}).
\end{equation}

A key quantity we will use through out this proof is the following measure of curvature of the loss around $\utheta^*_u$

\begin{equation}\label{eq:curv}
\delta\lc_M(\Delta,\utheta_u^*) \coloneqq \lc_M(\hutheta_u) -  \lc_M(\utheta^*_u) - \langle \Delta, \nabla \lc_M(\utheta^*_u) \rangle 
\end{equation}

\subsubsection{Proof of Theorem \ref{thm:learning_l1}}

\begin{proof}
First let us define two error vectors that respectively capture the error in all parameters connected to a variable at $u$ and only the error in the pairwise terms connected to $u.$
\begin{align}
    \Delta &\coloneqq \hutheta_{u} - \utheta^*_{u} \in \mathbb{R}^n, \\
    \Delta_{\su} &\coloneqq \hutheta_{\su} - \utheta^*_{\su} \in \mathbb{R}^{n-1}.
\end{align}
That is, these two error vectors only differ by $\hat{\theta}_u - \theta^*_u$.

Now from the definition of the curvature function in \eqref{eq:curv},
\begin{align}
\lc_M(\hutheta_{u}) -  \lc_M(\utheta^*_{u}) &=\delta\lc_M(\Delta,\utheta^*_{u}) +  \langle \Delta, \nabla \lc_M(\utheta^*_{u}) \rangle,\\
&\geq  \delta\lc_M(\Delta,\utheta_{u}^*) -  \left|\langle \Delta, \nabla \lc_M(\utheta_{u}^*) \rangle\right|,\\
&\geq  \delta\lc_M(\Delta,\utheta^*_{u}) -  ||\Delta||_1 || \nabla \lc_M(\utheta^*_{u})||_\infty
\end{align}

Now to avoid a contradiction with the definition of $\hutheta$ as the minimizer of $\lc_M$, the quantity on the left must be negative. This then implies that $\delta\lc_M(\Delta,\utheta^*_{u}) \leq   ||\Delta||_1 || \nabla \lc_M(\utheta^*_{u})||_\infty$.
Due to the $\ell_1$ constraint in the problem we have $|| \Delta ||_1 \leq 2 \gamma.$ This gives us,
\begin{equation}
\label{eq:bound_1}
\delta\lc_M(\Delta,\utheta^*_{u}) \leq   2 \gamma || \nabla \lc_M(\utheta^*_{u})||_\infty.
\end{equation}

Now we have two technical lemmas that give relevant bounds that can turn the above inequality in to a learning guarantee.

From Lemma \ref{lem:grad_small_stoc}, when $M \geq  \frac{8}{ \varepsilon^2_a} \log(\frac{4 n^2}{ \delta})$ the following statement holds with probability $1 - \frac{\delta }{2n}$,
\begin{equation}
||\nabla \lc_M(\utheta^*_{u})||_{\infty} \leq  \frac{4 \eta}{\omega_P}  +  \varepsilon_a
\end{equation}

From Lemma \ref{lem:rsc_stoc}, when $M \geq \frac{2 \gamma^4}{\varepsilon^2_b}\log(\frac{2n}{\delta})$, the  following statement holds with probability $1 - \frac{\delta}{2n}$, 
\begin{equation}
\delta\lc_M(\Delta, \utheta^*_{u})\geq   \frac{ e^{-4 \gamma}}{2}||\Delta_{\su}||_{\infty}^2  -  \frac{4 e^{-2 \gamma} \eta}{ \omega_P}   - \varepsilon_b
\end{equation}

Using these bounds in \eqref{eq:bound_1}  gives us the following bound with probability $1 - \delta/n$,

\begin{equation}
    ||\Delta_{\su}||_{\infty}^2  -  \frac{8 \eta  e^{2 \gamma }}{ \omega_P}    - 2 e^{4 \gamma }\varepsilon_b~ \leq~4 e^{4 \gamma } \gamma \left( \frac{4 \eta}{\omega_P} + \varepsilon_a \right).
\end{equation}

Now make the following choices; choose $\varepsilon_a$ and $\varepsilon_b$ such that $4 e^{4 \gamma } \gamma \varepsilon_a = 2 e^{4 \gamma }\varepsilon_b = \varepsilon^2/2$,  and choose $M = \left \lceil 2^{9} \frac{e^{ 8\gamma } \gamma^4}{ \varepsilon^4} \log(\frac{8 n^2}{\delta}) \right  \rceil \geq \max\left(\frac{2 \gamma^4}{\varepsilon^2_b}\log(\frac{2n}{\delta}), \frac{8}{ \varepsilon^2_a} \log(\frac{8 n^2}{ \delta}) \right).$  Plugging these choices in the expression above gives us the following bound with probability $1- \delta/n.$
\begin{align}
   ||\Delta_{\su}||_{\infty}^2  &\leq  \frac{8 \eta e^{2 \gamma }}{ \omega_P} + \frac{16 e^{4 \gamma } \gamma \eta}{\omega_P} + \varepsilon^2 ,\\
   &\leq \frac{16 \eta (1  + \gamma) e^{4 \gamma} }{\omega_P} + \varepsilon^2. \\
\end{align}

This implies that with the same probability,
\begin{equation}
    || \utheta_{\su}^* - \hutheta_{\su}||_{\infty} \leq ~~\sqrt{ \varepsilon^2 +  \frac{16 \eta (1  + \gamma) e^{4 \gamma} }{\omega_P}}  \leq~~ \varepsilon + 4 e^{2 \gamma}\sqrt{\frac{(1+\gamma) \eta}{\omega_P}}
\end{equation}

Now a union bound over each $u \in [n]$ gives us the desired bound in the theorem statement.

\end{proof}

\subsection{Technical lemmas for PL estimator without sparsity}

\subsubsection{Gradient concentration}
\begin{lemma}
\label{lem:grad_small}
For an $\eta$-strongly metastable $||\nabla \lc(\utheta_{u}^*)||_{\infty} \leq  \frac{4 \eta}{\omega_P} $
\end{lemma}

\begin{proof}
\begin{align}
\dfrac{\partial \lc(\utheta_{u}^*)}{\partial \theta_{uk}} &= \Enu ~\frac{ - 2\sigma_u \sigma_k}{ 1+ \exp(2 \sigma_u (\sum_{k' \neq u} \theta^*_{u,k'} \sigma_{k'} + \theta^*_u) ) }, \\
&=  - 2 \Enu ~\mu(-\sigma_u | \sigma_{\su}) \sigma_u \sigma_{k}, \\
&= -2\Enu ~\left(\mu(-\sigma_u | \sigma_{\su}) - \nu(-\sigma_u|\sigma_{\su})\right) \sigma_{k}\sigma_u.
\end{align}

In the last line, we have used following relation that holds for all $k \neq u$, 
$$\Enu \nu (-\sigma_u | \sigma_{\su}) \sigma_u\sigma_k = \sum_{\usigma} \nu(\usigma_{\su}) \nu(\sigma_u|\usigma_{\su})\nu(-\sigma_u|\usigma_{\su}) \sigma_u \sigma_k =  \sum_{\usigma_{\su}} \nu(\usigma_{\su}) \underbrace{\nu(\sigma_u|\usigma_{\su})\nu(-\sigma_u|\usigma_{\su})}_{\text{this is independent of $\sigma_u.$}} \sigma_k  \sum_u \sigma_u  = 0.$$ 

Now we can bound this gradient directly from Theorem \ref{thm:close_condtionals},

\begin{align}
\left|  \dfrac{\partial \lc(\utheta_{u}^*)}{\partial \theta_{u,k}}\right| &\leq 2 \Enu ~|\left(\mu(-\sigma_u | \sigma_{\su}) - \nu(-\sigma_u|\sigma_{\su})\right)|\\
&\leq \frac{4 \eta}{\omega_P}.
\end{align}

Similarly one can also show that,
\begin{align}
\left|  \dfrac{\partial \lc(\utheta_{u}^*)}{\partial \theta_{u}}\right| \leq \frac{4 \eta}{\omega_P}.
\end{align}

\end{proof}

\begin{lemma}
\label{lem:grad_small_stoc}
For a $\eta-$strongly metastable distribution and $\varepsilon_a, \delta_a > 0$. Given $M  \geq \frac{8}{ \varepsilon^2_a} \log(\frac{2 n}{ \delta_a})$ i.i.d. samples guarantees  with probability at least $1 - \delta_a$ that,
\begin{equation}
||\nabla \lc_M(\utheta_u^*)||_{\infty} \leq  \frac{4 \eta}{\omega_P}  +  \varepsilon_a
\end{equation}
\end{lemma}

\begin{proof}
By direct calculation as done in the proof of Lemma \ref{lem:grad_small},
\begin{align}
\dfrac{\partial \lc(\utheta_{u}^*)}{\partial \theta_{u,k}}
&=  - 2 \Enu ~\mu(-\sigma_u | \sigma_{\su}) \sigma_u \sigma_{k}.
\end{align}

Now the variable in the expectation can be easily bounded,

\begin{equation}
 -1\leq   \sigma_u \sigma_k \mu(-\sigma_u|\usigma_{\su}) \leq 1.
\end{equation}

This implies from Hoeffding's inequality that,
\begin{equation}
   Pr\left(  \left|\dfrac{\partial \lc(\utheta^*_u)}{\partial \theta_{u,k}} -  \dfrac{\partial \lc_M(\utheta^*_u)}{\partial \theta_{u,k}} \right|  \geq \varepsilon_a \right)  \leq  2 \exp \left( \frac{- M  \varepsilon^2_a}{8} \right).
\end{equation}

Repeating the same steps we can show that a similar tail bound also holds for the derivative w.r.t $\theta_u.$
 Now choosing $M \geq \frac{8}{\varepsilon^2_a} \log(\frac{2 n }{ \delta_a})$ gives us the desired results by the union bound.
\ca
\begin{equation}
   Pr\left( \left|\dfrac{\partial \lc(\utheta^*_u)}{\partial \theta_{u,k}} -  \dfrac{\partial \lc_M(\utheta^*_u)}{\partial \theta_{u,k}} \right|  \geq \varepsilon_a \right)  \leq  \frac{\delta_a}{n}.
\end{equation}
Now using the union bound here along with Lemma \ref{lem:grad_small} gives us the desired result.
\cb
\end{proof}

\subsubsection{Strong convexity-bounds on curvature}
It is useful to define the following function on reals,
\begin{equation} \label{eq:f_fn_def}
    f(x) \coloneqq \log(1 + \exp(-2x)).
\end{equation}

\begin{lemma}(Smoothnes and Strong convexity of $f$) \\
\label{lem:strong_convex_f}
    Let $f$  be as defined in \eqref{eq:f_fn_def}, and let $\delta f (x,\varepsilon) \coloneqq f(x +\varepsilon) - (f(x) + \varepsilon f'(x)). $ Then if $\max(|x|,|x + \varepsilon|) \leq \gamma,$
    \begin{equation}
    \frac{\varepsilon^2}{2}\geq  \delta f( x,\varepsilon) \geq  \frac{ \exp(-2\gamma) \varepsilon^2}{2} 
    \end{equation}
\end{lemma}
\begin{proof}
This lemma can be shown using standard arguments from the strong convexity of $f$ \cite{zhou2018fenchel}.
We can see that for any $|y| \leq \gamma$ we have,
\begin{equation}
\label{eq:Cgamma_def}
    f''(y)  = \frac{2}{1 + \cosh(2y)} \geq  \frac{2}{1 + \cosh(2\gamma)}  \geq \exp(-2\gamma).
\end{equation}

From this, it is clear that $f $ is convex in the domain $[-\gamma, \gamma]$. Moreover, the above expression gives that, $g(y) \coloneqq f(y) - \frac{exp(-2 \gamma)}{2} y^2$ is also a convex function in the same domain. Now from the first-order convexity condition, $g(x+\varepsilon) \geq g(x) + \varepsilon g'(x) $, we can find the lower bound on $\delta f.$

Similarly, we can see that $f''(y) < 1$. This implies that $h(y) \coloneqq \frac{y^2}{2} -f(y)$ is also a convex function. From the relation $h(x+\varepsilon) \geq h(x) + \varepsilon h'(x)$ we get the upper bound on the $\delta f.$
\end{proof}

\begin{lemma} \label{lemma:rsc_ple}  
   $\delta\lc(\Delta, \utheta_u^*) \geq  \frac{ e^{-4\gamma }}{2}||\Delta_{\su}||_{\infty}^2  -  \frac{8 e^{-2 \gamma} \eta}{\omega_P}$
\end{lemma}

\begin{proof}
 Define local energy function $E_u(\usigma; \utheta) = \sum_{k \neq u} \theta_{u,k}\sigma_u\sigma_k + \theta_u \sigma_u $. Notice that this function is linear in the $\theta$ variables. Now PL loss can be expressed as a function of this local energy. From the definition of $f$ in \eqref{eq:f_fn_def}, we have the following relations,

\begin{align}
    \lc(\utheta_u + \Delta)  =  \Enu f(E_u(\usigma;\utheta_u + \Delta))  =  \Enu f(E_u(\usigma;\utheta_u) + E_u(\usigma;\Delta)). 
\end{align}
\begin{align}
    \langle \nabla \lc(\utheta_u), \Delta \rangle &= \Enu f'(E_u(\usigma;\utheta_u))(\sum_{k \neq u} \dfrac{\partial E_u(\usigma;\utheta_u)}{\partial \theta_{uk}} \Delta_{uk} + \dfrac{\partial E_u(\usigma;\utheta_u)}{\partial \theta_{u}} \Delta_u )  = \Enu f'(E_u(\usigma;\utheta_u)) E_u(\usigma;\Delta).
\end{align}

Now from the definitions in Lemma \ref{lem:strong_convex_f}, we can define the curvature of $\lc$ mimicking \eqref{eq:curv} as follows,
\begin{equation}
   \delta\lc(\Delta, \utheta_u^*) =  \Enu \delta f(E_u(\usigma, \utheta_u^*), E_u(\usigma,\Delta)).
\end{equation}

Now from Condition \ref{cond:temperature_gen} we have,  $|E_u(\usigma,\utheta_u^*)| \leq ||\utheta^*_u||_1 \leq \gamma.$ The same condition will also hold for $|E_u(\usigma,\hutheta_u)|_1$ as these constraints are imposed in the optimization.

This implies from Lemma \ref{lem:strong_convex_f} that,

\begin{align}\label{eq:strong_convex_1}
   \delta\lc(\Delta, \utheta_{u}^*) &\geq \frac{\exp(-2 \gamma)}{2}\Enu  (E_u(\usigma,\Delta))^2 =  \frac{\exp(-2\gamma)}{2} \Enu \left(\sum_{k \neq  u} \sigma_k \Delta_{u,k} + \Delta_u \right)  \geq \frac{\exp(-2\gamma)}{2} ~ \Var{\usigma \sim \nu}~\left[\sum_{k \neq u} \Delta_{uk} \sigma_k\right].
\end{align}
 Now for some $i \neq u$, we can use the law of conditional variances to bound this quantity. The choice of $i$ will be made later to get optimal bounds.
 
\begin{equation}
     \Var{\usigma \sim \nu} ~\left[\sum_{k \neq u} \Delta_{uk} \sigma_k\right] \geq \EXp{ \usigma_{\setminus i} \sim \nu_{\setminus i}}~~ \Var{\sigma_i \sim \nu(.|\usigma_{\setminus i})} \left( \sum_{k \neq u} \Delta_{uk} \sigma_k {\bigg |} \usigma_{\setminus i} \right) = \Delta^2_{ui} \EXp{ \usigma_{\setminus i} \sim \nu_{\setminus i}}~~ \Var{\sigma_i \sim \nu(.|\usigma_{\setminus i})} \left(\sigma_k {\bigg |} \usigma_{\setminus i} \right) \label{eq:temp2}
\end{equation}

Now as a consequence of the closeness of conditionals proved in Theorem \ref{thm:close_condtionals}, we can show that the conditional variance of the metastable state is close to that of the true distribution $\mu$. This is shown in Lemma \ref{lem:variance_close}. Moreover the conditional variance w.r.t $\mu$ is naturally lower bounded by the finite temperature bound,
\begin{align}
\Var{\sigma_u \sim \mu(.|\usigma_{\setminus u})}~[ \sigma_u| \usigma_{\setminus u}] = 1 - \tanh^2( \sum_{j \neq u } \theta^*_{uj} \sigma_j + \theta^*_u) \geq 1 - \tanh^2(\gamma) \geq e^{-2 \gamma}.
\end{align}

Now define a function that captures the error incurred in replacing the metastable distribution with the equilibrium distribution in the conditional variance,\newline $G(\usigma_{\setminus i}) \coloneqq   \Var{\sigma_i \sim \nu(.|\usigma_{\setminus i})} \left[\sigma_i{|} \usigma_{\setminus i} \right] -   \Var{\sigma_i \sim \mu(.|\usigma_{\setminus i})} \left[\sigma_i{|} \usigma_{\setminus i} \right]$

Using this in \eqref{eq:temp2} we find,
\begin{align}
    \Var{\usigma \sim \nu} ~\left[\sum_{k \neq u} \Delta_{uk} \sigma_k\right] \geq \Delta^2_{ui} \left(e^{-2 \gamma} - \EXp{\usigma \sim \nu}  |G(\usigma_{\setminus i})| \right)  \\
\end{align}

Now by using the result in Lemma \ref{lem:variance_close} we can upper bound $\EXp{\usigma \sim \nu}  |G(\usigma_{\setminus i})|$. This gives us

\begin{align}
    \Var{\usigma \sim \nu} ~\left[\sum_{k \neq u} \Delta_{uk} \sigma_k\right] \geq \Delta^2_{ui} \left(e^{-2 \gamma} -  \frac{4 \eta}{\omega_P} \right) =  ||\Delta_{\su}||^2_{\infty} \left(e^{-2 \gamma} -  \frac{4 \eta}{\omega_P} \right)  \\
\end{align}

We have chosen the $i$ here such that $\Delta^2_{ui} = ||\Delta_{\su} ||_{\infty}^2$. Plugging this back in \eqref{eq:strong_convex_1} gives us the desired result.

\end{proof}

\begin{lemma}
For a $\eta-$strongly metastable distribution and $\varepsilon_b, \delta_a \geq 0$. Given $M \geq \frac{2 \gamma^4}{\varepsilon^2_b}\log(\frac{1}{\delta_b})$  guarantees  with probability at least $1 - \delta_b$ that,
$ \delta\lc_M(\Delta, \utheta_u^*) \geq  \frac{ e^{-4\gamma }}{2}||\Delta||_{\infty}^2  -  \frac{4 e^{-2 \gamma} \eta}{\omega_P} - \varepsilon_b$
\label{lem:rsc_stoc}
\end{lemma}

\begin{proof}
    
\begin{equation}
   \delta\lc_M(\Delta, \utheta^*) = \frac{1}{M}  \sum_{t=1}^M  \delta f(E_u(\usigma^{(t)}, \utheta^*), E_u(\usigma^{(t)},\Delta)).
\end{equation}

We will bound this using Hoeffding's inequality(Lemma  \ref{lem:hoeff}). To this end, we need upper and lower bounds on $\delta f$. We can get this easily from Lemma  \ref{lem:strong_convex_f},
\begin{equation}
0 <      \delta f(E_u(\usigma, \utheta^*), E_u(\usigma,\Delta)) \leq  \frac{|| \Delta_u ||^2_1}{2} \leq 2 \gamma^2.
\end{equation}

Now using Hoeffding inequality,

\begin{equation}
    Pr\left( \delta \lc_M(\Delta, \utheta^*) \leq \delta \lc(\Delta, \utheta^*) - \varepsilon_b\right) \leq \exp \left( \frac{M\varepsilon_b^2}{2 \gamma^4}\right)
\end{equation}
So choosing $M = \frac{2 \gamma^4}{ \varepsilon^2_b}\log(\frac{1}{\delta_b})$ ensures that $\delta \lc_M(\Delta, \utheta^*) > \delta \lc(\Delta, \utheta^*) - \varepsilon_b$ with probability $1 -\delta_b$. Now using the lower bound in Lemma \ref{lemma:rsc_ple} gives us the required lower bound in the lemma.
\end{proof}

\begin{lemma}(Closeness of conditional variance)\\
\label{lem:variance_close}
Let $\mu$ and $\nu$ close in conditionals as defined in Theorem \ref{thm:close_condtionals}. Let $F:\mathbb{R}^{n-1} \rightarrow \mathbb{R}$ be an arbitrary function.  Then the conditional variances of this random variable under these distributions are also close in the following sense,
\begin{align}
\EXp{\nu}\left|F(\usigma_{\setminus i})\left( \Var{\mu}[\sigma_i| \usigma_{\setminus i}] -   \Var{\nu}[\sigma_i| \usigma_{\setminus i}] \right) \right|  \leq    \frac{4\eta}{\omega_P} \left|\max_{\underline{x} \in \{-1,1\}^{n-1}} F(\underline{x})\right|.
\end{align}
\end{lemma}

\begin{proof}

First notice the elementary fact that the TV upper bounds the difference in any expectation value, $| \EXp{\usigma \sim P} x(\usigma) - \EXp{\usigma \sim Q}x(\usigma) | \leq 2|P-Q|_{TV} || x ||_\infty.$

Now for a fixed partial configuration of the spin $\usigma_{\setminus i}.$
\begin{align}
\Var{\mu}[\sigma_i| \usigma_{\setminus i}] -   \Var{\nu}[\sigma_i| \usigma_{\setminus i}]  &=   \EXp{\mu}[\sigma^2_i| \usigma_{\setminus i}] - \EXp{\nu}[\sigma^2_i| \usigma_{\setminus i}]  \\& - \left( \EXp{\mu}[\sigma_i| \usigma_{\setminus i}] - \EXp{\nu}[\sigma_i| \usigma_{\setminus i}] \right) \left( \EXp{\mu}[\sigma_i| \usigma_{\setminus i}] + \EXp{\nu}[\sigma_i| \usigma_{\setminus i}] \right) \nonumber
\end{align}

The first term vanishes as $\sigma^2_i = 1$.  Let $|\mu(.| \usigma_{\setminus i}) - \nu(.|\usigma_{V\su}) |_{TV} \coloneqq \delta_i(\usigma_{\setminus i}).$ Now using the TV upper bound,
\begin{align}
\left|\Var{\mu}[\sigma_i| \usigma_{\setminus i}] -   \Var{\nu}[\sigma_i| \usigma_{\setminus i}] \right|  \leq   4 \delta_i(\usigma_{\setminus i}) .
\end{align}

This implies that from Theorem \ref{thm:close_condtionals}
\begin{align}
\EXp{\nu} \left|F(\usigma_{\setminus i})\left(\Var{\mu}[\sigma_i| \usigma_{\setminus i}] -   \Var{\nu}[\sigma_i| \usigma_{\setminus i}]\right) \right|  \leq 4 \EXp{\nu} \left|F(\usigma_{\setminus i})\right| \delta_i(\usigma_{\setminus i})  \leq \frac{4\eta}{\omega_P} \left|\max_{\underline{x} \in \{-1,1\}^{n-1}} F(\underline{x})\right|.
\end{align}
\end{proof}

\begin{lemma}[Hoeffding's Inequality, \cite{vershynin2018high}]\label{lem:hoeff}
Let $ X^{(1)}, X^{(2)}, \ldots, X^{(M)}$ be independent and identically distributed random variables such that $a \leq X^{(i)} \leq b $  and $\varphi = \mathbb{E}[X^{(i)}]$ for all $ i = 1, 2, \ldots, M $. Define the empirical mean \( S_M = \frac{1}{M}\sum_{i=1}^M X^{(i)} \). Then, for any \( t > 0 \),
\[
\mathbb{P}(S_M - \varphi \geq t) \leq \exp\left(-\frac{2Mt^2}{(b-a)^2}\right),
\]
and
\[
\mathbb{P}(S_M - \varphi \leq -t) \leq \exp\left(-\frac{2Mt^2}{(b-a)^2}\right),
\]
and 
\[
\mathbb{P}(|S_M - \varphi| \leq t) \leq 2 \exp\left(-\frac{2Mt^2}{(b-a)^2}\right).
\]
\end{lemma}

\ca
\begin{lemma}(Closeness of conditional variance)\\
\label{lem:general_variance_close}
Let $\mu$ and $\nu$ close in conditionals as defined in Theorem ..., and let $x(\sigma)$  be a random variable defined over spin variables $\usigma$. Then the conditional variances of this random variable under these distributions are also close in the following sense,
\begin{align}
\EXp{\nu}\left|\Var{\mu}[x(\usigma)| \usigma_{\su}] -   \Var{\nu}[x(\usigma)| \usigma_{\su}] \right|  \leq   3 \eta || x ||^2_{\infty} .
\end{align}
where $|| x||_{\infty} = \max_{\usigma} |x(\usigma)|$
\end{lemma}

Let $|\mu(.| \usigma_{\su}) - \nu(.|\usigma_{V\su}) |_{TV} \coloneqq \delta_u(\usigma_{\su}).$

Notice that the TV upper bounds the difference in any expectation value, $| \EXp{\usigma \sim P} x(\usigma) - \EXp{\usigma \sim Q}x(\usigma) | \leq |P-Q|_{TV} || x ||_\infty.$

Now for a fixed partial configuration of the spin $\usigma_{\su}.$
\begin{align}
\Var{\mu}[x(\usigma)| \usigma_{\su}] -   \Var{\nu}[x(\usigma)| \usigma_{\su}]  &=   \EXp{\mu}[x^2(\usigma)| \usigma_{\su}] - \EXp{\nu}[x^2(\usigma)| \usigma_{\su}]  \\& - \left( \EXp{\mu}[x(\usigma)| \usigma_{\su}] - \EXp{\nu}[x(\usigma)| \usigma_{\su}] \right) \left( \EXp{\mu}[x(\usigma)| \usigma_{\su}] + \EXp{\nu}[x(\usigma)| \usigma_{\su}] \right) \nonumber
\end{align}

Now using the TV upper bound,
\begin{align}
\left|\Var{\mu}[x(\usigma)| \usigma_{\su}] -   \Var{\nu}[x(\usigma)| \usigma_{\su}] \right|  \leq   3 \delta_u(\usigma_{\su}) || x ||^2_{\infty} .
\end{align}

This implies that,
\begin{align}
\EXp{\nu}\left|\Var{\mu}[x(\usigma)| \usigma_{\su}] -   \Var{\nu}[x(\usigma)| \usigma_{\su}] \right|  \leq   3 \eta || x ||^2_{\infty} .
\end{align}

\begin{lemma}(Variance of binary random variable)\\
\label{lem:variance_bina ry}
\ajcomm{Is this used anywhere?}
Let $x$ be a binary random variable that takes values according to a probability $P$, then the variance of $x$ can be lower bounded as,
\begin{equation}
\Var{P}[x] \geq \min(P(0),P(1))^2~(x_0 - x_1)^2
\end{equation}
\end{lemma}
\begin{proof}
\begin{align}
    \Var{P}[x] &= P(0) x^2_0 + P(1) x^2_1 - ( P(0)x_0 + P(1)x_1 )^2,\\
   &= P(0)P(1) (  x^2_0 + x^2_1) - 2P(0)P(1)x_0 x_1, \\
   &= P(0)P(1) (x_0 - x_1)^2, \\
   &\geq \min(P(0), P(1))^2  (x_0 - x_1)^2.
\end{align}
    
\end{proof}
\cb
\subsection{Structure learning }

First we identify the edges in the true $d$-sparse graph  by the following criterion,
\begin{equation}
    \hat{E} = \{(u,v) \in [n] \times [n]~~ |~~ \max( |\hat{\theta}_{uv}|, |\hat{\theta}_{vu}|) > \alpha/2  \}
\end{equation}

 Choose $\varepsilon = \alpha/4$. This in turn implies that $M = \left \lceil 2^{10} \frac{e^{ 8\gamma } \gamma^4}{ \varepsilon^4} \log(\frac{8 n}{\delta}) \right  \rceil  = \left \lceil 2^{26} \frac{e^{ 8\gamma } \gamma^4}{ \alpha^4} \log(\frac{8 n}{\delta}) \right \rceil$.  

Now the condition assumed in the theorem is  $\alpha > 16 e^{2 \gamma} \sqrt{\frac{(1 + \gamma)\eta}{\omega_P}}$. Now the upper bound on error in the pairwise terms from Theorem \ref{thm:learning_l1} is  $\varepsilon + 4 e^{2 \gamma}\sqrt{\frac{(1+\gamma) \eta}{\omega_P}}$. The assumed condition lower bound on $\alpha$ then implies that $\varepsilon + 4 e^{2 \gamma}\sqrt{\frac{(1+\gamma) \eta}{\omega_P}} < \alpha/2$. This implies that if $\theta*_{u,v}$ is zero, then the PL estimate $\hat{theta}_{u,v}$ is guaranteed to be less than $\alpha/2$ with high probability. On the other hand is $\theta^*_{u,v} \neq 0$, the estimated value will  be greater that $\alpha/2$. So from Theorem \ref{thm:learning_l1}, with probability $1-\delta$, the estimated structure $\hat{E}$ matches the true structure of  $\mu$. This proves the structure learning result.

\subsection{Recovering magnetic fields}

We will use the following shorthands for vectors $\utheta_{E_u} = [\theta_{u,v}| (u,v) \in E]$  and $\usigma_{E_u} = [ \sigma_v|(u,v) \in E].$ The PL estimator for the magnetic field in given by,

\begin{equation}
\hat{\theta}_u = \argmin{| \theta_u| \leq h_{max}~} ~~  \frac{1}{M}\sum_{t = 1}^M ~\log( 1+  \exp(-2\sigma^{(t)}_u (\langle \hutheta_{E_u}, \usigma^{(t)}_{\setminus u} \rangle + \theta_u) ).
\end{equation}

The same machinery that we used prove Theorem \ref{thm:learning_l1} can be used to show that $\hat{\theta}_u$ is close to $\theta^*_u$.

Define the following loss functions,
\begin{equation}
    \mathcal{H}_M(\theta_u) = \frac{1}{M}\sum_{t = 1}^M ~\log( 1+  \exp(-2\sigma^{(t)}_u (\langle \hutheta_{E_u}, \usigma^{(t)}_{E_u} \rangle + \theta_u) ).
\end{equation}

\begin{equation}
    \mathcal{H}(\theta_u) = \EXp{\usigma \sim \nu}~\log( 1+  \exp(-2\sigma_u (\langle \hutheta_{E_u}, \usigma_{E_u} \rangle + \theta_u) ).
\end{equation}

Also define the curvature of this loss,
\begin{equation}
    \delta \mathcal{H}(\Delta_u, \theta_u) = \mathcal{H}(\theta_u + \Delta_u) - (\mathcal{H}(\theta_u) + \Delta_u \frac{\partial H(\theta_u)}{\partial \theta_u} )
\end{equation}

Now using the same argument that was used to derive \eqref{eq:bound_1}, we can show that
\begin{equation}\label{eq:h_main}
    \delta \mathcal{H}_M(\Delta_u, \theta^*_u) \leq 2 h_{max} \left|\frac{\partial \mathcal{H}_M(\theta^*_u)}{\partial \theta_u}\right|.
\end{equation}

We will exploit this relation in the same way as done in the proof of Theorem \ref{thm:learning_l1}. 

\subsubsection{Upperbound on gradient}

Upperbounding the gradient as in Lemma \ref{lem:grad_small_stoc} is slightly trickier in this case due to the fixed variables in the optimization.
 First define the following sigmoid type function,
\begin{equation}
    S(x) := \frac{1}{1 +\exp(2x)}.
\end{equation}

\begin{align}
    \frac{\partial \mathcal{H}(\theta^*_u)}{\partial \theta_u} &=  -2\EXp{\usigma \sim \nu}~\sigma_u S\left(\langle \hutheta_{E_u}, \usigma_{E_u} \rangle + \theta^*_u\right).
\end{align}

Now we can add and subtract $2\EXp{\usigma \sim \nu}~\sigma_u S\left(\langle \utheta^*_{E_u}, \usigma_{E_u} \rangle + \theta^*_u\right)$. This term can be bounded from Lemma \ref{lem:grad_small} and hence we can write,

\begin{align}
    |\frac{\partial \mathcal{H}(\theta^*_u)}{\partial \theta_u}| &\leq 2\EXp{\usigma \sim \nu}~\left|S\left(\langle \hutheta_{E_u}, \usigma_{E_u} \rangle + \theta^*_u\right)  -   S\left(\langle \utheta^*_{E_u}, \usigma_{E_u} \rangle + \theta^*_u\right)\right| + \frac{4 \eta}{\omega_P},\\
    &\leq 2  |\langle \hutheta_{E_u} -  \utheta^*_{E_u}, 
  \usigma_{E_u}   \rangle|  \sup_{x \in [-\gamma, \gamma ] } |S'(x)|  +  \frac{4 \eta}{\omega_P},\\
  &\leq 4 ||\hutheta_{E_u} -  \utheta^*_{E_u}||_1    + \frac{4 \eta}{\omega_P},\\
  &\leq 4d \varepsilon   + \frac{4 \eta}{\omega_P}.
\end{align}
In the third line, we have used Holders inequality to bound the inner product and also used the explicit formula for $|S'(x)| = |2S(x)S(-x)| < 2$ to bound the gradient. 

In the last line, we have exploited the fact that the estimates of the pairwise terms are only non-zero on the edges in $E$. This can be assumed as we know the structure of the underlying model. Then fact that the graph as at most degree $d$ allows us to bound the $\ell_1$ norm of the error with $\ell_\infty$ norm, and gaining only a factor of $d$ in the process.

Clearly $ -1 \leq \sigma_u S\left(\langle \hutheta_{E_u}, \usigma_{E_u} \rangle + \theta^*_u\right)  \leq 1$. Hence by Hoeffdings inequality (Lemma \ref{lem:hoeff}),
\begin{equation}
    Pr\left(\left| \frac{\partial \mathcal{H}(\theta^*_u)}{\partial \theta_u} - \frac{\partial \mathcal{H}_M(\theta^*_u)}{\partial \theta_u}\right| \geq \epsilon_a\right) \leq 2 \exp(-\frac{M\epsilon^2_a}{2})
\end{equation}

Hence with $M \geq \frac{2}{\epsilon^2_a} \log(\frac{4}{\delta}) $, we can ensure with probability at least $1-\delta/2$ that 

\begin{equation}
    \left|\frac{\partial \mathcal{H_M}(\theta^*_u)}{\partial \theta_u}\right|  \leq   4d \varepsilon  + \frac{4 \eta}{\omega_P} + \epsilon_a
\end{equation}

\subsubsection{Lower bound on curvature}

Now the curvature can be bounded using strong convexity arguments as in Lemma \ref{lem:strong_convex_f}. For this we first connect $\delta H$ to $\delta f$, with $f$ as defined in \eqref{eq:f_fn_def},

\begin{align}
    \delta H(\Delta_u, \theta^*_u) &= \EXp{\usigma \sim \nu} \left( f(\sigma_u ( \langle \hutheta_{E_u}, \usigma_{E_u} \rangle + \theta_u)) - f(\sigma_u ( \langle \hutheta_{E_u}, \usigma_{E_u} \rangle + \theta^*_u)) - \Delta_u \sigma_ u f' (\sigma_u ( \langle \hutheta_{E_u}, \usigma_{E_u} \rangle + \theta^*_u))  \right),\\
    &= \EXp{\usigma \sim \nu} \delta f (\Delta_u \sigma_u, \theta_u^*).
\end{align}

Hence from the lower bound in Lemma \ref{lem:strong_convex_f},

\begin{equation}
        \delta H(\Delta_u, \theta^*_u) \geq \frac{e^{-2 h_{max}}}{2} \EXp{\usigma \sim \nu} (\Delta_u \sigma_u)^2 =  \frac{e^{-2 h_{max}}}{2} (\Delta_u )^2 
\end{equation}

Now from the upper bound in Lemma \ref{lem:strong_convex_f}, we have $0<\delta H(\Delta_u, \theta^*_u) \leq \frac{h_{max}^2}{2}.$

We can use these these bounds with Hoeffding's inequality (Lemma \ref{lem:hoeff}) to get the following probabilistic upper-bound on the estimated curvature,
\begin{equation}
    Pr\left(\delta H_M(\Delta_u, \theta_u^*)  \leq \delta H(\Delta_u, \theta_u^*) - \epsilon_b \right)  \leq e^{\frac{-8 \epsilon^2_b M}{h^4_{max}}}.
\end{equation}

Hence ,$M \geq \frac{h^4_{max}}{8 \epsilon^2_b}\log(\frac{2}{\delta})$ ensures that the corresponding lower bound holds w.p greater than $1-\frac{\delta}{2}.$

\subsubsection{Magnetic field learning guarantee}

Now using the bounds on gradient and curvature in \eqref{eq:h_main} and using the union bound, we have the following bound that holds with probability $1-\delta$ when $M$ is greater than $\max \left(\frac{h^4_{max}}{8 \epsilon^2_b}\log(\frac{2}{\delta}), \frac{2}{\epsilon^2_a} \log(\frac{4}{\delta})\right) $, 
\begin{equation}
    \frac{e^{-2h_{max}}}{2} \Delta^2_u \leq \epsilon_b + 2 h_{max} \left( 4d \varepsilon  + \frac{4 \eta }{\omega_P} + \epsilon_a \right)
\end{equation}

Now choosing $\frac{\varepsilon^2_h}{2}= 2 e^{2h_{max}}\epsilon_b = 4 h_{max} e^{2h_{max}} \epsilon_a,$ gives us 
\begin{equation}
    \Delta_u \leq \varepsilon_h + 4 \sqrt{d \varepsilon h_{max}} e^{h_{max}} +4 \sqrt{\frac{ \eta h_{max}}{\omega_P}}e^{h_{max}}.
\end{equation}

Plugging the definition of $\varepsilon_h$ in the $M$ values, we see that this can be achieved with $M= \left \lceil \frac{2^7 h^4_{max} e^{4 h_{max}}}{\varepsilon^4_h} \right \rceil$ samples.

%The crucial difference comes in the curvature bound (i.e. the lower bound on $\delta L$), where we will now use the fact that the true model is $d$ sparse to get a stronger ($\ell_1$ compared $\ell_\infty$ in Theorem \ref{thm:learning_l1}) error guarantee.
\ca

\subsection{Learning with sparsity, one time }

\begin{lemma}[Sparsity of error]\label{lem:sparsity_err}
Let $\utheta_{E_u} = \{\theta_{u,1} \ldots \theta_{u,u-1}, \theta_{u,u+1}, \ldots, \theta_{u,n}\}$ be the set of optimization variables assocaited with the edges connected to the spin $u$.
Let $\hutheta_{u} = \underset{\utheta}{\text{argmin}}~~\lc(\utheta_{u}) + \lambda || \utheta_{E_u} ||_1$. Let us define the error as, $\Delta \equiv \hutheta_{u} - \utheta^*_{u}.$  

Now identity three different components of the error.  First  $\Delta_u = \hat{\theta}_u - \theta_u^*,$    that measures the error in the magnetic field. $\Delta_{N_u} = \hutheta_{N_u} - \utheta_{N_u}^*$ which measure the error along the non-zero edges in the true model ($N_u$ here denotes the neighbours of $u$ in the underlying graph). $\Delta_{N^c_u} = \hutheta_{N^c_u} - \utheta_{N^c_u}^* = \hutheta_{N^c_u} $ which measure the error along the edges not present in the true model.

Then if $\lambda \geq 2 ||\nabla \lc(\utheta^*_{\kc}) ||_{\infty},$ the following sparsity conditions hold on $\Delta$,

   \begin{enumerate}
       \item  $ 3|| \Delta_{N_u} ||_1 + |\Delta_u | \geq  || \Delta_{N^c_u}||_1$
       \item $||\Delta_{E_u}||_1 \leq  4 \gamma$

       \end{enumerate}
\end{lemma}

\begin{proof}
First use convexity of $\lc$ along with Cauchy-Schwartz inequality
\begin{equation}\label{eq:cs_convex}
\lc(\utheta_{u}^* + \Delta) - \lc( \utheta^*_{u}) \geq \langle \Delta, \nabla \lc( \utheta^*_{u}) \rangle \geq - || \Delta||_1 ||\nabla \lc( \utheta^*_{u})  ||_{\infty}
\end{equation}

Since $\hutheta_u = \utheta_{u}^* + \Delta$ is the minimum point of the $\ell_1$  regularized loss, we can use this to show,
\begin{equation}
    - || \Delta ||_1 ||\nabla \lc( \utheta^*_{u})  ||_{\infty} + \lambda (|| \utheta_{E_u}^* + \Delta_{E_u}||_1 - ||\utheta_{E_u}^*||_1) \leq  \lc(\utheta_{u}^* + \Delta) - \lc( \utheta^*_{u}) +  \lambda (|| \utheta_{E_u}^* + \Delta_{E_u}||_1 - ||\utheta_{E_u}^*||_1)  \leq 0
\end{equation}

Now from $\lambda \geq 2 || \nabla  \lc(\utheta^*_u)||_{\infty} $ we get,
\begin{equation}\label{eq:cs_2}
    || \utheta_{E_u}^* + \Delta_{E_u}||_1 - ||\utheta_{E_u}^*||_1 \leq \frac12 || \Delta ||_1
\end{equation}

For the first statement we follow \cite{neghaban2012high} and use the triangle inequality to get the following bound,
\begin{equation}
   || \utheta_{E_u}^* + \Delta_{E_u}||_1 - ||\utheta_{E_u}^*||_1  = || \Delta_{N_u} + \utheta_{N_u}^* ||_1 + ||\Delta_{N^c_u} || -  ||\utheta_{N_u}^*||_1 \geq - ||\Delta_{N_u} ||_1 +  || \Delta_{N^c_u} || _1
\end{equation}
Now using this in \eqref{eq:cs_2}  we get,
\begin{equation}
 - ||\Delta_{N_u} ||_1 +  || \Delta_{N^c_u} || _1 \leq \frac12 ||\Delta ||_1 = \frac12(  ||\Delta_{N_u} ||_1 +  || \Delta_{N^c_u} || _1 + |\Delta_u|)
 \implies 3|| \Delta_{N_u} ||_1 + |\Delta_u | \geq  || \Delta_{N^c_u}||_1
\end{equation}
Now for the second statement we apply the triangle inequality in a different direction,
\begin{equation}
      || \utheta_{E_u}^* + \Delta_{E_u}||_1 - ||\utheta_{E_u}^*||_1 \geq  || \Delta_{E_u}||_1 -2||\utheta_{E_u}^*||_1  \geq || \Delta_{E_u}||_1 -2||\utheta_{u}^*||_1  \geq  || \Delta_{E_u}||_1 -2 \gamma
\end{equation}
Using this in \eqref{eq:cs_2} gives,

\begin{equation}
    || \Delta_{E_u}||_1 -2\gamma  \leq \frac12 ||\Delta ||_1 \leq \frac12 ||\Delta_{E_u}||_1 \implies ||\Delta_{E_u}||_1 \leq 4 \gamma
\end{equation}

\end{proof}
\subsubsection{Proof of Theorem \ref{thm:learning_sparsity} and Corollary \ref{cor:learning_mag}}    
Corollary concerning the learning of Ising models with magnetic fields follow from Theorem \ref{thm:learning_sparsity}.  But for the purposes of exposition, we will prove the corollary first for making the ideas used in the proof clear.  The complete proof of Theorem 4 will then use the same ideas along with elementary, but tedious, case work.

The main tool we will use is the analysis of regularized $M-$estimators developed in Reference \cite{neghaban2012high}.  We restate the relevant result (Theorem 1, \cite{neghaban2012high}) here.
\begin{lemma}
\label{prop:neghaban}
Let, $\hutheta$ be the minimizer of a regularized convex loss ~$\lc_M(\utheta) + \lambda || \utheta ||_1$. The error of this minimizer from a $d$-sparse point $\utheta^*$ can be bounded as $|| \hutheta - \utheta^* ||_2 \leq \frac{9 \lambda^2}{\kappa^2} d + \frac{2\lambda}{\kappa} \tau^2,$ if the following two conditions hold;
\begin{enumerate}
    \item The regularization parameter staifies the following lower bound dependent of the gradient of the unregularized loss function.
    \begin{equation}
        \lambda \geq 2 ||\nabla \lc_M(\theta^*)||_{\infty}
    \end{equation}

    \item The curvature of the loss function at $\utheta^*$ is lower bounded as follows; let $\Delta = \hutheta - \utheta^*$
    \begin{equation}
        \delta \lc_M(\Delta, \utheta^*) \equiv \lc_M(\hutheta) - \lc_M(\utheta^*) - \langle \Delta, \nabla \lc_M(\utheta^*) \rangle \geq \kappa || \Delta||^2_2  - \tau^2
    \end{equation}
\end{enumerate}
\end{lemma}

%\begin{proof}[Proof of Corollary \ref{cor:learning_mag}]
From Lemma \ref{lem:sparsity_err}  we have that the error,$||\Delta||_1 \leq 4 \gamma$. Hence imposing the constraint  $||\utheta_u||_1 \leq 5 \gamma $ will not change the optimal point,
\begin{equation}
\hutheta_u = \underset{||\utheta||_1 \leq 5 \gamma }{\text{argmin}}~~\lc_M(\utheta_u) + \lambda || \utheta_u ||_1
\end{equation}

Now  from Lemma  \ref{lem:grad_small_stoc} we have that $||\nabla \lc(\utheta^*_u) ||_\infty \leq \frac{2 \eta}{\omega_P} + \sqrt{\frac{8}{M} \log(\frac{4 |\kc_u|}{\delta})}$ with probability $1 -\delta/2$. Now $\eta$ being a TV between two distribtuions, is always less than 1. Hence choosing $\lambda \geq \frac{4}{\omega_P} + 2\sqrt{\frac{8}{M} \log(\frac{4 |\kc_u|}{\delta})}$, will satisfy the first condition of Propostion \ref{prop:neghaban} with the same probability.

To prove the second condition in Proposition \ref{prop:neghaban}, we start from Equation \ref{eq:strong_convex_1}. Since we can work with the constraint $||\utheta_u||_1 \leq 5 \gamma $ without changing $\hutheta$, the curvature can be lower bounded here as,
\begin{align}\label{eq:strong_convex_mag}
   \delta\lc(\Delta, \utheta_u^*) &\geq \frac{\exp(-10 \gamma)}{2}\Enu  (E_u(\usigma,\Delta))^2 
\end{align}

Now let us focus on Corollary \ref{cor:learning_mag} alone.

In this case of Ising models with magnetic fields we can specilize the above bound as,

\begin{align}\label{eq:strong_convex_mag}
   \delta\lc(\Delta, \utheta_u^*) &\geq \frac{\exp(-10 \gamma)}{2}\Enu( \Delta_u + \sum_{j \neq u} \Delta_{u,j}\sigma_j )^2 
\end{align}

Now for a constant $ \psi \coloneqq \tanh( \gamma)  +  \frac{4 \eta}{\omega_P}$,  consider the following cases
\paragraph*{\bf Case 1:} $|\Delta_u| \geq  \frac12(1 + \psi)  \sum_{j \neq u} |\Delta_{u,j}| $
\begin{align}\label{eq:rsc_mag_1}
\Enu (\Delta_u + \sum_{j \neq u} \Delta_{u,j}\sigma_j )^2  &\geq  (\Delta_u + \sum_{j \neq u} \Delta_{u,j}\Enu \sigma_j )^2,\\
&\geq (\Delta_u)^2 - (\sum_{j \neq u} |\Delta_{u,j}|~ |\Enu \sigma_j| )^2.
\end{align}
Now from Lemma \ref{lem:mag_meta_bound},    we have $|\Enu \sigma_j| \leq \psi$. This gives  gives us from the assumed case condition,
\begin{align}
\Enu (\Delta_u + \sum_{j \neq u} \Delta_{u,j}\sigma_j )^2  &\geq  \Delta^2_u  - \psi^2(\sum_{j \neq u} |\Delta_{u,j}|)^2 \geq (\Delta_u)^2( 1 - \frac{4 \psi^2}{(1 + \psi)^2})
\end{align}

Notice that since $\psi < 1$, the term in the RHS above is positive.  Now by adding $|\Delta_u| \frac{1 + \psi}{2}$ to both sides of the assumed case condition, we get $(3 + \psi) |\Delta_u| \geq (1 + \psi) ||\Delta||_1.$ Using this in the above bound gives

\begin{align}
   \delta\lc(\Delta, \utheta_u^*) &\geq \frac{\exp(-10 \gamma)}{2}\Enu( \Delta_u + \sum_{j \neq u} \Delta_{u,j}\sigma_j )^2 \geq  \frac{e^{-10 \gamma}}{ 2} \frac{ 1 + 2 \psi - 3 \psi^2}{(3 + \psi)^2} ||\Delta||^2_1 \geq \frac{e^{-10 \gamma}}{ 2^5} (1 - \psi) ||\Delta||^2_1 
\end{align}

\paragraph*{\bf  Case 2:} $|\Delta_u| < \frac12(1 + \psi)   \sum_{j \neq u} |\Delta_{u,j}| $

\begin{align}\label{eq:rsc_mag_2}
\Enu (\Delta_u + \sum_{j \neq u} \Delta_{u,j}\sigma_j )^2  &\geq  \Var{\usigma \sim \nu} (\sum_{j \neq u} \Delta_{u,j} \sigma_j).\\
\end{align}

Now we use the law of Total variances here,
\begin{equation}
    \Var{\usigma \sim \nu} (\sum_{j \neq u} \Delta_{u,j} \sigma_j) \geq \EXp{\usigma_{\setminus i}}\Var{\usigma \sim \nu} (\sum_{j \neq u} \Delta_{u,j} \sigma_j| \usigma_{\setminus i}) = \Delta^2_{u,i} \EXp{\usigma_{\setminus i}} \Var{\sigma_i \sim \nu(.|\usigma_{\setminus i})}(\sigma_i | \usigma_{\setminus i})
\end{equation}

Now using Lemma \ref{lem:variance_close} we can bound it using the corresponding variance from the equilibrium distribution.

\begin{equation}
    \Var{\usigma \sim \nu} (\sum_{j \neq u} \Delta_{u,j} \sigma_j) \geq  \Delta^2_{u,i}\left( \EXp{\usigma_{\setminus i}} \Var{\sigma_i \sim \mu(.|\usigma_{\setminus i})}(\sigma_i | \usigma_{\setminus i}) - \frac{8 \eta}{\omega_P} \right).
\end{equation}

Now for the Ising model with magnetic fields $\Var{\sigma_i \sim \mu(.|\usigma_{\setminus i})}(\sigma_i | \usigma_{\setminus i}) = 1 - \tanh^2 \left( h_i + \sum_{j \neq i} J_{ij}^* \sigma_j \right) \geq \exp(-2 \gamma).$ 
Plugging this into the variance bound and maximizing over $i \in N(u)$ we get,

\begin{equation}
    \Var{\usigma \sim \nu} (\sum_{j \neq u} \Delta_{u,j} \sigma_j) \geq  \max_{i \in N(u)} (\Delta^2_{u,i})\left(e^{-2 \gamma} - \frac{8 \eta}{\omega_P}\right)
\end{equation}

Now from the sparsity of error in Lemma,
\begin{equation}
    3 \sum_{j \in N(u)} |\Delta_{uj}|   + |\Delta_u|  \geq   \sum_{j \notin N(u)} |\Delta_{uj}| ~\implies ~~ 4 \sum_{j \in N(u)} |\Delta_{uj}|  \geq ||\Delta||_1 - 2 |\Delta_u|
\end{equation}

Now the case condition can equivalently be written as $(3 + \psi) |\Delta_u| < ( 1+\psi) ||\Delta||_1.$ Using this bound in the above expression gives us
\begin{equation}
     \sum_{j \in N(u)} |\Delta_{uj}|  \geq \frac14 ||\Delta||_1 - \frac{1 + \psi}{2(3+\psi)} ||\Delta||_1 = ||\Delta||_1 \frac{1 - \psi}{4(3 + \psi)}
\end{equation}

Now assuming $d-$sparsity gives us $\max_{i \in N(u)} (\Delta^2_{u,i}) \geq \frac{1}{d^2}(\sum_{j \in N(u)} |\Delta_{uj}| )^2 $. Plugging this into the variance bound gives us,

\begin{align}
   \delta\lc(\Delta, \utheta_u^*) &\geq \frac{\exp(-10 \gamma)}{2d^2} \frac{(1-\psi)^2}{16(3 + \psi)^2} || \Delta||^2_1  \geq \frac{\exp(-10 \gamma)}{2^9d^2}(1-\psi)^2 || \Delta||^2_1  \label{eq:mag_f_var_bound} 
\end{align}

We can see that this is a weaker lower bound than what was obtained from Case 1. This can be easily veri.i.d. by comparing the terms in these bounds. Hence we can take \eqref{eq:mag_f_var_bound} to be the lowerbound on the curvature for both cases.

We can then use Hoeffding's lemma as in Lemma \ref{lem:rsc_stoc} to give the following bound on $\delta\lc_M(\Delta, \utheta_u^*).$

\begin{lemma}
Given a $\eta-$ strongly metastable distribution of a reversible chain with an equilibrium distribution satisfying Conditions \ref{cond:bounded_flip} and \ref{cond:temperature}, then $|\EXp{\usigma \sim \nu} \sigma_i | \leq \tanh(\gamma) + \frac{4 \eta}{\omega_P}.$ \label{lem:mag_meta_bound}
\end{lemma}
\begin{proof}
\begin{align}\label{eq:mag_meta}
    \EXp{\usigma \sim \nu} \sigma_i  =  \EXp{\usigma_{\setminus i} \sim \nu_{\setminus i}}~~ \EXp{\sigma_i \sim \mu(.|\usigma_{\setminus i})} \sigma_i +   \EXp{\usigma_{\setminus i} \sim \nu_{\setminus i}}\left(~~ \EXp{\sigma_i \sim \nu(.|\usigma_{\setminus i})} \sigma_i  -  ~~ \EXp{\sigma_i \sim \mu(.|\usigma_{\setminus i})} \sigma_i \right)
\end{align}

Now from Condition \ref{cond:temperature}, we have $\EXp{\sigma_i \sim \mu(.|\usigma_{\setminus i})} \sigma_i \leq \frac{1}{ 1+ e^{-2 \gamma}} - \frac{1}{1 + e^{2\gamma}} = \tanh(\gamma) $

Also from Theorem \ref{thm:close_condtionals} , 

\begin{equation}
      \EXp{\usigma_{\setminus i} \sim \nu_{\setminus i}}\left|~~ \EXp{\sigma_i \sim \nu(.|\usigma_{\setminus i})} \sigma_i  -  ~~ \EXp{\sigma_i \sim \mu(.|\usigma_{\setminus i})} \sigma_i \right| \leq    2 \EXp{\usigma_{\setminus i} \sim \nu_{\setminus i}}~~| \nu(.|\usigma_{\setminus i}) - \mu(.|\usigma_{\setminus i}) |_{TV} \leq \frac{4 \eta}{ \omega_P}
\end{equation}
Using these in \eqref{eq:mag_meta} we get,
\begin{align}\label{eq:mag_meta}
    |\EXp{\usigma \sim \nu} \sigma_i|  \leq  |\EXp{\usigma_{\setminus i} \sim \nu_{\setminus i}}~~ \EXp{\sigma_i \sim \mu(.|\usigma_{\setminus i})} \sigma_i| +   \EXp{\usigma_{\setminus i} \sim \nu_{\setminus i}}\left|~~ \EXp{\sigma_i \sim \nu(.|\usigma_{\setminus i})} \sigma_i  -  ~~ \EXp{\sigma_i \sim \mu(.|\usigma_{\setminus i})} \sigma_i \right| \leq \tanh(\gamma) + \frac{4 \eta}{\omega_P}.
\end{align}

\end{proof}
\cb

\end{document}